\documentclass[11pt]{amsart}

\usepackage[utf8]{inputenc}
\usepackage[english]{babel}
\usepackage[T1]{fontenc}
\usepackage{lmodern}

\usepackage[a4paper, margin=1in]{geometry}
\usepackage{microtype}
\usepackage{nicefrac}
\usepackage{url}

\usepackage{amsmath, amssymb, amsfonts}
\usepackage{mathtools}

\usepackage{graphicx}
\usepackage{tikz}
\usepackage{tikz-cd}
\usetikzlibrary{arrows.meta}
\usepackage{booktabs}
\usepackage{caption}
\usepackage{subcaption}

\makeatletter\renewcommand{\p@subfigure}{}\makeatother
\usepackage{float}
\usepackage{wrapfig}

\usepackage[ruled,vlined,linesnumbered]{algorithm2e}
\SetAlFnt{\small}
\SetAlCapFnt{\small}
\SetAlCapNameFnt{\bfseries\small}
\SetCommentSty{itshape}
\SetKwInput{KwIn}{Input}
\SetKwInput{KwOut}{Output}

\usepackage[numbers,sort&compress]{natbib}

\usepackage{xcolor}
\definecolor{darkgreen}{RGB}{0,100,0}

\usepackage[colorlinks=true,linkcolor=blue,urlcolor=blue,citecolor=blue]{hyperref}

\usepackage{fancyhdr}
\fancypagestyle{plain}{
  \fancyhf{}
  \cfoot{\small\thepage}
  
}

\newcommand{\R}{\mathbb{R}}
\newcommand{\E}{\mathbb{E}}

\newcommand{\dd}{\mathrm{d}}

\newcommand{\Tr}{\operatorname{Tr}}
\newcommand{\PP}{\mathbb{P}}
\newcommand{\no}{\noindent}

\theoremstyle{definition}
\newtheorem{definition}{Definition}[section]

\theoremstyle{plain}
\newtheorem{theorem}{Theorem}[section]
\newtheorem{lemma}{Lemma}[section]
\newtheorem{proposition}{Proposition}[section]

\theoremstyle{remark}
\newtheorem{remark}{Remark}[section]

\tikzset{
  whitep/.style={dashed, line width=0.9pt},
  blackp/.style={solid, line width=0.9pt},
  vtx/.style={circle,fill=black,inner sep=1.1pt},
  lbl/.style={font=\footnotesize}
}

\title{Generative Modeling via \\ \emph{Nonlinear} Schr\"odinger Bridges and Branching Diffusions}
\author{Mathias Grau and Pierre Henry-Labord\`ere$^\dagger$}

\begin{document}

\maketitle
\footnotetext[2]{Qube Research \& Technologies, Paris, France.
\href{mailto:mathias.grau@qube-rt.com}{\texttt{mathias.grau@qube-rt.com}},
\href{mailto:pierre.henrylabordere@qube-rt.com}{\texttt{pierre.henrylabordere@qube-rt.com}}}

\begin{abstract}
We propose a new generative modeling framework grounded in forward-backward partial differential equations (PDEs). 
These equations arise naturally in mean-field game (MFG) theory, and are closely related to nonlinear Schr\"odinger systems. 
Our main contribution is to reinterpret this connection within the setting of generative modeling and to develop an efficient numerical methodology for its practical implementation. 
Specifically, we derive and characterize a nonlinear extension of the classical Schr\"odinger bridge problem through a coupled system of PDEs, yielding a model we call \emph{SBNL (Schr\"odinger Bridge NonLinear)}. 
To compute its solution, we design a simulation algorithm based on branching diffusions and the probabilistic representation of semilinear PDEs.
The resulting method is embarrassingly parallel, scalable to high dimensions, and applicable to a broad class of nonlinear interaction functionals. 
We validate the approach through numerical experiments, including recovery of the classical Schr\"odinger bridge, nonlinear extensions with convex penalties (Gross-Pitaevskii) as well as potential interaction terms, and high-dimensional generative tasks such as image synthesis on the MNIST and FFHQ human face dataset. 
Furthermore, the same framework embeds score-based generative modeling, diffusion models, and tractable gaussian Schr\"odinger bridges.
\end{abstract}

\section{Introduction}

The Schr\"odinger bridge (SB) problem \citep{schrodinger1931umkehrung,follmer1988random,leonard2014survey} has recently attracted significant attention in generative modeling due to its ability to learn an optimally controlled stochastic dynamics that transports between two arbitrary endpoint distributions. By minimizing relative entropy with respect to a reference diffusion (typically Brownian motion), the SB framework reveals deep connections between optimal transport, stochastic control, and score-based generative modeling \citep{benamou2019entropy,song2021scorebased,chen2021likelihood,bortoli2021diffusion}.

When both terminal costs and the evolving marginal law are controlled, the problem extends beyond classical (linear) Schr\"odinger bridges to mean-field control and mean-field games (MFG) \citep{lasry2007mean,cardaliaguet2013notesMFG,GueantLasryLions2010,Gueant2012}. In this regime, interactions through the marginal distribution induce a semilinear forward--backward structure.

Existing iterative schemes for SB, such as iterative proportional fitting (IPF) / iterative Markovian fitting (IMF) and related diffusion-based approaches \citep{chen2021likelihood,bortoli2021diffusion,shi2023diffusion}, are primarily designed for \emph{linear} entropic transport. These methods rely on alternating optimization procedures, which can be computationally expensive and may suffer from error accumulation. More recent approaches based on Gaussian mixture energy models \citep{korotin2023light,korotin2024lightmatching} improve scalability in some settings but remain limited in expressivity in high dimensions. Even recent extensions based on controlled stochastic dynamics \citep{alouadi2026lightsbbmbridgingschrodingerbass} are not well suited to law-dependent running costs, which induce intrinsically coupled semilinear dynamics.

Although recent works in MFG and generalized Schr\"odinger bridge formulations clarify the structure of the coupling in terms of forward--backward potentials \citep{liu2022deepgeneralizedschrodingerbridge, tang2026branchedschrodingerbridgematching}, they often do not yield objectives that are both theoretically grounded and directly amenable to unbiased stochastic optimization in high dimensions.

Our contributions are as follows:
\begin{enumerate}
  \item We introduce a variational objective for semilinear mean-field Schr\"odinger bridges with law-dependent running and terminal costs, and characterize optimality through coupled $(\varphi,\bar\varphi)$ PDEs (equivalently, a Lasry - Lions $(m,v)$ system). We refer to the resulting model as \emph{SBNL (Schr\"odinger Bridge NonLinear)}.
  
  \item We develop a marked-branching stochastic algorithm for the two-type $(\varphi,\bar\varphi)$ coupling, prove unbiasedness and $L^2$-integrability of the estimator under explicit conditions, and train the potentials jointly (rather than via alternating refitting), with branching used both during training and sampling.
  
  \item We report experiments ranging from bridge recovery and soft terminal penalties to Gross-Pitaevskii-type couplings and high-dimensional latent image transport.
\end{enumerate}

\subsection{Related work}

\paragraph{Linear Schr\"odinger bridges.}
The Schr\"odinger bridge problem is classically linked to entropic optimal transport \citep{leonard2014survey,villani2003topics,benamou2019entropy}, with efficient Sinkhorn-type and iterative fitting procedures in the \emph{linear} setting. These methods scale well when dynamics are fixed between prescribed marginals. Recent neural approaches approximate bridges through forward-backward SDEs \citep{chen2021likelihood} or diffusion-style training \citep{bortoli2021diffusion,shi2023diffusion}; they mainly target linear bridges and often rely on expensive alternating refits (IPF/IMF). A complementary, simulation-free line parameterizes the Schr\"odinger potentials with Gaussian mixture forms, yielding lightweight solvers \citep{korotin2023light,korotin2024lightmatching} at the cost of limited expressivity. Even recent drift-and-volatility extensions (Schr\"odinger-Bass bridges) \citep{alouadi2026lightsbbmbridgingschrodingerbass} still assume (approximately) linear entropic transport and are not tailored to running costs that depend on the evolving marginal law.

\paragraph{Control formulations and mean-field games (MFG).}
Once running or terminal terms depend on the marginal law, the problem leaves the linear SB regime and sits in mean-field control / MFG \citep{lasry2007mean,cardaliaguet2013notesMFG}, with forward-backward PDE couplings studied in \citep{GueantLasryLions2010,swiecicki2016schrodinger,bonnemain2020negative,benamou2019entropy}. Extending Sinkhorn-type or alternating schemes to such coupled HJB-Fokker-Planck systems is challenging in high dimension. Recent interaction/unbalanced frameworks \citep{liu2024generalizedschrodingerbridgematching,zhang2025modelingcelldynamicsinteractions,tang2026branchedschrodingerbridgematching} encode growth, interaction, or state costs through composite weighted losses and staged training. Trajectory-level optimal-control views, e.g.\ adjoint matching \citep{domingoenrich2025adjointmatchingfinetuningflow}, couple learning to long-horizon simulation and can be costly. We retain entropic regularization while centering \emph{law-dependent} running costs, starting from the \emph{non linear} Sinkhorn system with jointly learned coupled Schr\"odinger potentials, using branching Monte Carlo.

\paragraph{Branching representations for semilinear PDEs.}
Branching diffusion methods provide probabilistic representations of semilinear PDEs via weighted particle systems \citep{mckean1975application,henrylabordere2019branching,HenryLabordereTouzi2021}. We extend these ideas to the coupled forward-backward system arising in mean-field Schr\"odinger bridges, where the interaction between the two potentials introduces a time-reflected dependency. This enables a scalable estimator for both the potentials and their gradients, which we integrate into a neural generative framework.

\section{Semilinear Schr\"odinger bridges}
\label{sec:semilinear-sb}

We denote by $\mu_0$ and $\mu_T$ laws on $\R^d$.
Let $(X_t)_{t\in[0,T]}$ be controlled by a drift $\alpha_\cdot$ under $\PP$, with $X_0\sim\mu_0$ and \( \dd X_t = \alpha_t\,\dd t + \dd W_t .\)

\subsection{Formulation}\label{sec:formulation}

\no As a generalization of the classical Schr\"odinger bridge formulation, we consider the following variational problem:
\begin{equation}
P := \inf_{\alpha_\cdot} \; \mathbb{E}_{X_0 \sim \mu_0}\left[ \int_0^T \left( \frac{1}{2c}|\alpha_t|^2 + f(m(t,\cdot)) \right) \mathrm{d}t \;+\; g(m(T,\cdot)) \right], \label{meanfield}
\end{equation}
where $c>0$, $m(t)$ denotes the law of $X_t$, and $f$ and $g$ are functionals. The infimum is taken over the control $\alpha$, where $m$ is constrained to satisfy the Fokker-Planck PDE.

The three terms appearing in $P$ play different roles:
\begin{itemize}
  \item The quadratic control cost \(\tfrac{1}{2c}|\alpha_t|^2\) penalizes deviations of the dynamics from the reference Brownian motion.
  By Girsanov's theorem, for feedback controls its expectation equals \(\tfrac1c H(\PP \mid \PP^0)\), the relative entropy
  between the controlled law \(\PP\) of the diffusion and the law \(\PP^0\) of the Brownian motion started from $\mu_0$ \citep{follmer1988random}.
  This quadratic structure is also crucial analytically: it is the \emph{only} choice of control penalty for which the associated Hamilton-Jacobi-Bellman (HJB) equation admits a tractable solution through a Cole-Hopf transformation.  For this reason, essentially all generative modeling frameworks based on stochastic control can be viewed as particular cases of the Schr\"odinger bridge.
  
  \item The functional \(f(m(t,\cdot))\) introduces a \emph{running cost} depending directly on the marginal law of the process at each time \(t\).  
  Such terms allow for mean-field type interactions, density regularization, or penalization of undesirable behaviors (e.g.\ concentration, mass spread) during the evolution. 
Classical examples include (see the review \citep{cardaliaguet2013notesMFG}):
  \begin{itemize}
      \item Quadratic interaction: 
      \(
      f(m) = \tfrac{\kappa}{2}\int_{\mathbb{R}^d} m(x)^2 \, \dd x,
      \) 
      which penalizes high concentrations and favors dispersion.  
      \item Entropy-type term: 
      \(
      f(m) = \int_{\mathbb{R}^d} m(x)\,\log m(x)\, \dd x,
      \) 
      encouraging spreading of the distribution.  
      \item Potential interaction: 
      \(
      f(m) = \int_{\mathbb{R}^d} V(x)\, m(x)\,\dd x,
      \) 
      where \(V\) is an external field acting on the density (as in \citep{liu2024generalizedschrodingerbridgematching}). 
      \item Pairwise interaction kernel: 
      \(
      f(m) = \tfrac{1}{2}\iint_{\mathbb{R}^d\times\mathbb{R}^d} K(x-y)\,m(x)\,m(y)\,\dd x \dd y,
      \) 
      modeling attractive or repulsive mean-field effects.
  \end{itemize}
  \item The terminal cost \(g(m(T,\cdot))\) enforces constraints or preferences on the final law.  
  In the hard-constraint setting of the classical Schr\"odinger bridge, \(g\) is chosen so that \(m(T,\cdot)=\mu_T\), i.e.\ the terminal distribution matches the prescribed target.  
  In the relaxed setting, \(g\) can encode convex penalties measuring the discrepancy between the final distribution and the target. 
  For use below, we denote by \(g^*\) the Legendre transform (convex conjugate) of \(g\):
   \begin{equation*} 
   g^*(v(T,\cdot)) := \sup_{m(T,\cdot)} \Big\{ \E_{X \sim m(T,\cdot)}[ v(T,X)] - g(m(T,\cdot)) \Big\}.
   \end{equation*}
Note that in practice it does not make sense to perfectly calibrate to the empirical measure, which is only a finite sample approximation of the true law.  
Instead, one seeks both proximity to the empirical distribution and sufficient regularity of the solution, precisely what the soft constraint achieves.
\end{itemize}

\begin{remark}[State dependence and MFG link]
Running and terminal costs that depend explicitly on the state $x$ can be rewritten in the law-only form \eqref{meanfield} by integrating against $m$; the resulting forward-backward MFG optimality system and its relation to Theorem~\ref{theorem:optimal_solution} are recorded in Appendix~\ref{app:state-mfg}.
\end{remark}

\subsection{Main result}

\begin{theorem}[Semilinear Schr\"odinger bridge] \label{theorem:optimal_solution}
We assume \textnormal{(A1)-(A4)} from Appendix~\ref{app:standing-assumptions}. The optimal controlled process of (\ref{meanfield}) is given under a measure $\PP$ by
\begin{equation}
  \label{eq:sde-process}
    \dd X_t = \nabla \log \varphi(t, X_t) \, \dd t +  \dd W_t,
\end{equation}
with $\mathrm{Law}(X_t)=\varphi(t,\cdot)\bar{\varphi}(T-t,\cdot)$ and $F:=\tfrac{\delta f}{\delta m(t,x)}$ where the pair \((\varphi,\bar{\varphi})\) satisfies the PDE system
\begin{align}
   \partial_t \varphi_t + \tfrac{1}{2} \Delta \varphi_t &= c \, \varphi_t \, F(t,x,\varphi_t\bar{\varphi}_{T-t}), \label{PDE1}\\
   \partial_t \bar{\varphi}_t + \tfrac{1}{2} \Delta \bar{\varphi}_t &= c \, \bar{\varphi}_t \, F(T-t,x,\bar{\varphi}_t\varphi_{T-t}), \label{PDE2}
\end{align}

\no The terminal conditions \((\varphi(T,\cdot), \bar{\varphi}(T,\cdot))\) are characterized by the coupled \emph{nonlinear Sinkhorn system}
\begin{equation}
  \label{eq:nl-sinkhorn}
  \left\{
  \begin{aligned}
  \varphi(0, x)\,\bar{\varphi}(T, x) &= \mu_0(x),\\
  \bar{\varphi}(0, x)\,\varphi(T, x) &= G^{-1}\!\big(-\tfrac{1}{c}\log \varphi(T, \cdot),\, x\big),
  \end{aligned}
  \right.
\end{equation}
where $G:=\tfrac{\delta g}{\delta m(T,x)}$, $G^{-1}(w,\cdot)$ is the density $m$ solving $G(\cdot,m)=w$, and $\varphi(0,\cdot),\bar\varphi(0,\cdot)$ are obtained by propagating the terminal potentials backward through \eqref{PDE1}-\eqref{PDE2}.
\end{theorem}

\no Appendix~\ref{demo-main-theorem} proves Theorem~\ref{theorem:optimal_solution} by a verification argument; see \citep{lasry2007mean} for the well-posedness of the forward-backward system.

\begin{remark}[Nonlinear Sinkhorn interpretation]
When $f\equiv0$ with a hard terminal constraint, \eqref{eq:nl-sinkhorn} reduces to the usual Sinkhorn system, with $(\varphi,\bar\varphi)$ propagated by \emph{linear} heat kernels and matched to $\mu_0,\mu_T$. Our formulation adds two nonlinear ingredients: the semilinear propagation \eqref{PDE1}-\eqref{PDE2} through the running cost $F$, and the terminal coupling $\varphi(T,\cdot)\mapsto G^{-1}(-\tfrac1c\log\varphi(T,\cdot),\cdot)$ induced by $g$. The values $\varphi(0,x), \bar\varphi(0,x)$ are obtained from the PDE system via the branching representation of Section~\ref{sec:branching}.
\end{remark}

\begin{remark}[Forward-backward system] \label{forward-backward system}
Although Theorem~\ref{theorem:optimal_solution} is stated in terms of the  system \eqref{PDE1}-\eqref{PDE2} with terminal conditions, it can equivalently be expressed
in the usual forward-backward form familiar from mean field games. Indeed,
defining the density \(m(t,x):=\varphi(t,x)\bar\varphi(T-t,x)\), one checks that
\(m\) satisfies the forward Fokker-Planck equation (see proof for details)
\[
\partial_t m + \nabla\!\cdot\!\big(m\,\nabla \log\varphi\big) - \tfrac12 \Delta m = 0,
\]
while \(v:=-\tfrac{1}{c}\log \varphi\) satisfies the backward Hamilton-Jacobi
equation
\[
\partial_t v + \tfrac12 \Delta v - \tfrac{c}{2}\,|\nabla v|^2
+ F(t,x,m) = 0
\]
Thus the pair \((m,v)\) solves the standard forward-backward MFG system of
Lasry - Lions type. Our formulation in terms of \((\varphi,\bar\varphi)\) is
algebraically equivalent but particularly well-suited for the branching Monte-Carlo representation introduced below.
\end{remark}

\begin{remark}[Soft constraint and matching marginal $\mu_T$]
  \label{rmk:matching_marginals}
Consider the terminal cost
\[
g(m(T,\cdot)) = \lambda \; D_f(\mu_T,m(T,\cdot)), \qquad \lambda>0,
\]
with $f(u) := -\log u -1$, i.e.\ $D_f(\mu_T,m)=\int m\log\frac{m}{\mu_T}-\int m$, which is $\mathrm{KL}(m(T,\cdot)\,\|\,\mu_T)-1$ on probability measures. Over nonnegative measures, the convex conjugate $g^*$ is
\begin{equation}
  \label{eq:gstar-kl}
g^*(v(T,\cdot)) = \lambda \; \E_{X_T \sim \mu_T}\!\left[ e^{\,v(T,X_T)/\lambda}\right],
\end{equation}
so that the terminal coupling in the nonlinear Sinkhorn system \eqref{eq:nl-sinkhorn} becomes explicit:
\begin{equation}
  \label{eq:soft-terminal-lambda-finite}
  G^{-1}\!\left(-\tfrac{1}{c}\log\varphi(T,\cdot),x\right)
  = \mu_T(x)\,\varphi(T,x)^{-a},
  \qquad
  a:=\frac{1}{\lambda c}.
\end{equation}
The system \eqref{eq:nl-sinkhorn} then reads
\begin{equation}
  \label{eq:nl-sinkhorn-kl}
  \left\{
  \begin{aligned}
  \varphi(0, x)\,\bar{\varphi}(T, x) &= \mu_0(x),\\
  \bar{\varphi}(0, x)\,\varphi(T, x) &= \mu_T(x)\,\varphi(T,x)^{-a}.
  \end{aligned}
  \right.
\end{equation}

\no In our numerical experiments, we consider this formulation for various values of $\lambda$.  
In the limit $\lambda \to \infty$ (i.e.\ $a\to 0$), the terminal penalty enforces the hard constraint $m(T,\cdot)=\mu_T$, and the terminal coupling in \eqref{eq:nl-sinkhorn-kl} reduces to the standard marginal constraint
\begin{equation}
  \label{eq:soft-terminal-lambda-infty}
  \bar{\varphi}(0, x)\,\varphi(T, x) = \mu_T(x),
\end{equation}
so that the nonlinear Sinkhorn system recovers the classical Schr\"odinger bridge.
\end{remark}

\noindent Appendix~\ref{app:wasserstein-terminal} records an analogous soft terminal condition based on a $p$-Wasserstein discrepancy.

\begin{remark}[Special cases of semilinear Schr\"odinger bridges]
Appendix~\ref{app:special-cases} records several tractable specializations: linear bridges with confining potentials and explicit heat kernels on Coxeter chambers (including Dyson and type-$C$ limits), the Gaussian semilinear bridge for log-linear $F$ \citep{bunne2023schrodinger} and a relative-entropy terminal cost (Proposition~\ref{propPHL}), lifted bridges with Gyöngy's Markov projection (Proposition~\ref{propGyongy}), and score-based / DDPM models as the exact $F\equiv 0$ solution with $\bar\varphi\equiv 1$ and noised-data prior (Appendix~\ref{app:score-based}), with proofs in Appendices~\ref{demo-gaussian} and~\ref{demo-gyongy}.
\end{remark}

\section{Branching diffusions}
\label{sec:branching}

\no The semilinear PDE system \eqref{PDE1}-\eqref{PDE2} is generally analytically intractable, and thus requires numerical approximation. 
A deterministic Sinkhorn-type iterative method based on PDE techniques was introduced by \citet{Gueant2012}; however, this approach suffers from the curse of dimensionality and is therefore not suitable in the context of generative modeling. 
\begin{wrapfigure}{r}{0.45\textwidth}
  \centering
  \captionsetup{font=small} 
  \includegraphics[width=\linewidth]{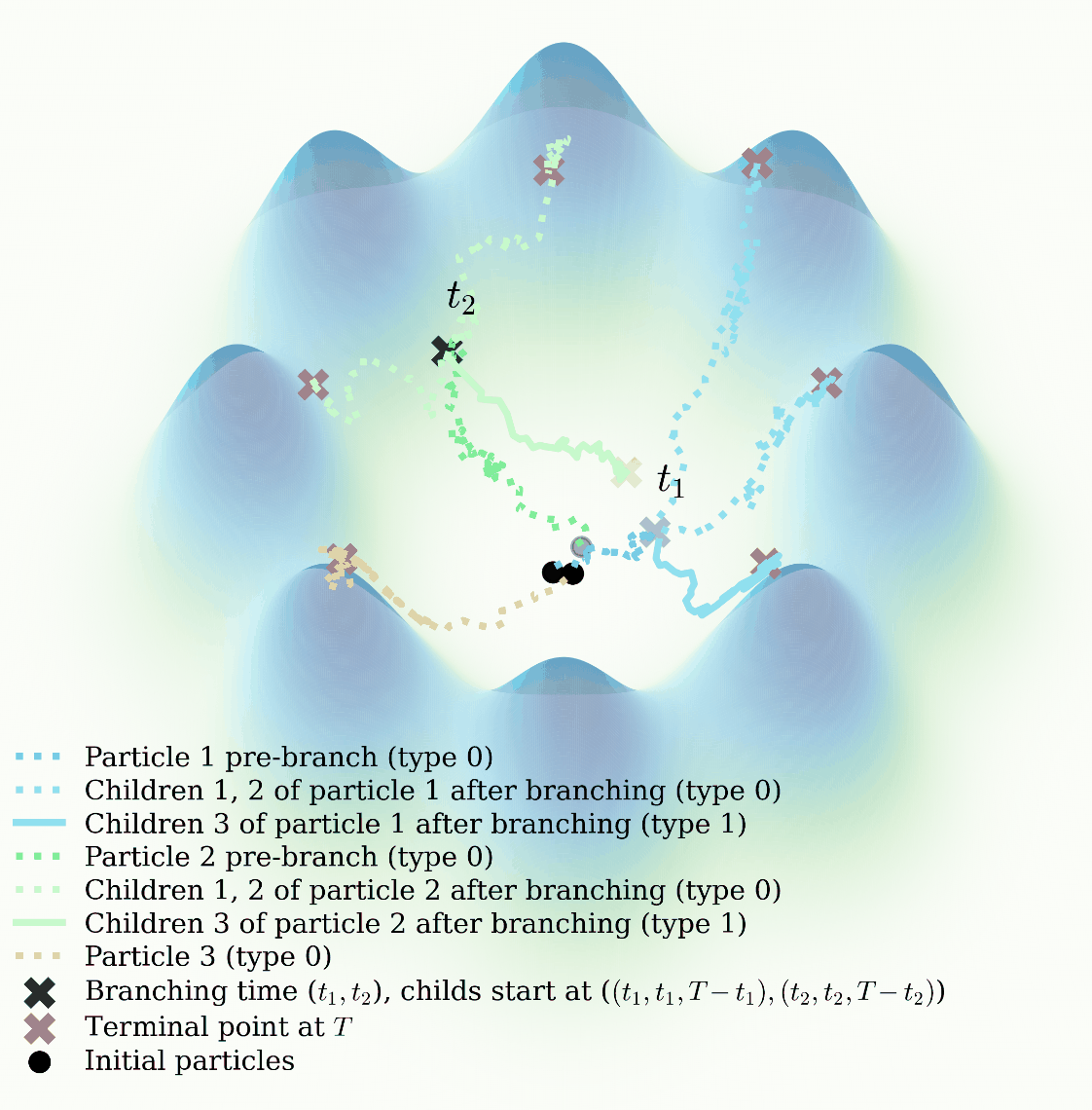}
  \caption{Branching paths on an 8-Gaussian terminal potential, illustrating Gross-Pitaevskii-type branching (Appendix~\ref{feynman-galton-watson-diagrams}).}
  \label{fig:branching-example}
\end{wrapfigure}
An alternative and more powerful approach is provided by \emph{marked branching diffusion representations}, 
originating in McKean’s probabilistic solution of the KPP equation \citep{mckean1975application} 
and subsequently extended to broad classes of semilinear and nonlinear PDEs \citep{HenryLabordere,henrylabordere2019branching,HenryLabordereTouzi2021}. 
These representations have been successfully applied to parabolic, elliptic, and even hyperbolic PDEs, including, notably, the complex Gross-Pitaevskii equation developed in Appendix~\ref{app:gross-pitaevskii}.

\no The key idea is to encode the nonlinear terms of the PDE into a branching mechanism: particles follow Brownian dynamics until a random exponential branching time, at which point they reproduce according to rules 
determined by the nonlinearity \(F\). 
The resulting branching trees provide unbiased MC estimators of the PDE solution. 
This stochastic representation is \emph{embarrassingly parallelizable}, which makes it particularly effective for large-scale simulations and high-dimensional problems.  

\no More generally, when \(F\) is local and analytic in $m(t,x)$, we truncate its power series at order \(K\), yielding a polynomial nonlinearity amenable to branching. Without loss of generality, we write it in the \emph{canonical form}
\begin{equation}
\label{canonical_form}
F(t,x,m) = 1 - \sum_{k=0}^K a_k(t,x)\,m^k := 1 - F_p(t,x,m),
\end{equation}
where the coefficients $(a_k)_{k=0,\cdots,K}$ are functions of $t$ and $x \in \R^d$.
Thus the system \eqref{PDE1}-\eqref{PDE2} becomes
\begin{equation}
  \left\{
  \begin{aligned}
  \partial_t \varphi_t + \tfrac12 \Delta \varphi_t 
  &= c\,\varphi_t\Big(1 - F_p\big( t, x, \varphi_t \bar\varphi_{T-t}\big)\Big), \\
  \partial_t \bar\varphi_t + \tfrac12 \Delta \bar\varphi_t 
  &= c\,\bar\varphi_t\Big(1 - F_p\big(T\!-\!t, x, \bar\varphi_t \varphi_{T-t}\big)\Big).
  \end{aligned}
  \right.
  \end{equation}
\subsection{Branching representation}

We use a marked branching Brownian motion with two types of particles:
white vertex \(\circ\) represents \(\varphi\),
black vertex \(\bullet\) represents \(\bar\varphi\). The branching rate is fixed to be $\beta = c$, and we choose arbitrarily probabilities \(p_k>0\) for \(k=0,1,\dots,K\) with \(\sum_{k=0}^K p_k=1\).
Each particle moves as a Brownian motion and lives an exponential time \(\tau\sim\mathrm{Exp}(c)\).
At a branching/death time \(s\), draw \(K^\star\) with probability \(p_{K^\star}\) and apply
\[
\left\{
\begin{aligned}
&\text{White parent }(\circ): &\quad 
\circ_s \;\longrightarrow\; (k{+}1)\,\circ_s \;+\; k\,\bullet_{T-s}, && k=0,1,\dots,K, \\[0.5pt]
&\text{Black parent }(\bullet): &\quad 
\bullet_s \;\longrightarrow\; k\,\circ_{T-s} \;+\; (k{+}1)\,\bullet_s, && k=0,1,\dots,K.
\end{aligned}
\right.
\]

Here, the subscript denotes the (re-)starting time of each particle: offspring with the same type as the parent start at time \(s\), and offspring with the opposite type start at time \(T-s\).

Each simulation carries a multiplicative \emph{weight} \(\mathcal{W}\), initialized to \(1\) at the root \((t,x)\).
At every branching event occurring at time \(s\) with outcome \(K^\star\), it is updated as
\begin{equation}
  \label{eq:weight-update}
  \mathcal{W} \;\longrightarrow\; \mathcal{W} \times \frac{a_{K^\star}(s^{\circ},X_s)}{p_{K^\star}},
  \qquad s^{\circ}:=s\ \text{(white parent)},\quad s^{\circ}:=T-s\ \text{(black parent)},
\end{equation}
where $X_s$ is the position of the particle at time $s$. Thus the global weight of a branching tree is the product of the factors associated with all branchings. At time \(T\), each surviving white (resp.\ black) particle at position $X_T^i$ (resp.\ $\bar X_T^j$) contributes the factor $\varphi(T,X_T^i)$ (resp.\ $\bar\varphi(T,\bar X_T^j)$). We denote by
\[
\left\{
\begin{aligned}
N_T^{\mathrm{white}} \quad &\text{the number of white particles at time $T$, with positions } X_T^1,\dots,X_T^{N_T^{\mathrm{white}}},\\[0.3pt]
\bar N_T^{\mathrm{black}} \quad &\text{the number of black particles at time $T$, with positions } \bar X_T^1,\dots,\bar X_T^{\bar N_T^{\mathrm{black}}}.
\end{aligned}
\right.
\]

\begin{proposition}[Unbiased stochastic representation for $\varphi$]
\label{prop:phi-branch-lamc}
Let $A:=\|\varphi(T,\cdot)\|_{\infty}$, $B:=\|\bar\varphi(T,\cdot)\|_{\infty}$ and $b_k:=\|a_k\|_{L^\infty([0,T]\times\R^d)}$ be finite, and assume
(i) $cT\sum_{k}k\,p_k<1$ and (ii) $A^2B^2\exp\big(cT(\sum_{k} b_k^2r^k/p_k-1)\big)\le r$ for some $r>0$. Set
\[
\Phi
:=
\mathcal{W}
\prod_{i=1}^{N_T^{\mathrm{white}}} \varphi(T,X_T^i)
\prod_{j=1}^{\bar N_T^{\mathrm{black}}} \bar\varphi(T,\bar X_T^j).
\]
\no Then the tree is a.s.\ finite, $\Phi\in L^2$, and $\varphi(t,x)=\E_{t,x}[\Phi]$ (and its black-rooted analogue $\bar\varphi$) is the mild solution of \eqref{PDE1}-\eqref{PDE2} dominated by \eqref{eq:majorant-L1}.
\end{proposition}
\no The proof is deferred to Appendix~\ref{demo-L1-L2-integrability}; illustrative Feynman-Galton-Watson diagrams in Appendix~\ref{feynman-galton-watson-diagrams}, and a numerical sanity check of the time-reflected coupling $u(t,\cdot)u(T-t,\cdot)$ in Appendix~\ref{app:branching-sanity-toy}.

\section{Stochastic-optimization form of the \emph{nonlinear} Sinkhorn system}
\label{sec:stochastic-opt-form}
To enforce the two nonlinear Sinkhorn conditions \eqref{eq:nl-sinkhorn}, we match the associated measures using the Maximum Mean Discrepancy (MMD) \citep{JMLR:v13:gretton12a}, which we now recall.

\begin{definition}[Maximum Mean Discrepancy]
\label{def:mmd}
Let $k$ be a positive definite kernel. For two probability laws $P,Q$ on $\R^d$, the MMD admits the closed form
\begin{equation*}
  \mathrm{MMD}_k^2(P,Q)
  =\E_{X,X'\sim P}[k(X,X')]
  +\E_{Y,Y'\sim Q}[k(Y,Y')]
  -2\,\E_{\substack{X\sim P ; Y\sim Q}}[k(X,Y)],
\end{equation*}
where $X,X'$ (resp.\ $Y,Y'$) are independent copies from $P$ (resp.\ $Q$).
\end{definition}

\begin{remark}[Characteristic kernel]
\label{rmk:characteristic_kernel}
If $k$ is characteristic, then $\mathrm{MMD}_k(P,Q)=0 \iff P=Q$. Background on MMD and RKHS, and the proof of Proposition~\ref{prop:mmd_objective}, are given in Appendix~\ref{app:rkhs_mmd}.
\end{remark}

\begin{proposition}[MMD objective]
\label{prop:mmd_objective}
Let $k$ be as in Lemma~\ref{lem:mmd-finite-measures} (e.g.\ Gaussian), assume the four densities below are integrable, and define
\begin{equation}
  \mathcal L_{\mathrm{MMD}}
  :=
  \mathrm{MMD}_k^2\!\Big(\varphi(0,\cdot)\,\bar\varphi(T,\cdot),\,\mu_0\Big)
  +
  \mathrm{MMD}_k^2\!\Big(\bar\varphi(0,\cdot)\,\varphi(T,\cdot),\,G^{-1}\!\big(-\tfrac{1}{c}\log\varphi(T,\cdot),\cdot\big)\Big).
\end{equation}
Then $\mathcal L_{\mathrm{MMD}}\!\ge \!0$, and \(
  \mathcal L_{\mathrm{MMD}}=0
  \Longleftrightarrow
  (\varphi(T,\cdot),\bar\varphi(T,\cdot))\) satisfy \eqref{eq:nl-sinkhorn} a.e.
\end{proposition}

\subsection{Meta-algorithm for training semilinear Schr\"odinger bridges}
\label{sec:meta_algorithm_train_ssb}
\leavevmode\par
\noindent
\begin{minipage}[t]{0.50\textwidth}
\vspace{0pt}
We parametrize the terminal potentials by neural networks
$\varphi_\theta(x):=\varphi_\theta(T,x)$ and
$\bar\varphi_\eta(x):=\bar\varphi_\eta(T,x)$.
Let $\mu^*$ be a reference density that is easy to sample and evaluate.

For $Z\sim\mu^*$, we define
$\widehat A_{\theta,\eta}(Z,\ \circ\ (\text{resp. } \bullet))$
using a white (resp.\ black) rooted branching tree starting from
$(0,Z)$ :
\begin{equation*}
\widehat A_{\theta,\eta}(Z,\circ)
:=
\mathcal W
\prod_{i=1}^{N_T^{\mathrm{white}}}
\varphi_\theta(Z_T^i)
\prod_{j=1}^{\bar N_T^{\mathrm{black}}}
\bar\varphi_\eta(\bar Z_T^j),
\end{equation*}

with analogous definition for a black-rooted tree.
\end{minipage}
\hfill
\begin{minipage}[t]{0.45\textwidth}
\vspace{0pt}
\begin{algorithm}[H]
\footnotesize
\caption{Training semilinear SB}
\label{alg:meta-train-ssb}

\KwIn{Laws $\mu_0,\mu_T, \mu^*$, rate $c$, rule $F$, penalty $\lambda>0$, kernel $k$, step $\gamma$, batch size $B$.}

\KwOut{Trained potentials $(\varphi_\theta,\bar\varphi_\eta)$.}

Initialize $\theta,\eta$\;

\While{not converged}{
Draw
$\{X_0^{(b)}\}_{b=1}^B\sim\mu_0$,
$\{X_T^{(b)}\}_{b=1}^B\sim\mu_T$, and
$\{Z^{(b)}\}_{b=1}^B\sim\mu^*$\;

Simulate white and black trees from each
$Z^{(b)}$ and form
$\widehat A_{\theta,\eta}(Z^{(b)},\circ)$ and $\widehat A_{\theta,\eta}(Z^{(b)},\bullet)$\;

$\widehat{\mathcal L}_{\mathrm{MMD}}
\! \gets \!
\frac{1}{B(B-1)}
\sum_{b\neq b'} \! J(b,b')$\tcp*[r]{U-statistic}

$(\theta,\eta)
\gets
(\theta,\eta)
-\gamma\nabla_{(\theta,\eta)}
\widehat{\mathcal L}_{\mathrm{MMD}}$\;
}
\end{algorithm}
\end{minipage}

\noindent
In the KL terminal case of Remark~\ref{rmk:matching_marginals}, using the
reweighted kernel \( \tilde{k}(x,y):=\mu^*(x)\,k(x,y)\,\mu^*(y) \),
which remains characteristic provided $\mu^*>0$ almost everywhere, the MMD
objective computed with $\tilde k$ reads

\begin{equation}
\label{eq:J-mc}
\min_{\theta,\eta}\mathcal L_{\mathrm{MMD}}(\theta,\eta)
:=
\min_{\theta,\eta}\E[J].
\end{equation}
where
\begin{equation}
\label{eq:J-mc-detailed}
\resizebox{0.90\linewidth}{!}{$
\begin{split}
J
&={}
\widehat A_{\theta,\eta}(Z,\circ)
\bar\varphi_\eta(Z)
\widehat A_{\theta,\eta}(Z',\circ)
\bar\varphi_\eta(Z')
k(Z,Z')
-
2\widehat A_{\theta,\eta}(Z,\circ)
\bar\varphi_\eta(Z)
k(Z,X_0')
\mu^*(X_0')
\\
&+
\mu^*(X_0)
k(X_0,X_0')
\mu^*(X_0')
+
\widehat A_{\theta,\eta}(Z,\bullet)
\varphi_\theta(Z)
\widehat A_{\theta,\eta}(Z',\bullet)
\varphi_\theta(Z')
k(Z,Z')
\\
&-
2\widehat A_{\theta,\eta}(Z,\bullet)
\varphi_\theta(Z)
\varphi_\theta(X_T')^{\frac{-1}{\lambda c}}
k(Z,X_T')
\mu^*(X_T')
+
\left(
\varphi_\theta(X_T)
\varphi_\theta(X_T')
\right)^{\frac{-1}{\lambda c}}
\mu^*(X_T)
k(X_T,X_T')
\mu^*(X_T').
\end{split}
$}
\end{equation}

\noindent
Here,
$(Z,Z')\sim\mu^*\otimes\mu^*$,
$(X_0,X_0')\sim\mu_0\otimes\mu_0$, and
$(X_T,X_T')\sim\mu_T\otimes\mu_T$
are drawn independently.
As the trees attached to $Z$ and $Z'$ are independent, $\E[J]=\mathcal L_{\mathrm{MMD}}$ (with kernel $\tilde k$); Algorithm~\ref{alg:meta-train-ssb} averages $J(b,b')$, i.e.\ $J$ at the $b$-th and $b'$-th samples, over $b\neq b'$, which keeps the estimate unbiased. Each term is estimated on mini-batches and minimized by stochastic gradient descent, with a Gaussian kernel and a suitable $\mu^*$ as stable defaults (Appendix~\ref{app:stochastic-opt-forms}).

\subsection{Meta-algorithm for sampling with semilinear Schr\"odinger bridges}
\label{sec:meta_algorithm_sample_ssb}
\leavevmode\par
\noindent
\begin{minipage}[t]{0.44\textwidth}
\vspace{0pt}

Once the potentials $(\varphi_\theta,\bar\varphi_\eta)$ have been learned, we can generate new samples by simulating the optimal controlled process (Algorithm~\ref{alg:meta-sample-ssb}): under $\PP$, \(\dd X_t \;=\; \nabla \log \varphi_\theta(t,X_t)\,\dd t + \dd W_t\). The branching representation of Proposition~\ref{prop:phi-branch-lamc} gives unbiased Monte-Carlo estimates $\widehat\varphi$ of $\varphi_\theta(t,x)$ around $x=X_t$, differentiated in $x$ with \textit{autograd}; the ratio $\nabla\widehat\varphi/\widehat\varphi$ is a consistent ($M\to\infty$) estimator of the score field. We then discretize the dynamics by Euler-Maruyama (or a higher-order scheme).

\end{minipage}
\hfill
\begin{minipage}[t]{0.52\textwidth}
\vspace{0pt}

\begin{algorithm}[H]
\footnotesize
\caption{Sampling from semilinear SB}
\label{alg:meta-sample-ssb}
\KwIn{Trained $(\varphi_\theta,\bar\varphi_\eta)$, rate $c$, rule $F$, Euler steps $N$, horizon $T$, samples $M$.}
\KwOut{Approximate bridge path $(X_{t_k})_{k=0}^N$.}
Initial $X_0\sim\mu_0$, $\Delta t\gets T/N$\;
\For{$k=0$ \KwTo $N-1$}{
Simulate $M$ white-rooted trees from $(t_k,X_{t_k})$\;
Form $\widehat\varphi_k(x)$ of $\varphi(t_k,x)$:
\[
\widehat\varphi_k(x)\!\gets\!\tfrac{1}{M}\!\sum_{m=1}^M\!\mathcal{W}^m\!
\prod_{i=1}^{N_T^{m,\mathrm{w}}}\!\varphi_\theta(X_{T,x}^{m,i})\!
\prod_{j=1}^{\bar N_T^{m,\mathrm{b}}}\!\bar\varphi_\eta(\bar X_{T,x}^{m,j})
\]
$v_k\gets \nabla_x\log\widehat\varphi_k(x)|_{x=X_{t_k}}$\tcp*[r]{autograd}
Draw $\xi_k\sim\mathcal{N}(0,I_d)$\;
$X_{t_{k+1}}\gets X_{t_k}+\Delta t\,v_k+\sqrt{\Delta t}\,\xi_k$\;
}
\end{algorithm}

\end{minipage}

\section{Numerical experiments}
\label{sec:experiments}

\subsection{Implementation details}

Each of $\varphi_\theta, \bar\varphi_\eta : \mathbb{R}^d \to (0,\infty)$ is a positive-scalar feedforward network. Unless stated otherwise, all use 3 to 8 fully connected layers with nonlinear activations.
Widths are $16$ for $d=2$, $32$ for $d=50$, and $64$ for $d=512$ for latent FFHQ experiments.
Potentials are trained jointly by stochastic gradient descent on the MMD objective \ref{eq:J-mc} (Adam, $\gamma=10^{-4}$, $B=64$). We use automatic differentiation of $\log \varphi_\theta$ for sampling (Section~\ref{sec:meta_algorithm_sample_ssb}); see Appendix~\ref{app:implementation-details}. We train on $10^4$ samples in $2$D and $10^5$ samples in $50$D, for $300$ epochs in both cases.
Sampling uses between $N=50$ and $N=200$ Euler steps and $M\in\{10^3,10^4\}$ branching trees per score evaluation (higher $M$ in higher dimension).
All runs use NVIDIA Tesla~T4 GPUs; bridge training takes $\sim\!30$\,s for $300$ epochs on $10^4$ points in $2$D, and generation $\sim\!3$\,s for $50$ steps.

\subsection{Results}

\paragraph{Classical Schr\"odinger bridges ($F=0$, $\lambda=\infty$, $d=2$).}
We begin with controlled experiments in low dimensions to validate that our algorithm recovers the \emph{classical Schr\"odinger bridge}. 
Specifically, we set $F \equiv 0$ and enforce a hard terminal constraint  (see Remark~\ref{rmk:matching_marginals} with $\lambda=\infty$), so that the problem reduces to entropic optimal transport. 
\begin{figure}[H]
\centering
\includegraphics[width=0.9\textwidth]{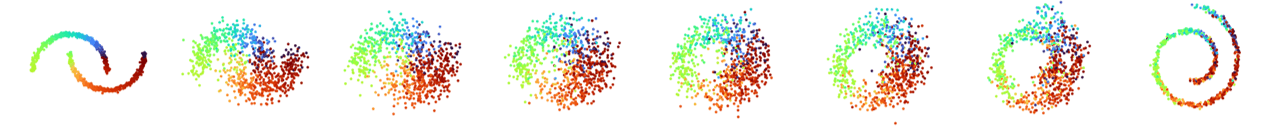}
\captionsetup{font=footnotesize}
\caption{Temporal evolution of $\mathrm{Law}(X_t)$ from Moons $\to$ Swiss (source $\mu_0$ on the left, target $\mu_T$ on the right).}
\label{fig:evolution_strip}
\end{figure}
We compare SBNL against DeepGSB \citep{liu2022deepgeneralizedschrodingerbridge}, LightSB \citep{korotin2023light}, LightSBB-M \citep{alouadi2026lightsbbmbridgingschrodingerbass}, DSBM \citep{shi2023diffusion}, and BranchSBM \citep{tang2026branchedschrodingerbridgematching} on a Gaussian $\to$ mixture of $8$ Gaussians benchmark, selecting methods that are closest to our setting in terms of modeling assumptions and problem formulation. Figure~\ref{fig:comparison_2D_samples_main} displays terminal samples, while Table~\ref{tab:comparison_2D_metrics} reports quantitative distances to $\mu_T$. SBNL captures all modes and achieves the best mixture $W_2$, whereas BranchSBM attains the lowest empirical $W_2$. Overall, SBNL remains competitive across evaluation metrics. Additional per-mode coverage results and full comparisons are provided in Appendix~\ref{app:comparison-2d}.
\begin{figure}[H]
\centering
\captionsetup{font=footnotesize}
\captionsetup[subfigure]{font=footnotesize}
\begin{subfigure}[t]{\textwidth}
  \centering
  \includegraphics[width=0.7\linewidth]{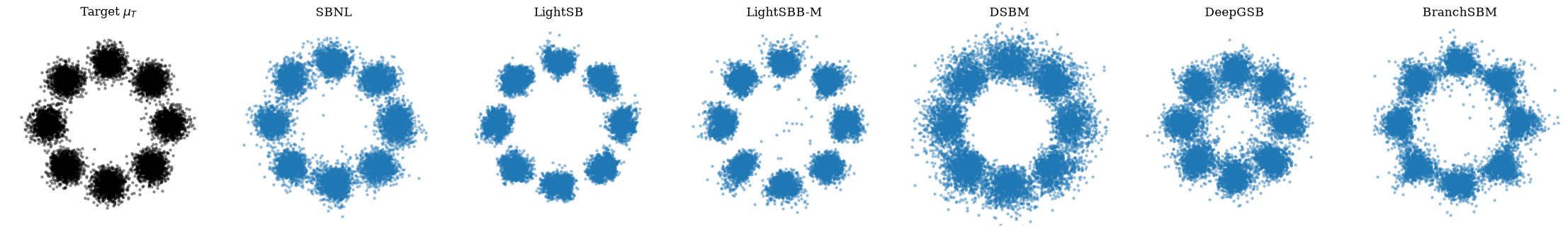}
  \caption{Terminal samples on the $8$-Gaussians benchmark.}
  \label{fig:comparison_2D_samples_main}
\end{subfigure}\\[2pt]
\begin{subfigure}[c]{0.55\textwidth}
  \centering
  \setlength{\tabcolsep}{1.5pt}
  \renewcommand{\arraystretch}{1.05}
  {\tiny
  \begin{tabular}{lccccc}
  \toprule
  Method & $MW_2$ & $W_2^{\text{emp}}$ & $E$ & $\|\Delta \bar{x}\|$ & $\|\Delta\Sigma\|_F$ \\
  \midrule
  SBNL
  & $\mathbf{1.52}_{\pm 0.07}$
  & $1.14_{\pm 0.23}$
  & $0.023_{\pm 0.012}$
  & $0.29_{\pm 0.04}$
  & $2.97_{\pm 0.22}$ \\

  BranchSBM
  & $1.57_{\pm 0.14}$
  & $\mathbf{1.05}_{\pm 0.12}$
  & $\mathbf{0.018}_{\pm 0.009}$
  & $\mathbf{0.08}_{\pm 0.05}$
  & $\mathbf{1.74}_{\pm 0.19}$ \\

  LightSBB-M
  & $1.55_{\pm 0.27}$
  & $1.26_{\pm 0.15}$
  & $0.018_{\pm 0.009}$
  & $0.16_{\pm 0.07}$
  & $1.79_{\pm 0.39}$ \\

  LightSB
  & $1.54_{\pm 0.24}$
  & $1.22_{\pm 0.12}$
  & $0.024_{\pm 0.009}$
  & $0.11_{\pm 0.06}$
  & $4.69_{\pm 0.49}$ \\

  DSBM
  & $2.54_{\pm 0.15}$
  & $1.61_{\pm 0.12}$
  & $0.044_{\pm 0.017}$
  & $0.23_{\pm 0.09}$
  & $12.05_{\pm 0.84}$ \\

  DeepGSB
  & $2.48_{\pm 0.19}$
  & $2.02_{\pm 0.11}$
  & $0.199_{\pm 0.020}$
  & $0.24_{\pm 0.08}$
  & $25.93_{\pm 0.24}$ \\
  \bottomrule
\end{tabular}
  }
  \subcaption{Comparison ($MW_2$ : closed-form Wasserstein, $W_2^{\text{emp}}$ : empirical $W_2$ on $10^4$ samples, $E$ : energy distance, and $\|\Delta\bar x\|$ (resp.\ $\|\Delta\Sigma\|_F$) : Euclidean (resp.\ Frobenius) error on the sample mean (resp.\ covariance)).}
  \label{tab:comparison_2D_metrics}
\end{subfigure}
\hfill
\begin{subfigure}[c]{0.43\textwidth}
  \centering
  \includegraphics[width=\linewidth]{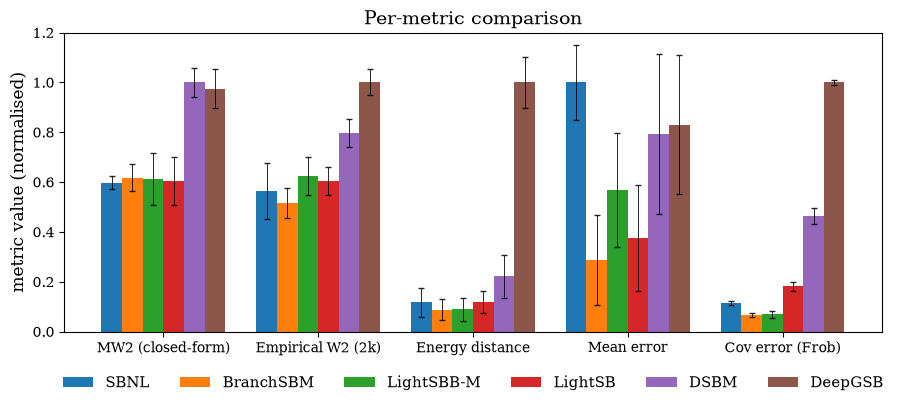}
  \caption{Visual comparison across all metrics, each column-wise max-normalized to $1$ (lower is better).}
  \label{fig:comparison_MW2_main}
\end{subfigure}
\end{figure}

\paragraph{High-dimensional mixtures ($F(m)=0$, $\lambda=\infty$, $d=50$).}
We use the same $d{=}50$ Gaussian-mixture benchmark as in \citep{shi2023diffusion} (Figures~\ref{fig:gen_vs_true_50D}-\ref{fig:evolution_50D}).
The empirical mean, standard deviation, and skewness of $\mathrm{Law}(X_T)$ converge to those of $\mu_T$ within $50$ training epochs (Figure~\ref{fig:three_in_a_row}).
We observe no variance collapse, attributable to the consistent branching estimator of $\nabla\log\varphi_\theta$ avoiding the bias that score-matching pipelines accumulate at high $d$ \citep{bortoli2021diffusion}.

\paragraph{Soft terminal conditions ($F=0$, $\lambda<\infty$, $d=2$).}
We next relax the hard constraint by introducing a finite penalty parameter $\lambda$.  
Figure~\ref{fig:evolution_lambda_generated} shows $\mathrm{Law}(X_t)$ for decreasing values of $\lambda$, and Figure~\ref{fig:KL_vs_lambda} reports the induced Kullback-Leibler divergence at the terminal time.  
As expected, $D_{\mathrm{KL}}(\mu_T,\mathrm{Law}(X_T))$ decreases monotonically with $\lambda$, confirming that stronger penalties enforce closer alignment with the target law.
\begin{figure}[H]
\centering
\captionsetup{font=footnotesize}
\captionsetup[subfigure]{font=footnotesize}
\begin{minipage}[c]{0.48\textwidth}
  \begin{subfigure}[c]{\linewidth}
    \centering
    \includegraphics[width=0.8\linewidth]{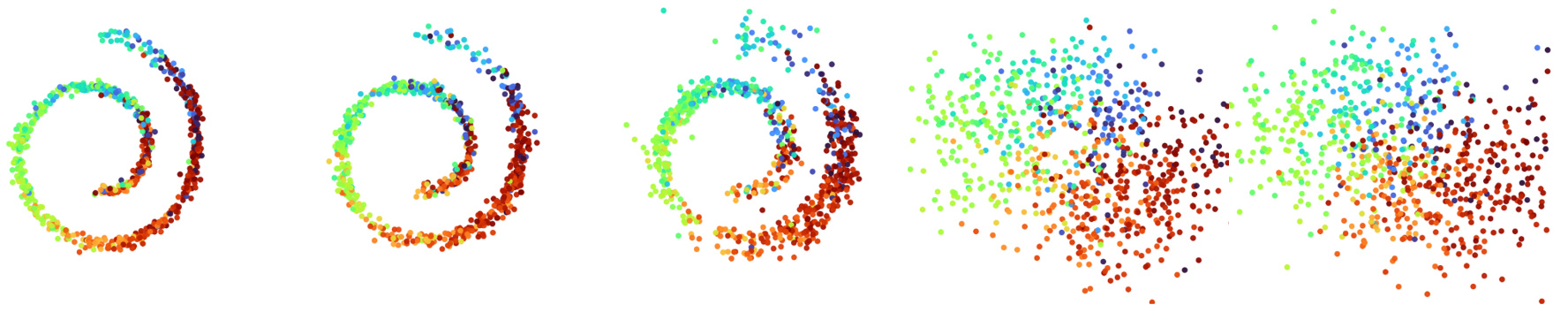}
    \caption{Generation for $\lambda \in \{10^5,10^4,5\!\times\!10^3,10^3,10^2\}$.}
    \label{fig:evolution_lambda_generated}
  \end{subfigure}\\[0.5em]
  \begin{subfigure}[c]{\linewidth}
    \centering
    \includegraphics[width=0.6\linewidth]{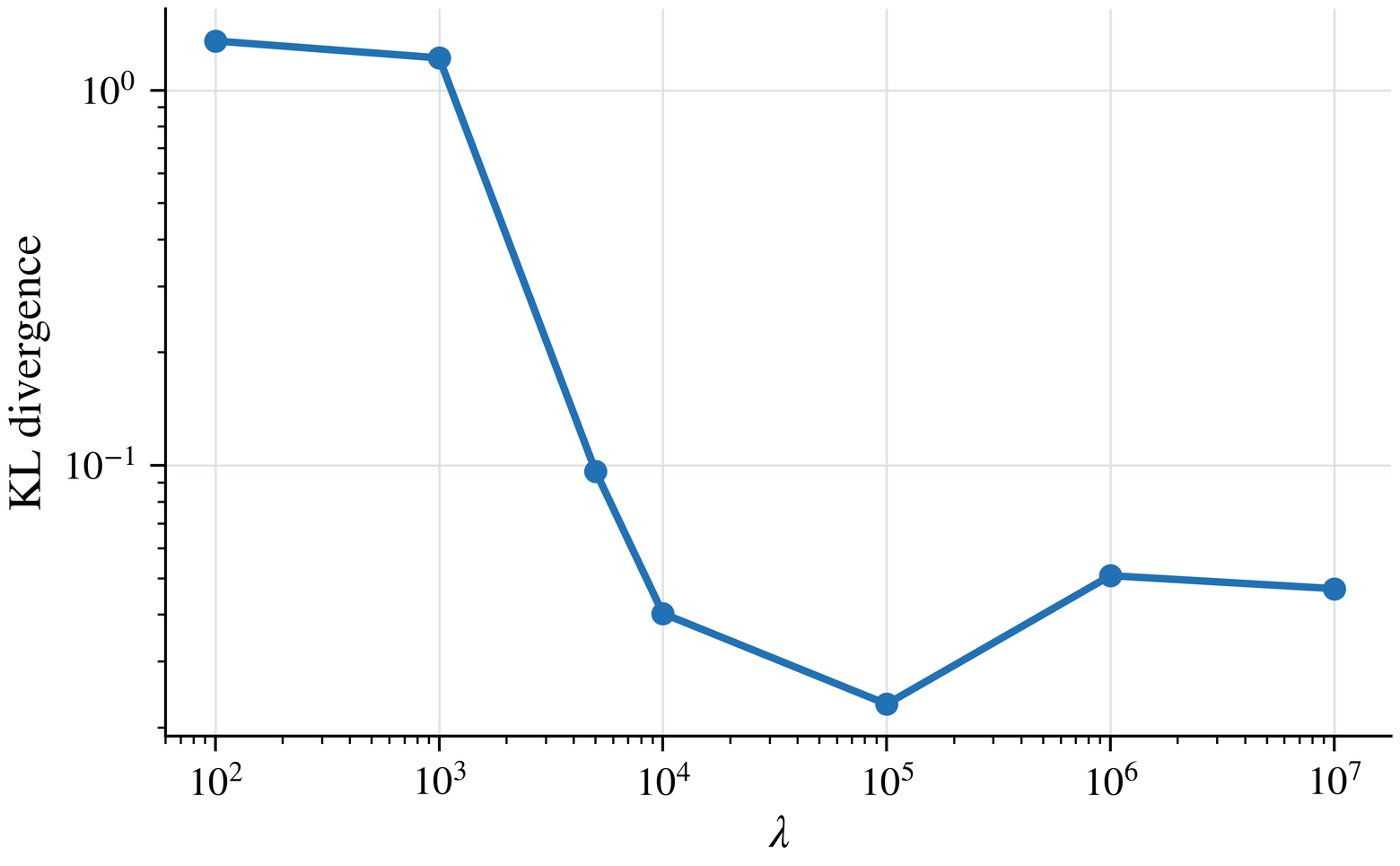}
    \caption{Terminal $D_{\mathrm{KL}}(\mu_T,\mathrm{Law}(X_T))$ versus $\lambda$.}
    \label{fig:KL_vs_lambda}
  \end{subfigure}
\end{minipage}
\hfill
\begin{subfigure}[c]{0.48\textwidth}
  \centering
  \includegraphics[width=0.9\linewidth]{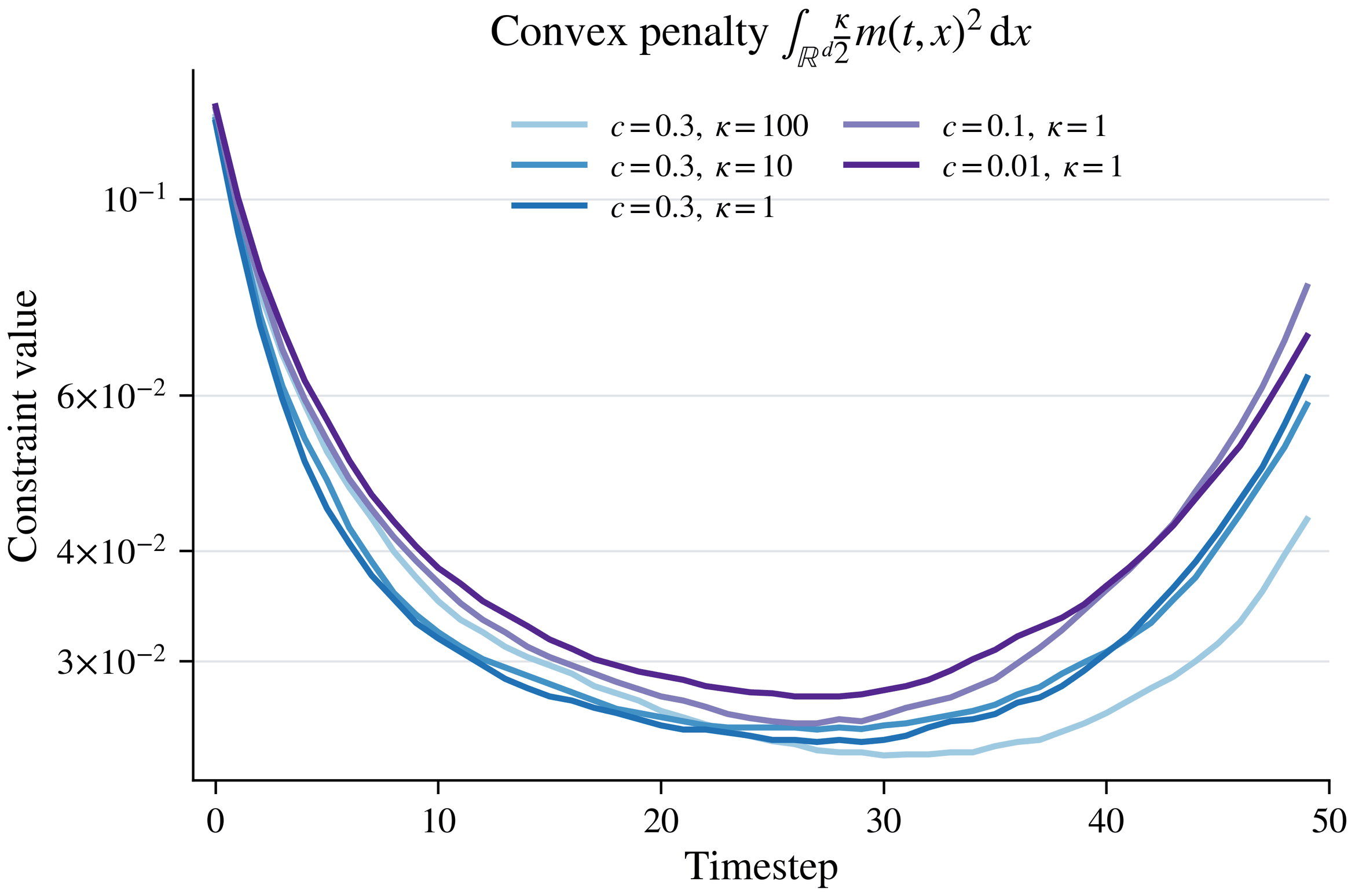}
  \caption{Evolution of the convex penalty during generation for $c\in \{0.01 , 0.1, 0.3\}$ and $\kappa \in \{1, 10, 100\}$.}
  \label{fig:evolution_convex}
\end{subfigure}
\end{figure}
\paragraph{Gross-Pitaevskii penalty ($F(m)=\kappa m$, $\lambda=\infty$, $d=2$).}
We also investigate the convex penalty functional \(f(m(t,\cdot)) := \tfrac{\kappa}{2} \int_{\R^d} m(t,x)^2 \, \dd x,\) so that $F(m)=\kappa m$.  
Larger $\kappa$ and $c$ increase branching, making the dynamics more sensitive to the nonlinearity (Figure~\ref{fig:evolution_convex}).

\paragraph{Potential running cost ($F(t,x,m)=V(x)$, $\lambda=\infty$, $d=2$).}
We finally illustrate a \emph{state-dependent} running cost of potential type,
$f(m(t,\cdot)) := \int_{\R^d} V(x)\, m(t,x)\,\dd x$, so that $F(t,x,m)=V(x)$.
This lets the environment \emph{shape} the transport path: the optimal bridge routes mass around the obstacle while still matching both marginals.
Figure~\ref{fig:potential_samples} (resp.\ Figure~\ref{fig:potential_density}) shows the evolving samples (resp.\ density estimation) of $\mathrm{Law}(X_t)$ deflecting around the high-potential region, and Figure~\ref{fig:potential_3d} display a 3D evolution. Further potential landscapes are reported in Appendix~\ref{sec:potential-running-cost}.
\begin{figure}[H]
\centering
\captionsetup{font=footnotesize}
\captionsetup[subfigure]{font=footnotesize}
\begin{minipage}[c]{0.66\textwidth}
  \begin{subfigure}[c]{\linewidth}
    \centering
    \includegraphics[width=\linewidth]{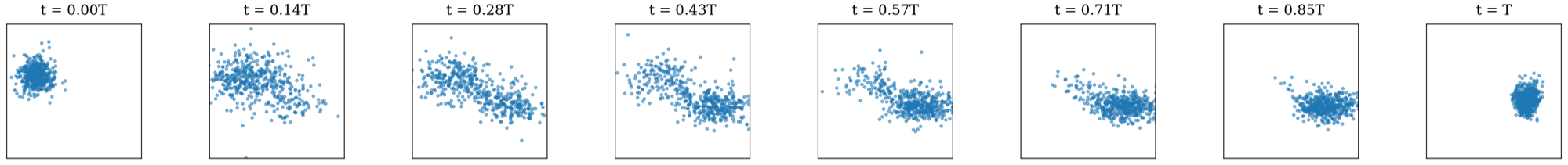}
    \caption{Samples $\mathrm{Law}(X_t)$ deflecting around high-$V$.}
    \label{fig:potential_samples}
  \end{subfigure}\\[1.0em]
  \begin{subfigure}[c]{\linewidth}
    \centering
    \includegraphics[width=\linewidth]{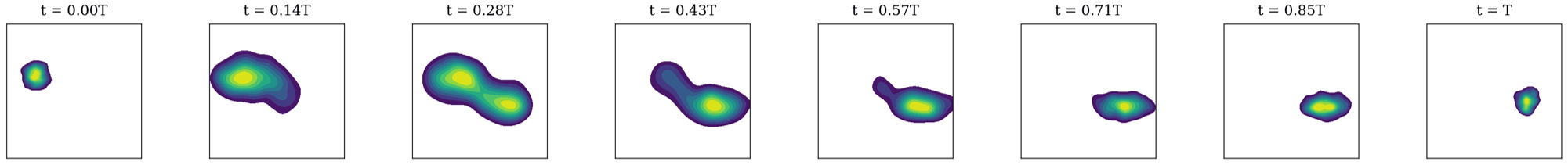}
    \caption{Estimated density $m(t,\cdot)$ snapshots, $t\!\in\![0,T]$.}
    \label{fig:potential_density}
  \end{subfigure}
\end{minipage}
\hfill
\begin{subfigure}[c]{0.33\textwidth}
  \centering
  \captionsetup{position=top, skip=0pt}
  \caption{$3$D evolution.}
  \label{fig:potential_3d}
  \includegraphics[width=\linewidth]{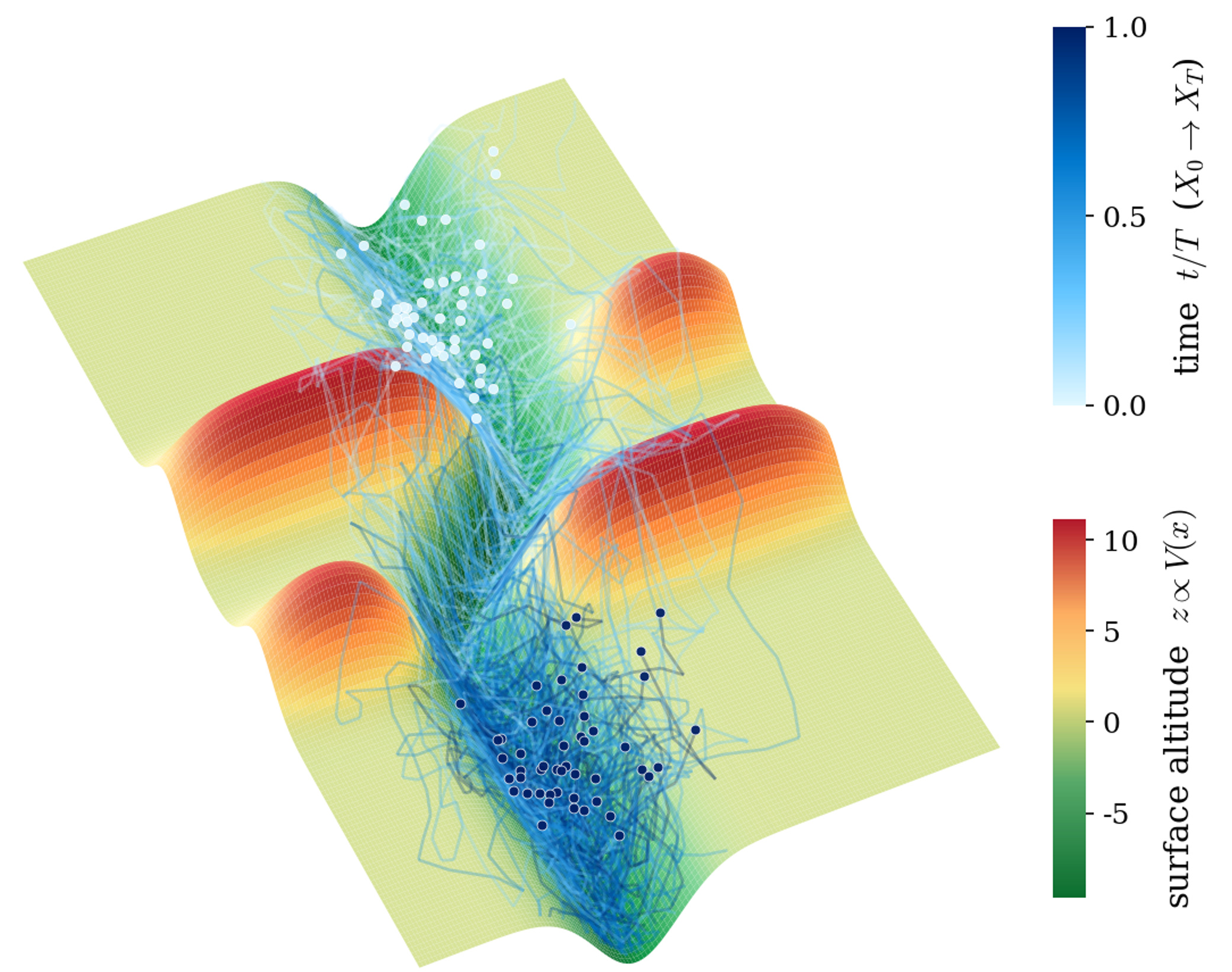}
\end{subfigure}
\end{figure}

\paragraph{MNIST in latent space ($F(m)=0$, $\lambda=\infty$, $d=32$).}
On MNIST we run the bridge in the $d{=}32$ latent space of a pretrained VAE on $(1,24,24)$ grayscale images, with $\mu_0,\mu_T$ the codes of digits $1$ and $3$ and $6{\times}10^3$ samples per class for training.
Figure~\ref{fig:image_diffusion_1_to_3} shows $\mathrm{Law}(X_t)$ and decoded trajectories; intermediate reconstructions stay visually smooth and digit-like.

\begin{figure}[H]
\centering
\includegraphics[width=0.8\textwidth]{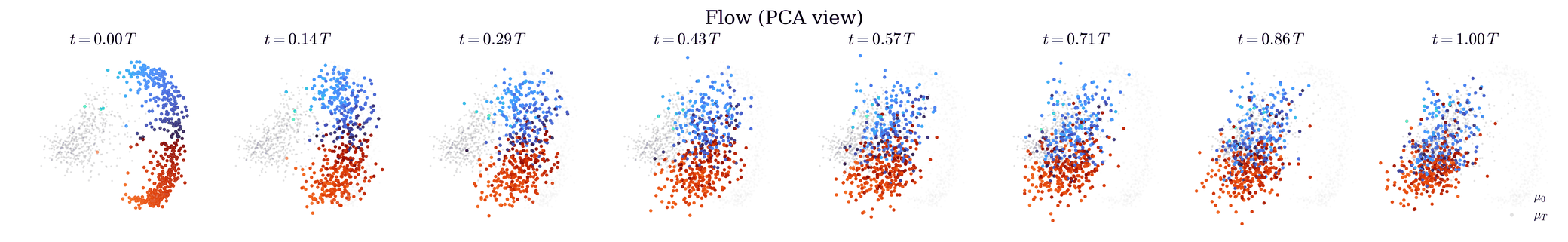}
\includegraphics[width=0.75\textwidth]{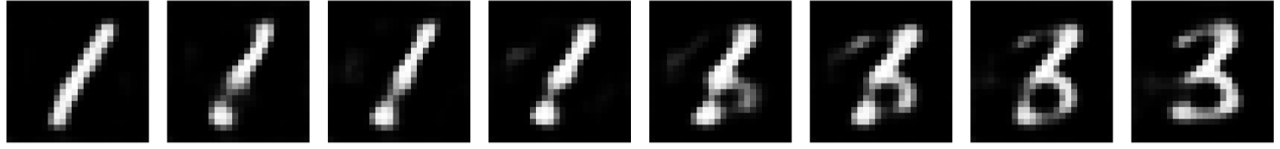}
\captionsetup{font=footnotesize}
\caption{Temporal evolution of the latent variables and their corresponding decoded representations.}
\label{fig:image_diffusion_1_to_3}
\end{figure}

\paragraph{FFHQ young-to-old ($F(m)=0$, $\lambda=\infty$, $d=512$).}
We evaluate FFHQ young-to-old transfer in the $d{=}512$ ALAE latent space
\citep{pidhorskyi2020adversariallatentautoencoders,
pham_alouadi_hl_loeper_touzi_bridging_preprint}, using $3{\times}10^4$
samples per age group. Figure~\ref{fig:ffhq_traj_and_comparison} shows
decoded trajectories and, for each baseline, the source child and its decoded
terminal sample. SBNL preserves identity and attributes, while several
baselines collapse to a near-canonical face; the extended panel is in
Appendix~\ref{app:comparison-ffhq}.

\begin{figure}[H]
\centering
\captionsetup{font=footnotesize}
\captionsetup[subfigure]{font=footnotesize}

\begin{minipage}[t]{0.55\textwidth}
  \vspace{0pt} 
  \centering

  \begin{subfigure}[t]{\linewidth}
    \centering
    \includegraphics[width=\linewidth]
    {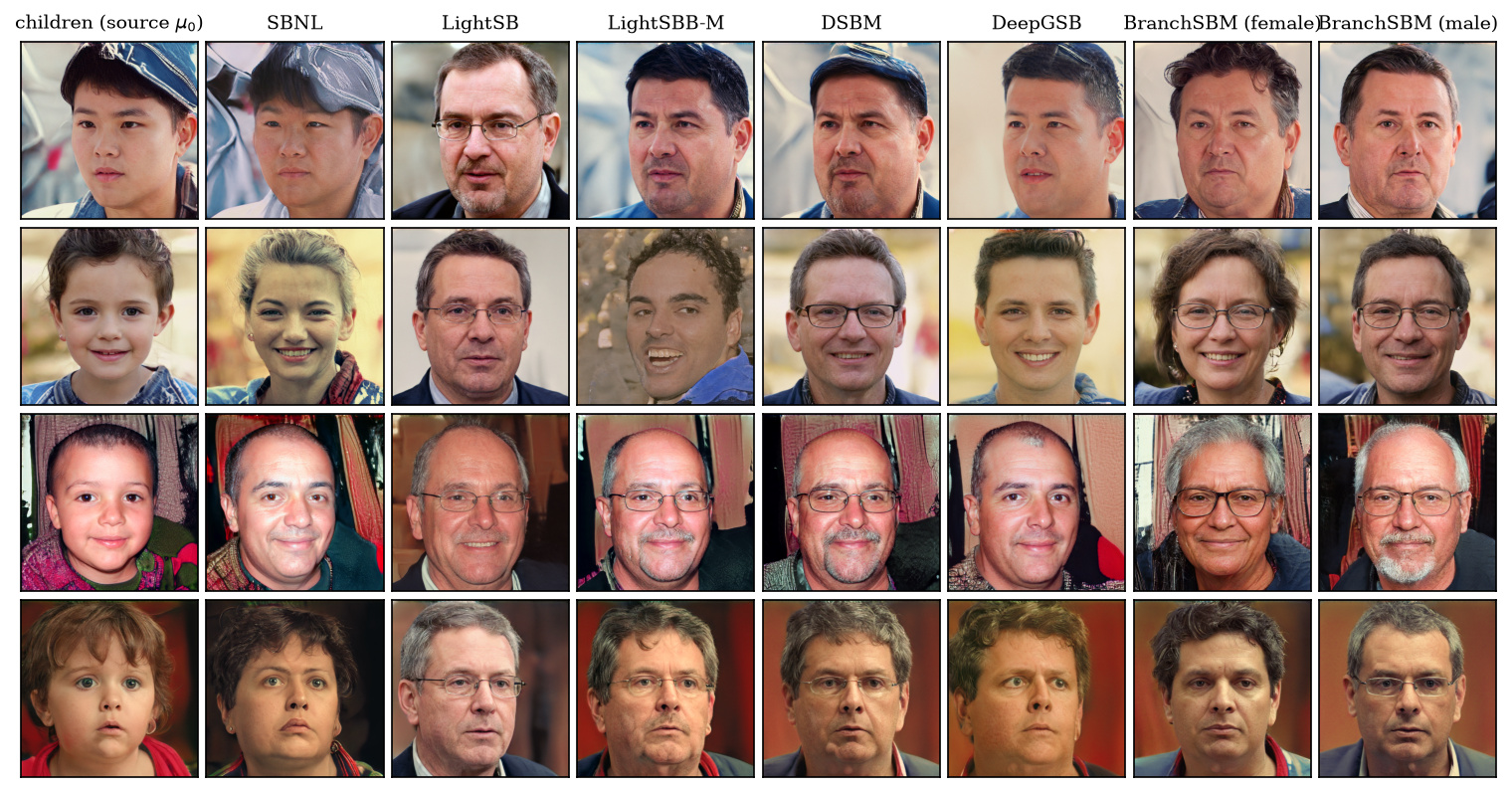}
    \caption{Comparison of source (left) and terminal samples.}
    \label{fig:comparison_ffhq_main}
  \end{subfigure}
\end{minipage}
\hfill
\begin{minipage}[t]{0.44\textwidth}
  \vspace{0pt} 
  \centering

  \begin{subfigure}[t]{\linewidth}
    \centering
    \includegraphics[width=\linewidth]
    {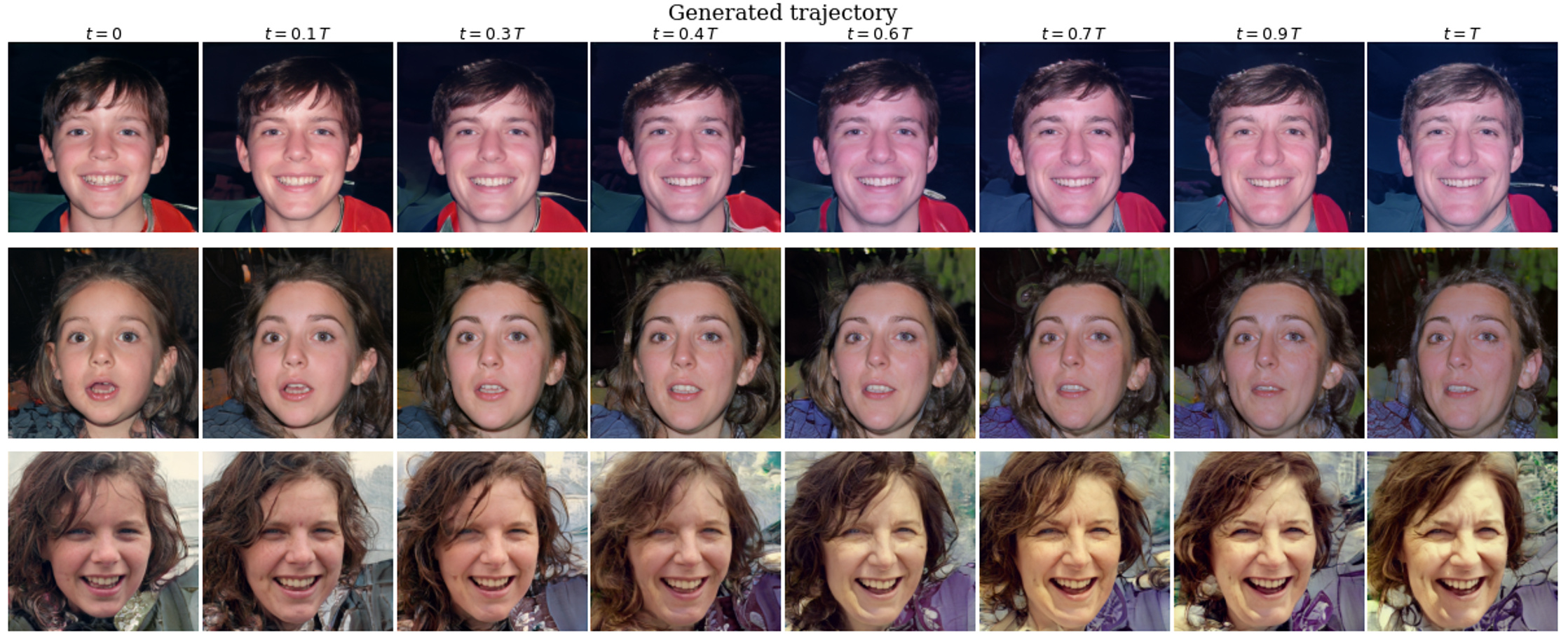}
    \caption{Decoded trajectories during sampling.}
    \label{fig:young_old_traj_small}
  \end{subfigure}

  \vspace{4pt}

{\scriptsize
\centering
\textbf{FID of the different methods}

\vspace{2pt}

\resizebox{\linewidth}{!}{%
\begin{tabular}{@{}ccccccc@{}}
  \toprule
  \textbf{SBNL}
  & \textbf{LightSB}
  & \textbf{LightSBB-M}
  & \textbf{DSBM}
  & \textbf{DeepGSB}
  & \textbf{BranchSBM-F}
  & \textbf{BranchSBM-M}
  \\
  \midrule
  \textbf{36.03}
  & 59.90
  & 40.76
  & 38.92
  & 66.39
  & 61.96
  & 39.54
  \\
  \bottomrule
\end{tabular}%
}
}
\end{minipage}

\caption{FFHQ young-to-old in ALAE latent space. Lower FID is better.}
\label{fig:ffhq_traj_and_comparison}
\end{figure}

\section{Discussion}

We formulate semilinear mean-field Schr\"odinger bridges with law-dependent $f(m(t,\cdot))$, $g(m(T,\cdot))$ and characterize optimality through coupled $(\varphi,\bar\varphi)$ PDEs, equivalently a Lasry - Lions $(m,v)$ system -- beyond the linear entropic setting of Sinkhorn and iterative Markovian fitting. We recast the terminal conditions as a \textit{nonlinear} Sinkhorn system, which we solve by Maximum Mean Discrepancy, using branching diffusions to obtain unbiased stochastic representations of $\varphi$. Our experiments cover bridge recovery, Gross-Pitaevskii-type couplings, potential interactions, and latent transport.

The numerics are encouraging, but we do not claim systematic superiority over mature diffusion Schr\"odinger-bridge pipelines in the linear setting. Training in the nonlinear regime still requires careful tuning of the hyperparameters (branching rate $c$, kernel bandwidths, penalty $\lambda$) to remain stable, and both the parallelization of the branching simulation and the variance of the underlying Monte Carlo estimators leave room for improvement. We see this work as a new variational and optimality system whose first mesh-free realization invites variance reduction and hybridization with standard training pipelines.

\bibliographystyle{unsrtnat}
\bibliography{references}

\newpage
\appendix

\section{Proof of Theorem \ref{theorem:optimal_solution}}
\label{demo-main-theorem}

\paragraph{Setting.}
Let $\mathcal A$ be the set of progressively measurable processes $\alpha$ with $\E\int_0^T|\alpha_t|^2\,\dd t<\infty$. For $\alpha\in\mathcal A$, let $X^\alpha_t:=X_0+\int_0^t\alpha_s\,\dd s+W_t$, where $X_0\sim\mu_0$ is independent of $W$, let $m^\alpha_t:=\mathrm{Law}(X^\alpha_t)$ (it has a density, since the relative entropy of $\mathrm{Law}(X^\alpha)$ with respect to the law of $X_0+W$ is at most $\frac12\E\int_0^T|\alpha_t|^2\,\dd t$), and
\[
J(\alpha):=\E\Big[\int_0^T\frac{1}{2c}|\alpha_t|^2\,\dd t\Big]+\int_0^T f(m^\alpha_t)\,\dd t+g(m^\alpha_T),
\qquad
P=\inf_{\alpha\in\mathcal A}J(\alpha).
\]
The functionals $f,g$ may take the value $+\infty$ (e.g.\ $f(m)=\frac\kappa2\int m^2$ when $m$ has no density in $L^2$); $f$ may also depend on $t$, in which case $F(t,\cdot,m)$ is the derivative of $f(t,\cdot)$. Throughout, $(v^*,m^*)$ denotes the solution of the forward-backward system given by (A4),
\begin{align}
&\partial_t v^*+\tfrac12\Delta v^*-\tfrac c2|\nabla_x v^*|^2+F\big(t,x,m^*_t\big)=0,
&& v^*(T,x)=G\big(x,m^*_T\big),
\label{eq:HJB}\\
&\partial_t m^*=c\,\nabla_x\!\cdot\!\big(m^*\,\nabla_x v^*\big)+\tfrac12\Delta m^*,
&& m^*_0=\mu_0,
\label{FP}
\end{align}
where $m^*_t:=m^*(t,\cdot)$, $F(t,x,m^*_t)$ depends on the whole density $m^*_t$, and \eqref{FP} is the Fokker-Planck equation of the feedback drift $-c\nabla_x v^*$.

\paragraph{Step 1: convexity.}
\begin{lemma}\label{lem:convexity}
Assume (A2). Let $m,m'$ be probability densities with $f(m')<\infty$ and $F(t,\cdot,m')\in L^1(m)\cap L^1(m')$. Then
\[
f(m)\;\ge\; f(m')+\int_{\R^d}F(t,x,m')\,\big(m(x)-m'(x)\big)\,\dd x ,
\]
and the same inequality holds for $g$ and $G$.
\end{lemma}

\begin{proof}
If $f(m)=+\infty$ there is nothing to prove. Otherwise set $m_\varepsilon:=m'+\varepsilon(m-m')$, $\varepsilon\in[0,1]$. By definition of the linear functional derivative, $f(m)-f(m')=\int_0^1\!\int F(t,x,m_\varepsilon)(m-m')(x)\,\dd x\,\dd\varepsilon$, hence
\begin{align*}
&f(m)-f(m')-\int F(t,x,m')(m-m')\,\dd x\\
&\qquad=\int_0^1\!\!\int\big(F(t,x,m_\varepsilon)-F(t,x,m')\big)(m-m')(x)\,\dd x\,\dd\varepsilon\\
&\qquad=\int_0^1\frac1\varepsilon\!\int\big(F(t,x,m_\varepsilon)-F(t,x,m')\big)\big(m_\varepsilon-m'\big)(x)\,\dd x\,\dd\varepsilon\;\ge\;0,
\end{align*}
because $m_\varepsilon-m'=\varepsilon(m-m')$ and each inner integral is nonnegative by (A2). The proof for $g$ is identical.
\end{proof}

\paragraph{Step 2: verification.}
Let $X^*$ be the solution of $\dd X^*_t=-c\nabla_x v^*(t,X^*_t)\,\dd t+\dd W_t$, $X^*_0=X_0$, which exists and is pathwise unique because $\nabla_x v^*$ is Lipschitz, and set $\alpha^*_t:=-c\nabla_x v^*(t,X^*_t)$. By (A3)-(A4), $\alpha^*\in\mathcal A$ and $m^{\alpha^*}_t=m^*_t$.

Fix $\alpha\in\mathcal A$ with $J(\alpha)<\infty$. Itô's formula applied to $v^*(t,X^\alpha_t)$, combined with \eqref{eq:HJB}, gives
\begin{align*}
v^*(T,X^\alpha_T)=v^*(0,X_0)&+\int_0^T\Big[\tfrac c2|\nabla_x v^*|^2+\alpha_t\cdot\nabla_x v^*-F(t,\cdot,m^*_t)\Big](t,X^\alpha_t)\,\dd t\\
&+\int_0^T\nabla_x v^*(t,X^\alpha_t)\cdot\dd W_t .
\end{align*}
Since $\E\sup_{t\le T}|X^\alpha_t|^2<\infty$ (by (A3) and $\alpha\in\mathcal A$), the growth conditions of (A4) make every term integrable and the stochastic integral a martingale, so that
\begin{equation}\label{eq:ito-v}
\begin{split}
\E\big[v^*(T,X^\alpha_T)\big]=\E\big[v^*(0,X_0)\big]&+\E\int_0^T\Big[\tfrac c2|\nabla_x v^*|^2+\alpha_t\cdot\nabla_x v^*\Big](t,X^\alpha_t)\,\dd t\\
&-\E\int_0^T F(t,X^\alpha_t,m^*_t)\,\dd t .
\end{split}
\end{equation}
Lemma~\ref{lem:convexity} with $m=m^\alpha_t$, $m'=m^*_t$ (for $f$) and $m=m^\alpha_T$, $m'=m^*_T$ (for $g$), together with the terminal condition $G(\cdot,m^*_T)=v^*(T,\cdot)$, yields
\begin{align*}
\int_0^T\!f(m^\alpha_t)\,\dd t+g(m^\alpha_T)
\;\ge\;&
\int_0^T\!f(m^*_t)\,\dd t+g(m^*_T)
+\E\!\int_0^T\!\Big[F(t,X^\alpha_t,m^*_t)-F(t,X^*_t,m^*_t)\Big]\dd t\\
&+\E\big[v^*(T,X^\alpha_T)-v^*(T,X^*_T)\big].
\end{align*}
We insert \eqref{eq:ito-v} for $\alpha$ and for $\alpha^*$ (both processes start from $X_0$): the terms in $F$ cancel, and we obtain
\begin{align}
J(\alpha)-J(\alpha^*)
&\ge
\E\!\int_0^T\!\Big[\tfrac1{2c}|\alpha_t|^2+\alpha_t\cdot\nabla_x v^*+\tfrac c2|\nabla_x v^*|^2\Big](t,X^\alpha_t)\,\dd t
\nonumber\\
&\quad-\E\!\int_0^T\!\Big[\tfrac1{2c}|\alpha^*_t|^2+\alpha^*_t\cdot\nabla_x v^*+\tfrac c2|\nabla_x v^*|^2\Big](t,X^*_t)\,\dd t
\nonumber\\
&=\frac1{2c}\,\E\int_0^T\big|\alpha_t+c\nabla_x v^*(t,X^\alpha_t)\big|^2\,\dd t\;\ge\;0,
\label{eq:verification}
\end{align}
where we used $\tfrac1{2c}|a|^2+a\cdot p+\tfrac c2|p|^2=\tfrac1{2c}|a+cp|^2$ and $\alpha^*_t+c\nabla_x v^*(t,X^*_t)=0$. Hence $J(\alpha^*)\le J(\alpha)$ for every $\alpha\in\mathcal A$: the control $\alpha^*$ is optimal and $P=J(\alpha^*)$.

\paragraph{Step 3: uniqueness of the optimal control.}
If $J(\alpha)=J(\alpha^*)$, then \eqref{eq:verification} forces $\alpha_t=-c\nabla_x v^*(t,X^\alpha_t)$, $\dd t\otimes\dd\PP$-a.e. Thus $X^\alpha$ solves the equation of $X^*$ with the same initial condition and the same Brownian motion, and pathwise uniqueness gives $X^\alpha=X^*$, hence $\alpha=\alpha^*$.

\paragraph{Step 4: Cole-Hopf transform.}
Set $\varphi:=e^{-cv^*}>0$. Then $\partial_t v^*=-\frac{\partial_t\varphi}{c\varphi}$, $\nabla_x v^*=-\frac{\nabla_x\varphi}{c\varphi}$ and $\Delta v^*=-\frac{\Delta\varphi}{c\varphi}+\frac{|\nabla_x\varphi|^2}{c\varphi^2}$, so that the terms in $|\nabla_x\varphi|^2$ cancel:
\[
\partial_t v^*+\tfrac12\Delta v^*-\tfrac c2|\nabla_x v^*|^2=-\frac{1}{c\varphi}\Big(\partial_t\varphi+\tfrac12\Delta\varphi\Big).
\]
Hence \eqref{eq:HJB} is equivalent to $\partial_t\varphi+\frac12\Delta\varphi=c\,\varphi\,F(t,x,m^*_t)$, and the optimal feedback is $-c\nabla_x v^*=\nabla_x\log\varphi$, which is \eqref{eq:sde-process}.

\paragraph{Step 5: factorization of the optimal law.}
Set $\psi:=m^*/\varphi$. Expanding $m^*=\varphi\psi$ and using the equation of Step~4,
\begin{align*}
\partial_t m^*+\nabla_x\!\cdot\!\big(m^*\nabla_x\log\varphi\big)-\tfrac12\Delta m^*
&=\psi\Big(\partial_t\varphi+\tfrac12\Delta\varphi\Big)+\varphi\Big(\partial_t\psi-\tfrac12\Delta\psi\Big)\\
&=\varphi\Big(\partial_t\psi-\tfrac12\Delta\psi+c\,\psi\,F(t,x,m^*_t)\Big),
\end{align*}
where the cross terms $\pm\nabla_x\varphi\cdot\nabla_x\psi$ cancel. As $c\nabla_x\!\cdot(m^*\nabla_x v^*)=-\nabla_x\!\cdot(m^*\nabla_x\log\varphi)$ and $\varphi>0$, \eqref{FP} is equivalent to
\[
\partial_t\psi-\tfrac12\Delta\psi=-c\,\psi\,F(t,x,m^*_t),\qquad m^*_t=\varphi(t,\cdot)\,\psi(t,\cdot).
\]
The time reversal $\bar\varphi(t,x):=\psi(T-t,x)$ turns this forward equation into the backward equation $\partial_t\bar\varphi+\frac12\Delta\bar\varphi=c\,\bar\varphi\,F(T-t,x,m^*_{T-t})$. Since $m^*_t=\varphi(t,\cdot)\bar\varphi(T-t,\cdot)$ and $m^*_{T-t}=\bar\varphi(t,\cdot)\varphi(T-t,\cdot)$, the pair $(\varphi,\bar\varphi)$ solves \eqref{PDE1}-\eqref{PDE2}, and $\mathrm{Law}(X^*_t)=m^*_t=\varphi(t,\cdot)\bar\varphi(T-t,\cdot)$.

\paragraph{Step 6: boundary conditions.}
The initial condition $m^*_0=\mu_0$ reads $\varphi(0,\cdot)\bar\varphi(T,\cdot)=\mu_0$. The terminal condition $v^*(T,\cdot)=G(\cdot,m^*_T)$, with $v^*(T,\cdot)=-\frac1c\log\varphi(T,\cdot)$ and $m^*_T=\bar\varphi(0,\cdot)\varphi(T,\cdot)$, reads $G\big(\cdot,\bar\varphi(0,\cdot)\varphi(T,\cdot)\big)=-\frac1c\log\varphi(T,\cdot)$, i.e.
\[
\bar\varphi(0,x)\,\varphi(T,x)=G^{-1}\big(-\tfrac1c\log\varphi(T,\cdot),x\big),
\]
where $G^{-1}(w,\cdot)$ is the density $m$ such that $G(\cdot,m)=w$. It is unique when the second inequality of (A2) is strict for $m_1\neq m_2$: $G(\cdot,m_1)=G(\cdot,m_2)$ forces $\int(G(x,m_1)-G(x,m_2))(m_1-m_2)(x)\,\dd x=0$. This is the nonlinear Sinkhorn system \eqref{eq:nl-sinkhorn}, which completes the proof.
\qed

\paragraph{Hard terminal constraint.}
For the classical constraint $m_T=\mu_T$ ($g=0$ at $\mu_T$ and $+\infty$ elsewhere), $\mathcal A$ is restricted to the controls with $m^\alpha_T=\mu_T$, and the terminal condition of \eqref{eq:HJB} is replaced by $m^*_T=\mu_T$, $v^*(T,\cdot)$ becoming an unknown. The proof is unchanged, except that in Step~2 the terminal terms vanish: $g(m^\alpha_T)=g(m^*_T)=0$ and $\E[v^*(T,X^\alpha_T)]=\E[v^*(T,X^*_T)]$ because $X^\alpha_T$ and $X^*_T$ both have law $\mu_T$. Step~6 then gives $\bar\varphi(0,\cdot)\varphi(T,\cdot)=\mu_T$, i.e.\ \eqref{eq:soft-terminal-lambda-infty}.

\paragraph{Origin of the system: Lagrangian duality.}
The system \eqref{eq:HJB}-\eqref{FP} is the optimality system of the following dual problem. For $v\in C^{1,2}$ with the growth of (A4), let
\[
\mathcal L v:=\partial_t v-\frac c2|\nabla_x v|^2+\frac12\Delta v,
\]
\[
f^*(\psi):=\inf_m\Big\{f(m)+\int\psi\,m\Big\},\qquad
g^*(\phi):=\sup_m\Big\{\int\phi\,m-g(m)\Big\}.
\]
For $\alpha\in\mathcal A$, Itô's formula gives $\E[v(T,X^\alpha_T)]-\E[v(0,X_0)]=\E\int_0^T(\partial_t v+\alpha_t\cdot\nabla_x v+\frac12\Delta v)(t,X^\alpha_t)\,\dd t$. Adding this zero quantity to $J(\alpha)$ and minimizing pointwise in $\alpha_t$, which gives $\alpha_t=-c\nabla_x v$ and $\inf_a\{\frac1{2c}|a|^2+a\cdot\nabla_x v\}=-\frac c2|\nabla_x v|^2$, we get the weak duality
\begin{equation}\label{eq:dual_problem}
P\;\ge\;\int_0^T f^*\big(\mathcal L v(t,\cdot)\big)\,\dd t-g^*\big(v(T,\cdot)\big)+\E_{X_0\sim\mu_0}\big[v(0,X_0)\big].
\end{equation}
The first-order conditions of the infimum defining $f^*(\mathcal L v)$ and of the supremum defining $g^*(v(T,\cdot))$ are $F(t,\cdot,m_t)=-\mathcal L v$, which is the HJB equation \eqref{eq:HJB}, and $v(T,\cdot)=G(\cdot,m_T)$, which is its terminal condition. Conversely, for $v=v^*$, Lemma~\ref{lem:convexity} shows that the infimum and the supremum are attained at $m^*_t$ and $m^*_T$, and Itô's formula for $\alpha^*$ shows that the right-hand side of \eqref{eq:dual_problem} equals $J(\alpha^*)=P$: there is no duality gap.

\section{Stochastic-optimization forms of the \emph{nonlinear} Sinkhorn system}
\label{app:stochastic-opt-forms}

Using the branching representation of Proposition~\ref{prop:phi-branch-lamc}, the nonlinear Sinkhorn conditions read
\begin{equation}
  \left\{
  \begin{aligned}
    \E_{x, \text{white}}\!\left[\mathcal{W}\prod_{i=1}^{N_T^{\mathrm{white}}} \varphi(T,X_T^i)\prod_{j=1}^{\bar N_T^{\mathrm{black}}} \bar\varphi(T,\bar X_T^j)\right]\bar{\varphi}(T, x) &= \mu_0(x),\\
    \E_{x, \text{black}}\!\left[\mathcal{W}\prod_{i=1}^{N_T^{\mathrm{white}}} \varphi(T,X_T^i)\prod_{j=1}^{\bar N_T^{\mathrm{black}}} \bar\varphi(T,\bar X_T^j)\right]\varphi(T, x) &= G^{-1}(-\tfrac{1}{c} \log \varphi(T, \cdot), x).
  \end{aligned}
  \right.
\end{equation}

We now recast the two nonlinear Sinkhorn conditions in sample-based forms that can be used to train neural terminal potentials. We parametrize the terminal potentials $\varphi(T,\cdot),\bar\varphi(T,\cdot)$ by neural networks and, with a slight abuse of notation, write
\[
  \varphi_\theta(x):=\varphi(T,x), 
  \qquad
  \bar\varphi_\eta(x):=\bar\varphi(T,x),
\]
and let $\mu^*$ be a reference probability density on $\mathbb{R}^d$ that is easy to sample and evaluate.

Following the notation of Section~\ref{sec:meta_algorithm_train_ssb}, for $Z\sim\mu^*$ a white- (resp.\ black-) rooted branching tree started at $(0,Z)$ yields the unbiased one-tree estimators $\widehat A_{\theta,\eta}(Z,\circ)$ (resp.\ $\widehat A_{\theta,\eta}(Z,\bullet)$), with
\begin{equation}
  \left\{
  \begin{aligned}
    \E\!\big[\widehat A_{\theta,\eta}(Z,\circ)\big]
    &= A_{\theta,\eta}(Z,\circ)\approx\varphi(0,Z),\\
    \E\!\big[\widehat A_{\theta,\eta}(Z,\bullet)\big]
    &= A_{\theta,\eta}(Z,\bullet)\approx\bar\varphi(0,Z),
  \end{aligned}
  \right.
\end{equation}
where the approximation replaces the true terminal potentials by their neural-network parametrization $(\varphi_\theta,\bar\varphi_\eta)$.
The two boundary equations \eqref{eq:nl-sinkhorn} thus read
\begin{equation}
  \left\{
  \begin{aligned}
    A_{\theta,\eta}(x,\circ)\,\bar\varphi_\eta(x) &= \mu_0(x),\\
    A_{\theta,\eta}(x,\bullet)\,\varphi_\theta(x) &= G^{-1}\!\big(-\tfrac{1}{c}\log\varphi_\theta,x\big).
  \end{aligned}
  \right.
\end{equation}
In the KL terminal case of \eqref{eq:gstar-kl}, the second condition becomes $ A_{\theta,\eta}(x,\bullet)\,\varphi_\theta(x) = \mu_T(x)\,\varphi_\theta(x)^{-a}$ with $a:=1/(\lambda c)$.

\paragraph{Possibility 1: RKHS/MMD matching.}
Each boundary equation matches a (reweighted) measure to a marginal, so we enforce them by minimizing the MMD between the two sides. Using a characteristic kernel $k$ with RKHS $\mathcal H$ (Appendix~\ref{app:rkhs_mmd}), we set
\[
  \mathcal L_{\mathrm{MMD}}(\theta,\eta)
  :=
  \mathrm{MMD}_k^2\!\big(A_{\theta,\eta}(\cdot,\circ)\,\bar\varphi_\eta,\;\mu_0\big)
  +
  \mathrm{MMD}_k^2\!\big(A_{\theta,\eta}(\cdot,\bullet)\,\varphi_\theta,\;\varphi_\theta^{-a}\mu_T\big).
\]
Estimating $A_{\theta,\eta}(\cdot,\circ),A_{\theta,\eta}(\cdot,\bullet)$ by the branching trees and importance-sampling against $\mu^*$, the closed form of Definition~\ref{def:mmd} gives the sample-based objective \eqref{eq:J-mc}-\eqref{eq:J-mc-detailed}, minimized directly by stochastic gradient descent (Algorithm~\ref{alg:meta-train-ssb}); a mixture of Gaussian kernels is a stable default.

\section{RKHS and MMD}
\label{app:rkhs_mmd}

For completeness, we recall the definitions and elementary properties of reproducing kernel Hilbert spaces (RKHS) and the maximum mean discrepancy (MMD) used in Section~\ref{sec:meta_algorithm_train_ssb} and in the proof of Theorem~\ref{theorem:optimal_solution}. We refer to \citep{JMLR:v13:gretton12a} for a detailed treatment.

\paragraph{Reproducing kernel Hilbert space.}
Let $k:\R^d\times\R^d\to\R$ be a symmetric positive definite kernel, i.e.\ for any $n\in\mathbb N$, points $x_1,\dots,x_n\in\R^d$ and scalars $c_1,\dots,c_n\in\R$,
\[
  \sum_{i,j=1}^n c_i c_j\, k(x_i,x_j)\;\ge\;0.
\]
By the Moore-Aronszajn theorem, there exists a unique Hilbert space $\mathcal H_k$ of functions $f:\R^d\to\R$, called the \emph{reproducing kernel Hilbert space} of $k$, such that:
\begin{itemize}
  \item for every $x\in\R^d$, the function $k(x,\cdot)\in\mathcal H_k$;
  \item (\emph{reproducing property}) for every $f\in\mathcal H_k$ and every $x\in\R^d$,
  \[
    f(x)\;=\;\langle f,\,k(x,\cdot)\rangle_{\mathcal H_k}.
  \]
\end{itemize}
In particular, taking $f=k(y,\cdot)$ gives $\langle k(x,\cdot),k(y,\cdot)\rangle_{\mathcal H_k}=k(x,y)$, so that $k$ acts as an inner product between feature maps $x\mapsto k(x,\cdot)$.

\paragraph{Maximum mean discrepancy.}
Given two probability measures $P,Q$ on $\R^d$, the \emph{maximum mean discrepancy} associated with $k$ is defined as the largest difference in expectations over the unit ball of $\mathcal H_k$:
\begin{equation}
  \label{eq:mmd-def}
  \mathrm{MMD}_k(P,Q)
  \;:=\;
  \sup_{\|f\|_{\mathcal H_k}\le 1}
  \Big\{\;\E_{X\sim P}\!\big[f(X)\big]-\E_{Y\sim Q}\!\big[f(Y)\big]\;\Big\}.
\end{equation}
Intuitively, MMD searches for a witness function $f$ in the RKHS unit ball that best distinguishes samples from $P$ and $Q$.

\paragraph{Kernel mean embedding.}
Assume the integrability condition $\E_{X\sim P}[\sqrt{k(X,X)}]<\infty$. The \emph{kernel mean embedding} of $P$ is the element of $\mathcal H_k$ defined by
\begin{equation}
  \label{eq:kme-def}
  \mu_P(\cdot)\;:=\;\E_{X\sim P}\!\big[k(X,\cdot)\big]\;\in\;\mathcal H_k.
\end{equation}
By the reproducing property, it satisfies $\E_{X\sim P}[f(X)]=\langle f,\mu_P\rangle_{\mathcal H_k}$ for all $f\in\mathcal H_k$, i.e.\ integration against $P$ is represented by an inner product with $\mu_P$.

\paragraph{MMD as a distance between embeddings.}
Substituting $\E_P[f]=\langle f,\mu_P\rangle_{\mathcal H_k}$ and $\E_Q[f]=\langle f,\mu_Q\rangle_{\mathcal H_k}$ into \eqref{eq:mmd-def}, the supremum over the unit ball is attained at $f=(\mu_P-\mu_Q)/\|\mu_P-\mu_Q\|_{\mathcal H_k}$, and the MMD reduces to the RKHS distance between the mean embeddings:
\begin{equation}
  \label{eq:mmd-embedding}
  \mathrm{MMD}_k(P,Q)\;=\;\big\|\mu_P-\mu_Q\big\|_{\mathcal H_k}.
\end{equation}

\paragraph{Expansion via the reproducing property.}
Expanding the squared norm in \eqref{eq:mmd-embedding} and using $\langle k(x,\cdot),k(y,\cdot)\rangle_{\mathcal H_k}=k(x,y)$ gives the closed-form expression
\begin{equation}
  \label{eq:mmd-expansion}
  \mathrm{MMD}_k^2(P,Q)
  \;=\;
  \E_{X,X'\sim P}\!\big[k(X,X')\big]
  +\E_{Y,Y'\sim Q}\!\big[k(Y,Y')\big]
  -2\,\E_{\substack{X\sim P\\ Y\sim Q}}\!\big[k(X,Y)\big],
\end{equation}
where $X,X'$ (resp.\ $Y,Y'$) are independent copies drawn from $P$ (resp.\ $Q$). This form only involves expectations of kernel evaluations and is therefore directly estimable from samples, which is what makes the objective \eqref{eq:J-mc} tractable.

\paragraph{Characteristic kernels.}
The kernel $k$ is said to be \emph{characteristic} if the embedding map $P\mapsto\mu_P$ is injective on the space of probability measures. Equivalently,
\[
  \mathrm{MMD}_k(P,Q)=0
  \quad\Longleftrightarrow\quad
  P=Q.
\]
Classical examples of characteristic kernels on $\R^d$ include the Gaussian kernel $k(x,y)=\exp(-\|x-y\|^2/2\sigma^2)$ and, more generally, mixtures of Gaussian kernels across several bandwidths, which we use as a stable default in our experiments.

\paragraph{MMD between finite measures.}
The objective of Proposition~\ref{prop:mmd_objective} compares \emph{unnormalized} measures, for which we need the following version of the characteristic property.

\begin{lemma}\label{lem:mmd-finite-measures}
Let $k(x,y)=k_0(x-y)$ with $k_0(z)=\int_{\R^d}e^{\mathrm i\,\omega\cdot z}\,\psi(\omega)\,\dd\omega$ for an integrable function $\psi>0$; for instance the Gaussian kernel $k_0(z)=e^{-|z|^2/2\sigma^2}$, for which $\psi(\omega)=(\sigma^2/2\pi)^{d/2}e^{-\sigma^2|\omega|^2/2}$, or any finite positive mixture of Gaussian kernels. For finite nonnegative measures $P,Q$ on $\R^d$, let $\mathrm{MMD}_k^2(P,Q):=\iint k(x,y)\,(P-Q)(\dd x)\,(P-Q)(\dd y)$, which is \eqref{eq:mmd-expansion} with expectations replaced by integrals. Then
\[
\mathrm{MMD}_k^2(P,Q)=\int_{\R^d}\big|\widehat P(\omega)-\widehat Q(\omega)\big|^2\,\psi(\omega)\,\dd\omega,
\qquad
\widehat P(\omega):=\int_{\R^d}e^{\mathrm i\,\omega\cdot x}\,P(\dd x),
\]
so that $\mathrm{MMD}_k(P,Q)\ge0$, with equality if and only if $P=Q$.
\end{lemma}

\begin{proof}
Let $\nu:=P-Q$, a finite signed measure. Since $|e^{\mathrm i\,\omega\cdot(x-y)}\psi(\omega)|\le\psi(\omega)$ is integrable with respect to $\dd\omega\otimes|\nu|(\dd x)\otimes|\nu|(\dd y)$, Fubini's theorem gives
\begin{align*}
\iint k(x,y)\,\nu(\dd x)\,\nu(\dd y)
&=\int\psi(\omega)\Big(\int e^{\mathrm i\,\omega\cdot x}\nu(\dd x)\Big)\Big(\int e^{-\mathrm i\,\omega\cdot y}\nu(\dd y)\Big)\dd\omega\\
&=\int|\widehat\nu(\omega)|^2\,\psi(\omega)\,\dd\omega .
\end{align*}
If this vanishes, then $\widehat\nu=0$ almost everywhere because $\psi>0$, hence everywhere because $\widehat\nu$ is continuous, and $\nu=0$ because the Fourier transform is injective on finite signed measures.
\end{proof}

\paragraph{Reweighted kernel.}
Let $\mu^*$ be a bounded density with $\mu^*>0$ almost everywhere and $\tilde k(x,y):=\mu^*(x)\,k(x,y)\,\mu^*(y)$. For integrable densities $p,q$, $\mathrm{MMD}^2_{\tilde k}(p,q)=\mathrm{MMD}^2_k(\mu^*p,\mu^*q)$, which by Lemma~\ref{lem:mmd-finite-measures} vanishes if and only if $\mu^*p=\mu^*q$ a.e., i.e.\ $p=q$ a.e. This is the property used for the objective \eqref{eq:J-mc}.

\paragraph{Proof of Proposition~\ref{prop:mmd_objective}.}
$\mathcal L_{\mathrm{MMD}}$ is the sum of two terms that are nonnegative by Lemma~\ref{lem:mmd-finite-measures}; hence $\mathcal L_{\mathrm{MMD}}\ge0$, with equality if and only if both terms vanish. By Lemma~\ref{lem:mmd-finite-measures}, applied to the finite measures with densities $\varphi(0,\cdot)\bar\varphi(T,\cdot)$ and $\mu_0$ for the first term, and $\bar\varphi(0,\cdot)\varphi(T,\cdot)$ and $G^{-1}(-\frac1c\log\varphi(T,\cdot),\cdot)$ for the second, this holds if and only if both equations of \eqref{eq:nl-sinkhorn} hold almost everywhere.
\qed

\section{Standing assumptions for well-posedness}
\label{app:standing-assumptions}

\noindent
(A1)-(A3) are structural conditions on the data; (A4) collects the properties of the solution of the forward-backward system \eqref{eq:HJB}-\eqref{FP} that are used in the proof of Theorem~\ref{theorem:optimal_solution} (Appendix~\ref{demo-main-theorem}). That proof is a verification argument: it uses (A1)-(A2), through the convexity of $f$ and $g$ (Lemma~\ref{lem:convexity}), together with (A3)-(A4). The existence of a solution with the properties of (A4) is not part of the argument. For regularizing nonlocal couplings it is established in \citep{lasry2007mean,cardaliaguet2013notesMFG} under regularity and growth conditions on $F$ and $G$, to which we refer for precise statements. Local couplings such as $F(m)=\kappa m$ or $F(m)=\kappa\log m$ satisfy (A1)-(A2) but not those conditions and require separate arguments, which we do not reproduce; in the Gaussian case of Proposition~\ref{propPHL} the solution is explicit.

\begin{itemize}
\item[(A1)] \emph{Differentiability.} The functionals $f$ (possibly depending on $t$) and $g$ admit linear functional derivatives
\[
F(t,x,m):=\frac{\delta f}{\delta m(t,x)},\qquad
G(x,m):=\frac{\delta g}{\delta m(T,x)},
\]
measurable in $(t,x)$, in the sense that $f(m)-f(m')=\int_0^1\!\int F\big(t,x,m'+\varepsilon(m-m')\big)\,(m-m')(x)\,\dd x\,\dd\varepsilon$ for all probability densities $m,m'$ at which $f$ is finite, and similarly for $g$ and $G$.
\item[(A2)] \emph{Lasry - Lions monotonicity.} For any probability densities $m_1,m_2$,
\[
\int_{\R^d}\!\Big(F(t,x,m_1)-F(t,x,m_2)\Big)\,\big(m_1(x)-m_2(x)\big)\,\dd x \;\ge 0,\quad \forall t\in[0,T],
\]
and
\[
\int_{\R^d}\!\Big(G(x,m_1)-G(x,m_2)\Big)\,\big(m_1(x)-m_2(x)\big)\,\dd x \;\ge 0.
\]
These monotonicity conditions ensure uniqueness of the mean-field equilibrium \citep{lasry2007mean}.
\item[(A3)] \emph{Initial data.} The initial law $m(0,\cdot)=\mu_0$ has sufficient regularity (e.g., $\mu_0\in L^1\cap L^\infty$ with finite second moment).
\item[(A4)] \emph{Regularity of the forward-backward solution.} The system \eqref{eq:HJB}-\eqref{FP} has a solution $(v^*,m^*)$ such that $v^*\in C^{1,2}([0,T]\times\R^d)$, $\nabla_x v^*$ is Lipschitz in $x$ uniformly in $t$, $|v^*|+|\partial_t v^*|+|\Delta v^*|\le C(1+|x|^2)$, and $m^*_t=\mathrm{Law}(X^*_t)$, where $\dd X^*_t=-c\nabla_x v^*(t,X^*_t)\,\dd t+\dd W_t$, $X^*_0\sim\mu_0$. Moreover $x\mapsto F(t,x,m^*_t)$ and $x\mapsto G(x,m^*_T)$ grow at most quadratically, uniformly in $t$, and $\int_0^T f(m^*_t)\,\dd t+g(m^*_T)<\infty$.
\end{itemize}
In (A2), the condition on $G$ is called \emph{strict} if equality implies $m_1=m_2$; it makes $G^{-1}$ in \eqref{eq:nl-sinkhorn} well defined.

\section{State dependence and relation to mean-field games}
\label{app:state-mfg}

One could allow the running and terminal costs to depend explicitly on the state variable $x$, e.g.
\[
\int_0^T \!\!\int_{\R^d} \ell\big(t,x, m(t,\cdot)\big)\, m(t,\dd x)\,\dd t
\quad\text{and}\quad
\int_{\R^d} h\big(x, m(T,\cdot)\big)\, m(T,\dd x).
\]
This causes no loss of generality: by setting
\[
f\big(m(t,\cdot)\big):=\!\int_{\R^d}\!\ell\big(t,x, m(t,\cdot)\big)\, m(t,\dd x),
\qquad
g\big(m(T,\cdot)\big):=\!\int_{\R^d}\!h\big(x, m(T,\cdot)\big)\, m(T,\dd x),
\]
we recover exactly the functional form \eqref{meanfield}, so that any explicit $x$-dependence can be ``integrated out'' into a dependence on $m(t,\cdot)$ and $m(T,\cdot)$.
The resulting stochastic optimal control with mean-field (law) costs is standard in the mean-field game (MFG) literature; the corresponding optimality system is a forward-backward pair of PDEs (Fokker-Planck for $m$ and HJB for the value), coupled through the functional derivatives $F=\delta f/\delta m$ and $G=\delta g/\delta m$ (see, e.g., \citep{lasry2007mean}) as made explicit in Theorem~\ref{theorem:optimal_solution}.

\noindent
Altogether, the functional \eqref{meanfield} extends the original entropic optimal transport problem of Schr\"odinger to a rich class of mean-field type control problems.
It preserves the central feature that the quadratic control cost corresponds to entropy minimization, thereby maintaining analytic tractability.

\section{Soft terminal constraint via Wasserstein cost}
\label{app:wasserstein-terminal}

\begin{remark}[Soft constraint via Wasserstein cost]
In analogy with the $f$-divergence penalization of Remark~\ref{rmk:matching_marginals}, one may also
impose a soft terminal condition through an optimal transport discrepancy.
Specifically, taking
\[
g(m) \;=\; \lambda\, W_p^p(\mu_T,m), \qquad \lambda > 0,
\]
the convex conjugate $g^*$ can be expressed
in terms of the $c$-transform associated with the cost $|x-y|^p$:
\[
g^*(v)
=\E_{X_T\sim\mu_T}\Big[\sup_{y\in\R^d}\big\{v(y)-\lambda|X_T-y|^p\big\}\Big]
= -\,\lambda \,\E_{X_T \sim \mu_T}\!\big[ (v/\lambda)^{c}(X_T)\big],
\]
with the $c$-transform $w^{c}(x):=\inf_{y\in\mathbb{R}^d}\big\{|x-y|^p - w(y)\big\}$.
Indeed, writing $W_p^p(\mu_T,m)=\inf_{\pi\in\Pi(\mu_T,m)}\int|x-y|^p\,\pi(\dd x,\dd y)$, the supremum over $m$ in $g^*(v)=\sup_m\sup_{\pi\in\Pi(\mu_T,m)}\int\big(v(y)-\lambda|x-y|^p\big)\pi(\dd x,\dd y)$ frees the second marginal of $\pi$, and the remaining supremum over couplings with first marginal $\mu_T$ equals $\int\sup_y\{v(y)-\lambda|x-y|^p\}\,\mu_T(\dd x)$ (transport each $x$ to a measurable $\varepsilon$-maximizer of $y\mapsto v(y)-\lambda|x-y|^p$).
In practice, directly computing $v^{\,c}$ in $\mathbb{R}^d$ may be challenging;
a common approach is to consider an entropic relaxation of the transport cost,
which yields a simpler expression more amenable to computation while
retaining the same geometric intuition.
\end{remark}

\section{Link to the (complex) semilinear Schr\"odinger equation and Gross-Pitaevskii}
\label{app:gross-pitaevskii}

Consider the classical semilinear Schr\"odinger equation (NLS) on physical time \(\tau\in\mathbb{R}\) (and its complex conjugate):
\begin{eqnarray}
i\,\partial_{\tau}\psi &= -\tfrac{1}{2}\Delta \psi \;+\; c\,\psi \, F(|\psi|^{2}), \label{eq:NLS}\\
-\,i\,\partial_{\tau}\bar\psi &= -\tfrac{1}{2}\Delta \bar\psi \;+\; c\,\bar\psi \; F(|\psi|^{2}), \label{eq:NLSconj}
\end{eqnarray}
where \(|\psi|^{2}=\bar\psi\,\psi\) and \(F\) is a (real-valued) nonlinearity acting on the density.
The Gross-Pitaevskii (cubic NLS) corresponds to the choice \(F(u)=u\) (see e.g.\ \citep{diu1989mecanique}):
\[
i\,\partial_{\tau}\psi \;=\; -\tfrac{1}{2}\Delta \psi \;+\; c\, \psi \,|\psi|^{2}.
\]

\medskip
\noindent
Perform a Wick rotation to imaginary time by setting \(\tau=\,i\,t\) (so that \(\partial_{\tau}=-i\,\partial_{t}\)).
Substituting into \eqref{eq:NLS}-\eqref{eq:NLSconj} yields
\begin{align}
\partial_t \psi + \tfrac{1}{2}\Delta \psi &= c\,\psi\,F(\bar\psi\psi),\\
-\partial_t \bar\psi + \tfrac{1}{2}\Delta \bar\psi &= c\,\bar\psi\,F(\bar\psi\psi).
\end{align}
Identifying \((\psi,\bar\psi)\) with \((\varphi,\bar\varphi)\), we recover exactly our system \eqref{PDE1}-\eqref{PDE2} written in forward-backward form (see Remark~\ref{forward-backward system}).
In the cubic case (Gross-Pitaevskii), the sign of $c$ distinguishes two qualitatively different regimes:

\begin{itemize}
\item {Defocusing / repulsive (\(c>0\)).}
The nonlinearity disperses mass, counteracting concentration.
Classically, this regime admits global-in-time solutions for broad classes of initial data (e.g., in \(H^1\)), scattering in low dimensions, and exhibits no finite-time blow-up \citep{cazenave2003semilinear}.
After Wick rotation (imaginary time), the system becomes a well-behaved gradient(-like) flow for the energy;
in our forward-backward formulation \eqref{PDE1}-\eqref{PDE2} this corresponds to a monotone mean-field interaction.

\item {Focusing / attractive (\(c<0\)).}
The nonlinearity promotes concentration and can balance dispersion to form \emph{bright solitons};
it may also produce finite-time blow-up (collapse) for sufficiently concentrated initial data (depending on the spatial dimension and mass) \citep{sulem1999nonlinear}.
Under Wick rotation, the flow may descend a \emph{nonconvex} energy landscape; in \eqref{PDE1}-\eqref{PDE2} this yields a mean-field term with negative monotonicity.
\end{itemize}

\section{Special cases of semilinear Schr\"odinger bridges}
\label{app:special-cases}

\subsection{Confining potentials}

When the interaction term reduces to a potential $F(t,x,m)\equiv V(t,x)$, 
the Schr\"odinger system \eqref{PDE1}-\eqref{PDE2} becomes \emph{linear}.
In this case, the optimal drift admits a Feynman-Kac representation:
\begin{equation}
\dd X_t
= \nabla \!\log \E\!\left[
\varphi\!\big(T,B_T^{t,X_t}\big)\,
\exp\!\Big(-c\!\int_t^T\! V(s,B_s^{t,X_t})\,\dd s\Big)
\right]\dd t
+ \dd W_t, \label{SDEV}
\end{equation}
where $(B_s^{t,x})_{s\ge t}$ is a Brownian motion starting from $x$ at time $t$.

\no A natural choice in generative modeling is a \emph{confining potential}
\[
V(x) = \Lambda\, \mathbf{1}_{x \notin \mathcal{D}},
\]
which restricts the diffusion to a prescribed open domain $\mathcal{D}$ as $\Lambda \to \infty$.
In this  limit, the optimal diffusion \eqref{SDEV} takes the form
\begin{align*}
\dd X_t
&= \nabla \log \E\!\left[\, \varphi(T,B^{t,X_t}_T)\, 
\mathbf{1}_{\{B^{t,X_t}_\cdot \in \mathcal{D}\}} \,\right]\dd t + \dd W_t \\[0.3em]
&= \nabla \log\!\left(
\int_{\mathcal{D}} \varphi(T,y)\, p^{\mathcal{D}}_{\,T-t}(X_t,y)\,\dd y
\right)\dd t
+ \dd W_t,
\end{align*}
where $p^{\mathcal{D}}_\tau(x,y)$ denotes the heat kernel on $\mathcal{D}$ with absorbing (Dirichlet) boundary conditions.

\subsubsection{Coxeter domains and Dyson-type diffusions}

When $\mathcal{D}$ is a \emph{Coxeter chamber} associated with a finite reflection group~\citep{grabiner1999brownian}, 
the heat kernel $p^{\mathcal{D}}_\tau(x,y)$ can be computed explicitly from the 
\emph{Coxeter reflection principle}, which generalizes the 
Karlin-McGregor determinant formula and expresses it as a signed sum of free Gaussian kernels
\[
p_\tau(x,y) = (2\pi\tau)^{-d/2}\,
\exp\!\Big(-\frac{|x-y|^2}{2\tau}\Big).
\]
Let $R\subset\R^d$ be a root system with positive roots $R_+$, and $W$ the corresponding finite Coxeter group\footnote{%
A \emph{Coxeter group} is a finite group $W$ generated by reflections $\omega_\alpha$ across hyperplanes orthogonal to a finite set of nonzero vectors $R\subset\mathbb{R}^d$, called the \emph{root system}.
The reflections satisfy the relations $(\omega_\alpha\omega_\beta)^{m_{\alpha\beta}}=1$ for integers $m_{\alpha\beta}\ge2$, which determine the group structure via a \emph{Coxeter matrix}.
The subset $R_+\subset R$ of \emph{positive roots} selects one of the two connected components of $\mathbb{R}^d\setminus\bigcup_{\alpha\in R} H_\alpha$,
where $H_\alpha=\{x:\langle\alpha,x\rangle=0\}$ are the reflection hyperplanes.
The corresponding \emph{Coxeter (or fundamental) chamber} is 
$\mathcal{D}=\{x\in\mathbb{R}^d:\langle\alpha,x\rangle>0,\ \forall\alpha\in R_+\}$.
When the root system is crystallographic, $W$ coincides with a \emph{Weyl group} associated with a simple Lie algebra, but more general (noncrystallographic) Coxeter types such as $H_3$, $H_4$, or $I_2(m)$ also fit this framework.
}
generated by reflections
$\omega_\alpha(x)=x-\tfrac{2\langle \alpha,x\rangle}{\langle \alpha,\alpha\rangle}\alpha$.
The corresponding Coxeter chamber is
\[
\mathcal{D}=\{x\in\R^d:\ \langle \alpha,x\rangle>0\ \text{for all } \alpha\in R_+\}.
\]

\noindent
For the classical Coxeter types $A$ and $C$, the corresponding chambers are
\begin{align*}
\mathcal{D}^A &:= \{x\in\R^d : x_1 < x_2 < \cdots < x_d\},\\[0.3em]
\mathcal{D}^C &:= \{x\in\R^d : 0 < x_1 < x_2 < \cdots < x_d\}.
\end{align*}
These domains impose symmetry-breaking constraints relevant to generative modeling.
For type $A$, the chamber $\mathcal{D}^A$ enforces a strict coordinate ordering, providing a ranked representation that removes permutation ambiguity in exchangeable systems (e.g., sorted eigenvalues or ordered latent features).
For type $C$, the positivity constraint ($x_i>0$) confines the process to the positive orthant, suitable for modeling nonnegative and ordered latent variables such as intensities or squared singular values.
Thus, Coxeter chambers like $\mathcal{D}^A$ and $\mathcal{D}^C$ serve as geometric priors for generative diffusions on manifolds with permutation or reflection symmetries.

\noindent
From the Coxeter reflection principle,\footnote{%
\textbf{Proof sketch.}
Let $W$ be a finite Coxeter group acting orthogonally on $\mathbb{R}^d$, generated by reflections $\{\omega_\alpha\}$ across the hyperplanes $H_\alpha=\{x:\langle\alpha,x\rangle=0\}$, and let $\mathcal D$ be a fundamental chamber. We verify $p^{\mathcal D}_\tau$ is the fundamental solution of the   heat kernel on $\mathcal D$ with  Dirichlet boundary conditions:

\emph{(i) Heat equation.}
Each term $p_\tau(x,wy)$ satisfies $(\partial_\tau-\tfrac12\Delta_x)p_\tau=0$ in $x$, hence so does their finite signed sum by linearity.

\emph{(ii) Boundary condition (reflection principle).}
Every reflection $\omega_\alpha\in W$ is an orthogonal involution with $\det(\omega_\alpha)=-1$, so $|\omega_\alpha x-wy|=|x-\omega_\alpha wy|$, i.e.\ $p_\tau(\omega_\alpha x,wy)=p_\tau\big(x,\omega_\alpha wy\big)$, for all $x$.
Since $w\mapsto\omega_\alpha w$ is a bijection of $W$ and $\det(\omega_\alpha w)=-\det(w)$,
\[
p^{\mathcal D}_\tau(\omega_\alpha x,y)
=\sum_{w\in W}\det(w)\,p_\tau(x,\omega_\alpha wy)
=\sum_{w'\in W}\det(\omega_\alpha w')\,p_\tau(x,w'y)
=-\,p^{\mathcal D}_\tau(x,y).
\]
Thus $p^{\mathcal D}_\tau(\cdot,y)$ is \emph{odd} under each reflection. For $x\in H_\alpha$, $\omega_\alpha x=x$ gives $p^{\mathcal D}_\tau(x,y)=0$; as $\partial\mathcal D\subset\bigcup_{\alpha\in R_+}H_\alpha$, $p^{\mathcal D}_\tau$ vanishes on $\partial\mathcal D$.

\emph{(iii) Initial condition.}
As $\tau\downarrow0$, $p_\tau(\cdot,wy)\to\delta_{wy}$. For $y\in\mathcal D$ and $w\neq e$, the point $wy$ lies in the chamber $w\mathcal D$, which is disjoint from $\overline{\mathcal D}$ ($W$ acts simply transitively on the chambers). Hence only $w=e$ contributes on $\mathcal D$, and $p^{\mathcal D}_\tau(\cdot,y)\to\delta_y$ there.
}
we obtain the general Coxeter–reflection formula
\[
p_\tau^{\mathcal{D}}(x,y)=\sum_{w\in W}\det(w)\,p_\tau\!\big(x, w y\big),
\qquad x,y\in\mathcal{D},
\]
where $\det(w)\in\{\pm1\}$ is the sign of the orthogonal transformation $w$.
For $\mu_0=\delta_{X_0}$, the potentials of the (linear, $F\equiv0$ inside $\mathcal D$) bridge are $\varphi(T,y)\propto\mu_T(y)/p^{\mathcal D}_T(X_0,y)$ and $\varphi(t,x)=\int_{\mathcal D}p^{\mathcal D}_{T-t}(x,y)\,\varphi(T,y)\,\dd y$, so that
\[
\dd X_t
=  \nabla \log \int_{\mathcal{D}} \mu_T(dy)\,
\frac{p^{\mathcal{D}}_{\,T-t}(X_t,y)}{p^{\mathcal{D}}_{\,T}(X_0,y)} \,\dd t
+ \dd W_t.
\]

\begin{remark}[Dyson Brownian motion (Type A), $\lambda=\infty$]
For type $A$ and the hard constraint $\lambda=\infty$, let $\mu_0=\mu_T=\delta_{X_0}$ with $X_0\in\mathcal D^A$. Then $\varphi(t,x)\propto p^{\mathcal D^A}_{T-t}(x,X_0)=\det\big[p^{(1)}_{T-t}(x_k,X_0^l)\big]_{k,l=1}^d$, where $p^{(1)}$ is the one-dimensional Gaussian kernel (the Coxeter formula for $W=S_d$ is the Karlin-McGregor determinant). The optimal process supported on $\mathcal{D}^A$ is the non-colliding Brownian bridge
\[
\dd X_t^i
= \bigg(
\frac{X_0^i - X_t^i}{T-t}
+ \partial_{x_i}\log\det\Big[e^{x_k(X_0^l-X_0^k)/(T-t)}\Big]_{k,l=1}^d\Big|_{x=X_t}
\bigg)\dd t + \dd W_t^i,\qquad i=1,\dots,d.
\]
For $d=2$ the interaction term equals $-\frac{X_0^2-X_0^1}{T-t}\big(e^{(X_t^2-X_t^1)(X_0^2-X_0^1)/(T-t)}-1\big)^{-1}$. By the Cauchy-Binet formula, $\det[e^{x_ky_l/\tau}]_{k,l}=\tau^{-d(d-1)/2}\Delta(x)\Delta(y)\big(\prod_{j=1}^{d-1}j!\big)^{-1}\big(1+O(\tau^{-1})\big)$ as $\tau\to\infty$, with $\Delta(x):=\prod_{k<l}(x_l-x_k)$; the determinant in the drift equals $e^{-\sum_kx_ky_k/\tau}\det[e^{x_ky_l/\tau}]_{k,l}$ with $y=X_0$, $\tau=T-t$, so its logarithmic derivative has the same limit. Hence, in the long-time limit $T\to\infty$ ($t$ fixed), the bridge term vanishes, the interaction term tends to $\partial_{x_i}\log\Delta(X_t)$, and we recover the celebrated Dyson Brownian motion
\[
\dd X_t^i = \sum_{j\neq i}\frac{1}{X_t^i - X_t^j}\,\dd t + \dd W_t^i,\qquad i=1,\dots,d,
\]
which describes the eigenvalue dynamics of Hermitian Brownian ensembles (GUE).
Thus, the linear Schr\"odinger bridge framework naturally connects to classical random-matrix diffusions.
\end{remark}

\begin{remark}[Type C, $\lambda=\infty$]
Analogously, for type $C$ with $\mu_0=\mu_T=\delta_{X_0}$, $X_0\in\mathcal D^C$, the group $W$ is the group of signed permutations, $p^{\mathcal D^C}_\tau(x,y)=\det\big[p^{(1)}_\tau(x_k,y_l)-p^{(1)}_\tau(x_k,-y_l)\big]_{k,l=1}^d$, and the optimal diffusion confined to $\mathcal{D}^C$ reads
\[
\dd X_t^i
= \partial_{x_i}\log\det\Big[p^{(1)}_{T-t}(x_k,X_0^l)-p^{(1)}_{T-t}(x_k,-X_0^l)\Big]_{k,l=1}^d\Big|_{x=X_t}\,\dd t + \dd W_t^i,\qquad i=1,\dots,d.
\]
Since $p^{(1)}_\tau(x,y)-p^{(1)}_\tau(x,-y)=2(2\pi\tau)^{-1/2}e^{-(x^2+y^2)/2\tau}\sinh(xy/\tau)$, the same expansion (with odd powers $1,3,\dots,2d-1$) shows that, as $T\to\infty$, the drift tends to $\partial_{x_i}\log h(X_t)$ with $h(x):=\prod_{k}x_k\prod_{k<l}(x_l^2-x_k^2)$, a harmonic function vanishing on $\partial\mathcal D^C$:
\[
\dd X_t^i
= \left(
\frac{1}{X_t^i}
+ \sum_{j\neq i}\frac{2X_t^i}{(X_t^i)^2-(X_t^j)^2}
\right)\dd t + \dd W_t^i,\qquad i=1,\dots,d.
\]
This is the type-$C$ analogue of Dyson Brownian motion \citep{grabiner1999brownian}: three-dimensional Bessel processes (drift $1/x$) conditioned never to collide.
\end{remark}

\subsection{Score-based generative modeling as an exact linear-SB solution ($F\equiv 0$, noised-data prior)}
\label{app:score-based}

Take $F\equiv 0$ (classical Schr\"odinger bridge), set $\mu_T=p_{\mathrm{data}}$, and choose the reference marginal as a noisy version of $\mu_T$: $\mu_0 := \mu_T * \mathcal N(0, T\,I_d)$. The $(\varphi,\bar\varphi)$ system of Theorem~\ref{theorem:optimal_solution} collapses to two linear (backward) heat equations:
\[
\partial_t\varphi+\tfrac12\Delta\varphi=0,
\qquad
\partial_t\bar\varphi+\tfrac12\Delta\bar\varphi=0.
\]
We claim the exact solution is $\bar\varphi\equiv 1$ and
\[
\varphi(t,x)\;=\;\big(p_{\mathrm{data}}*\mathcal N(0,(T-t)I_d)\big)(x)\;=:\;p_{T-t}(x).
\]
Indeed: (i) both solve the heat equations trivially; (ii) the coupling $\varphi(0,x)\bar\varphi(T,x)=\mu_0(x)$ holds since $\varphi(0,\cdot)=p_T=\mu_T*\mathcal N(0,T\,I_d)=\mu_0$; (iii) $m(T,x)=\bar\varphi(0,x)\varphi(T,x)=p_{\mathrm{data}}(x)$, so the hard terminal constraint $m(T,\cdot)=\mu_T$ holds. Consequently the optimal drift of Theorem~\ref{theorem:optimal_solution} reduces to
\[
\nabla\log\varphi(t,X_t)\;=\;\nabla\log p_{T-t}(X_t),
\]
the \emph{score function} of the Gaussian-smoothed data at noise level $T-t$. Simulating $\dd X_t=\nabla\log\varphi(t,X_t)\,\dd t+\dd W_t$, $X_0\sim\mu_0$, coincides with the reverse-time score-based sampler of \citet{song2021scorebased,ho2020denoising}: DDPM / score-based diffusion is the classical-SB solution with noised-data prior and $\bar\varphi\equiv 1$.

\subsection{Analytical semilinear bridge between Gaussians for $F(m)=\kappa \log m$ and $G(x,m(T,\cdot))=\lambda \log \frac{m(T,x)}{\mu_T(x)}$}

\begin{proposition}[Optimal diffusion between Gaussians]
 \label{propPHL}
Let the initial and terminal laws be Gaussian,
\[
\mu_0 = \mathcal N(m_{0},\Sigma_{0}),
\qquad
\mu_T = \mathcal N(m_{T},\Sigma_{T}),
\]
with $\Sigma_0,\Sigma_T \in \mathbb S_{++}^d$ and $m_0, m_T\in \R^d$. We take
$F(m)=\kappa \log m$ and $G(x,m(T,\cdot))=\lambda \log \frac{m(T,x)}{\mu_T(x)}$ with $\kappa,\lambda>0$, and assume that the forward-backward system of ODEs
\begin{align*}
\dot A_t &= c\,A_t^2 + \kappa\,\hat{\Sigma}_t^{-1}, \quad A_T=\lambda( \Sigma_T^{-1}-\hat{\Sigma}_T^{-1}) \\
\dot b_t &= c\,A_t b_t - \kappa\,\hat{\Sigma}_t^{-1}\hat{\mu}_t, \quad b_T=\lambda( \hat{\Sigma}_T^{-1} \hat{\mu}_T-\Sigma_T^{-1} m_T ) \\
\dot{\hat{\mu}}_t &= -\,c\,(A_t \hat{\mu}_t + b_t),\quad \hat{\mu}_{0}=m_0 \\
\dot{\hat{\Sigma}}_t &= -\,c\,(A_t \hat{\Sigma}_t + \hat{\Sigma}_t A_t) + I_d,
\quad \hat{\Sigma}_{0}=\Sigma_0
\end{align*}
has a solution on $[0,T]$ with $A_t\in\mathbb S^d$ symmetric, $b_t,\hat\mu_t\in\R^d$ and $\hat\Sigma_t\in\mathbb S^d_{++}$. Then the \emph{optimal diffusion process} is
\[
\dd X_t \;=\; -\,c\,(A_t X_t + b_t)\,\dd t \;+\; \dd W_t,\qquad X_0\sim\mu_0,
\]
with $\mathrm{Law}(X_t)=\mathcal N(\hat\mu_t,\hat\Sigma_t)$, and it is the unique optimal control.
\end{proposition}

\noindent
The proof is provided below in Appendix~\ref{demo-gaussian}. It proceeds by inserting a Gaussian ansatz for \(m(t,\cdot)\) and a quadratic ansatz for \(v(t,\cdot)\) into the forward Fokker-Planck and backward HJB equations, matching the coefficients in \(x\), and concluding with the verification argument of Appendix~\ref{demo-main-theorem}.

\medskip
\noindent
\textbf{Remark.} When \(\kappa=0\) and \(\lambda\to\infty\), this system reduces to the \emph{classical Schr\"odinger bridge} between Gaussian marginals, yielding the well-known result that the bridge is a linear Gaussian diffusion with quadratic potentials governed by coupled Riccati-type ODEs \citep{bunne2023schrodinger}.

\subsection{Lifted (semilinear) Schr\"odinger bridge and  Gyöngy's Markov projection}

Gyöngy's Markov projection theorem \citep{gyongy1986mimicking} has played a central role in 
mathematical finance. It provides a way to construct a Markov diffusion in $\mathbb R^d$ 
with prescribed marginal distributions at each time. This idea underlies the calibration of \emph{local volatility} and 
\emph{local stochastic volatility} models, as pioneered by Dupire's formula 
\citep{dupire1994pricing} and systematically developed in the monograph \citep{guyon2011nonlinear}. 

\no In this subsection, we apply Gyöngy's projection to the setting of Schr\"odinger bridges. 
The key observation is that one may lift the marginals $\mu_0,\mu_T$ to an augmented space 
$\mathbb R^{d+1}$ (eventually higher), solve the Schr\"odinger bridge there, and then project back onto 
$\mathbb R^d$. Gyöngy's mimicking theorem, in the form of \citet{brunick2013mimicking}, then provides a solution of a Markovian SDE in
$\mathbb R^d$ with the correct marginals $\mu_0,\mu_T$. 
Although this projected diffusion has the right one-time marginals, it is in general 
\emph{not} itself a Schr\"odinger bridge. It therefore defines a new class of generative 
models, obtained by combining lifted SB with Markov projection.

\begin{center}
\begin{tikzcd}[row sep=large, column sep=large]
\tilde{\mu}_0 \ar[r, "\text{SB in }\mathbb{R}^{d+1}"] \ar[d, two heads, "projection"']
& \tilde{\mu}_T \ar[d, two heads, "projection"] \\
\mu_0 \ar[r, "\text{Markov proj.}"']
& \mu_T
\end{tikzcd}
\end{center}

\begin{proposition}[Lifted Schr\"odinger bridge and Gyöngy projection]
\label{propGyongy}
Let $\mu_0,\mu_T \in \mathcal P(\mathbb R^d)$ be given marginals, and let
$\tilde\mu_0,\tilde\mu_T \in \mathcal P(\mathbb R^{d+1})$ be their lifted versions
via conditional distributions $\nu_0,\nu_T$:
\[
\tilde\mu_0(dx,dy) = \mu_0(dx)\,\nu_0(dy|x),
\qquad
\tilde\mu_T(dx,dy) = \mu_T(dx)\,\nu_T(dy|x).
\]
Consider the Schr\"odinger bridge in $\mathbb R^{d+1}$ between
$\tilde\mu_0$ and $\tilde\mu_T$ under the Wiener reference $\PP^0$, with dynamics
\[
d\tilde X_t = \nabla_{(x,y)} \log \varphi(t,\tilde X_t)\,dt + d\tilde W_t,
\qquad \tilde X_t = (X_t,Y_t),
\]
where $(\varphi,\bar\varphi)$ are the Schr\"odinger potentials in $\mathbb R^{d+1}$, and assume that the bridge has finite energy, $\E\int_0^T|\nabla_{(x,y)}\log\varphi(t,\tilde X_t)|^2\,\dd t<\infty$. Consider
\begin{equation}\label{eq:gyongy-sde}
d X_t =
b(t,X_t)\,\dd t + \dd W_t ,\quad X_0\sim\mu_0,
\qquad
b(t,x):=\frac{\displaystyle \int \nabla_x \varphi(t,x,y)\, \bar \varphi(T-t,x,y)\,\dd y}
     {\displaystyle \int \varphi(t,x,y)\, \bar \varphi(T-t,x,y)\,\dd y},
\end{equation}
with $b(t,x):=0$ where the denominator vanishes. Then \eqref{eq:gyongy-sde} has a weak solution $X$ such that $\mathrm{Law}(X_t)$ is the $x$-marginal of $\mathrm{Law}(\tilde X_t)$ for every $t\in[0,T]$; in particular $\mathrm{Law}(X_0)= \mu_0$ and $\mathrm{Law}(X_T) = \mu_T$. If moreover $\nabla_x\log\varphi$ is bounded, every solution of \eqref{eq:gyongy-sde} has these marginals.
\end{proposition}
\no The proof is deferred to Appendix~\ref{demo-gyongy}.

\section{Proof of Proposition \ref{propPHL}}
\label{demo-gaussian}

\begin{proof}
We look for a solution of \eqref{eq:HJB}-\eqref{FP} of the form
\[
m(t,\cdot)=\mathcal N\!\big(\hat{\mu}_t,\hat{\Sigma}_t\big),
\qquad
v(t,x)=\tfrac12 x^\top A_t x + b_t^\top x + a_t,\qquad A_t\in\mathbb S^d,
\]
where $(A_t,b_t,\hat\mu_t,\hat\Sigma_t)$ solves the system of the proposition and $a_t$ is determined below.

\emph{Fokker-Planck equation.} The feedback drift $-c\nabla_x v(t,x)=-c(A_tx+b_t)$ is affine. Hence the solution of $\dd X_t=-c(A_tX_t+b_t)\,\dd t+\dd W_t$, $X_0\sim\mathcal N(m_0,\Sigma_0)$, is Gaussian, and Itô's formula gives for its mean and covariance $\dot{\hat\mu}_t=-c(A_t\hat\mu_t+b_t)$ and $\dot{\hat\Sigma}_t=-c(A_t\hat\Sigma_t+\hat\Sigma_tA_t)+I_d$, with $\hat\mu_0=m_0$, $\hat\Sigma_0=\Sigma_0$. Therefore $m(t,\cdot)=\mathcal N(\hat\mu_t,\hat\Sigma_t)$ is the law of $X_t$ and solves \eqref{FP}.

\emph{HJB equation.} We have $\kappa\log m(t,x)=-\frac\kappa2(x-\hat\mu_t)^\top\hat\Sigma_t^{-1}(x-\hat\mu_t)+\gamma_t$, with $\gamma_t:=-\frac\kappa2\log\det(2\pi\hat\Sigma_t)$, and $|\nabla_x v|^2=x^\top A_t^2x+2b_t^\top A_tx+|b_t|^2$ since $A_t$ is symmetric. The left-hand side of \eqref{eq:HJB} is therefore
\[
\tfrac12x^\top\big(\dot A_t-cA_t^2-\kappa\hat\Sigma_t^{-1}\big)x
+\big(\dot b_t-cA_tb_t+\kappa\hat\Sigma_t^{-1}\hat\mu_t\big)^{\!\top}x
+\Big(\dot a_t+\tfrac12\operatorname{tr}A_t-\tfrac c2|b_t|^2-\tfrac\kappa2\hat\mu_t^\top\hat\Sigma_t^{-1}\hat\mu_t+\gamma_t\Big).
\]
The first two terms vanish by the equations for $A_t$ and $b_t$, and the third one vanishes when $a_t$ solves the corresponding linear ODE; $a_t$ does not enter the drift.

\emph{Terminal condition.} $G(x,m(T,\cdot))=\lambda\log\frac{m(T,x)}{\mu_T(x)}=\frac\lambda2x^\top\big(\Sigma_T^{-1}-\hat\Sigma_T^{-1}\big)x+\lambda\big(\hat\Sigma_T^{-1}\hat\mu_T-\Sigma_T^{-1}m_T\big)^{\!\top}x+\mathrm{const}$, so that $v(T,\cdot)=G(\cdot,m(T,\cdot))$ is equivalent to the stated values of $A_T$ and $b_T$, together with $a_T=\mathrm{const}$.

\emph{Optimality.} The functionals $f(m)=\kappa\int(m\log m-m)$ and $g(m)=\lambda\int\big(m\log\frac{m}{\mu_T}-m\big)$ have derivatives $F=\kappa\log m$ and $G=\lambda\log\frac{m}{\mu_T}$, which satisfy (A2), strictly, for $\kappa,\lambda>0$. The pair $(v,m)$ satisfies (A4): $\nabla_x v$ is affine with bounded coefficients on $[0,T]$, and $v$, $F(t,\cdot,m(t,\cdot))$ and $G(\cdot,m(T,\cdot))$ are quadratic in $x$. Steps 2 and 3 of Appendix~\ref{demo-main-theorem} then show that the affine feedback $-c(A_tx+b_t)$ is the unique optimal control.
\end{proof}

\section{Proof of Proposition \ref{propGyongy}}
\label{demo-gyongy}
\begin{proof}
Write $\tilde X_t=(X^{(x)}_t,Y_t)\in\R^d\times\R$ for the lifted bridge. By Theorem~\ref{theorem:optimal_solution} applied in $\R^{d+1}$ with $f\equiv0$, $\tilde X_t$ has density $\tilde m(t,x,y)=\varphi(t,x,y)\,\bar\varphi(T-t,x,y)$. Since $\varphi\,\nabla_x\log\varphi=\nabla_x\varphi$, the drift $b$ of \eqref{eq:gyongy-sde} is the conditional expectation
\[
b(t,x)=\frac{\int\nabla_x\log\varphi(t,x,y)\,\tilde m(t,x,y)\,\dd y}{\int\tilde m(t,x,y)\,\dd y}
=\E\big[\nabla_x\log\varphi(t,\tilde X_t)\,\big|\,X^{(x)}_t=x\big].
\]
The $x$-component is the Itô process $\dd X^{(x)}_t=\delta_t\,\dd t+\dd W_t$, with $\delta_t:=\nabla_x\log\varphi(t,\tilde X_t)$ and identity diffusion matrix, and $\E\int_0^T|\delta_t|\,\dd t\le\big(T\,\E\int_0^T|\delta_t|^2\,\dd t\big)^{1/2}<\infty$ by the finite-energy assumption. Corollary~3.7 of \citet{brunick2013mimicking}, which extends \citet{gyongy1986mimicking} to random initial conditions and removes its boundedness and nondegeneracy assumptions, provides a weak solution $\hat X$ of $\dd\hat X_t=\hat b(t,\hat X_t)\,\dd t+\dd\hat W_t$, $\hat X_0\sim\mu_0$, with $\mathrm{Law}(\hat X_t)=\mathrm{Law}(X^{(x)}_t)$ for all $t$, where $\hat b(t,X^{(x)}_t)$ is a version of $\E[\delta_t\,|\,X^{(x)}_t]$ for almost every $t$. By the display above, $b(t,\cdot)=\hat b(t,\cdot)$ $\mathrm{Law}(X^{(x)}_t)$-a.e.\ for almost every $t$; since $\hat X_s$ has the law of $X^{(x)}_s$, Fubini's theorem gives $\E\int_0^T|b-\hat b|(s,\hat X_s)\,\dd s=0$, so $\hat X$ also solves \eqref{eq:gyongy-sde}. If $\nabla_x\log\varphi$ is bounded, so is $b$, and \eqref{eq:gyongy-sde} is unique in law by Girsanov's theorem, so every solution has these marginals. Consistently, integrating in $y$ the Fokker-Planck equation of $\tilde X$, whose $\partial_y$-divergence term integrates to zero, shows that $\rho_t:=\int\tilde m(t,\cdot,y)\,\dd y$ solves $\partial_t\rho=-\nabla_x\cdot(b\,\rho)+\frac12\Delta_x\rho$, the Fokker-Planck equation of \eqref{eq:gyongy-sde}. Finally, $\rho_0$ and $\rho_T$ are the $x$-marginals of $\tilde\mu_0$ and $\tilde\mu_T$, that is $\mu_0$ and $\mu_T$.
\end{proof}

\section{Numerical sanity check: branching with time-reflected coupling}
\label{app:branching-sanity-toy}

\noindent
Marked branching for systems of coupled semilinear equations is established \citep{HenryLabordere,henrylabordere2019branching,HenryLabordereTouzi2021}; the bridge system \eqref{PDE1}-\eqref{PDE2} is special in that the nonlinearity pairs $\varphi(t,\cdot)$ with $\bar\varphi(T-t,\cdot)$ (see the main text diagram).
The scalar model below isolates $u(t,x)\,u(T-t,x)$ as an analytic benchmark for that $T{-}t$ product.

\subsection{Sanity check}
\label{sec:toy-branching-u}

We illustrate the branching diffusion method on a simple one-dimensional
semilinear PDE for a single potential $u$:
\begin{equation}
\label{eq:toy-pde-system}
\begin{cases}
\partial_t u(t,x) + \frac12\,\partial_{xx} u(t,x)
+ u(t,x)\,u(T-t,x) \;=\; 0,\qquad (t,x)\in\R\times\R,\\[0.5em]
u(T,x) \;=\; 3\,\operatorname{sech}^2(x)-1-\tanh\Bigl(\dfrac{T}{2}\Bigr).
\end{cases}
\end{equation}

\noindent
The explicit solution is given by
\[
u(t,x) \;=\; 3\,\operatorname{sech}^2(x)-1-\tanh\Bigl(t-\frac{T}{2}\Bigr),
\]
so that $u(0,\cdot)$ is known in closed form. This example is used only as a
numerical sanity check of the branching implementation. In this toy case, we work with a single family
of particles (say ``black'' particles, representing $u$). We rewrite the
nonlinearity $u(t,x)u(T-t,x)$ in \eqref{eq:toy-pde-system} using the
\textit{canonical} form introduced in \eqref{canonical_form}:
\[
\partial_t u(t,x) + \frac12\,\partial_{xx} u(t,x)
+1 \times \left( \tfrac{1}{2}\times 2 u(t,x)\,u(T-t,x) + \tfrac{1}{2}\times 2u(t,x) - u(t,x) \right) \;=\; 0,
\]
i.e.\ branching rate $1$, probabilities $p_0=p_1=\tfrac12$ and weights $a_k/p_k=2$: at a branching time $s$, the particle either restarts alone at $(s,X_s)$, or is replaced by two particles started at $(s,X_s)$ and $(T-s,X_s)$.

\noindent Figure \ref{fig:branching-u0-validation} compares the analytical solution $u(0,x)$ as a function of $x$ with the Monte Carlo branching algorithm for $T=1$.
The branching estimator closely matches the analytical curve, confirming the correctness of the branching scheme.
To assess the impact of the nonlinear term, we also include the solution of the heat kernel with terminal condition $u(T,x)$.

\begin{figure}[H]
\centering
\includegraphics[width=0.7\textwidth]{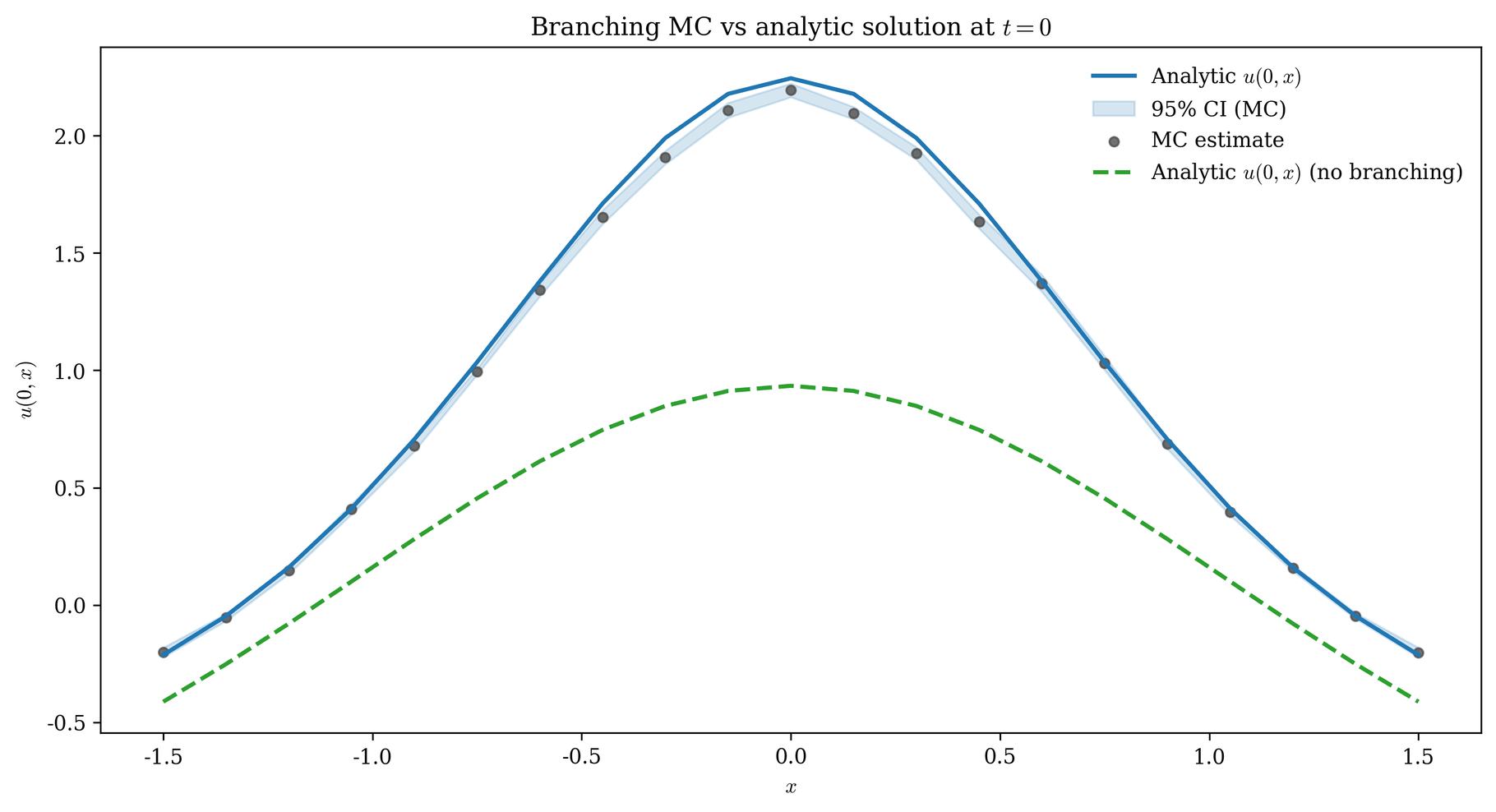}
\caption{
Comparison of the analytic solution $u(0,x)$ with the Monte Carlo estimate
$\hat u_M(0,x)$, in the branching case and the heat kernel,
highlighting the impact of the nonlinearity.}
\label{fig:branching-u0-validation}
\end{figure}

\section{Proof of Proposition \ref{prop:phi-branch-lamc}}
\label{demo-L1-L2-integrability}

\paragraph{Notation and mild solutions.}
We write $\varphi_T:=\varphi(T,\cdot)$, $\bar\varphi_T:=\bar\varphi(T,\cdot)$, $A:=\|\varphi_T\|_\infty$, $B:=\|\bar\varphi_T\|_\infty$, $b_k:=\|a_k\|_{L^\infty([0,T]\times\R^d)}$ and $\bar K:=\sum_k k\,p_k$ (the mean mark). Only $p_k>0$ for $a_k\not\equiv0$ is needed, with the convention $0/0=0$ in the ratios $a_k/p_k$ and $b_k^2/p_k$. Let $P_s h(x):=\E[h(x+W_s)]$ be the heat semigroup. In the canonical form \eqref{canonical_form}, \eqref{PDE1} reads $\partial_t\varphi+\frac12\Delta\varphi-c\varphi=-c\sum_k a_k(t,\cdot)\,\varphi(t,\cdot)^{k+1}\bar\varphi(T-t,\cdot)^{k}$, and symmetrically for \eqref{PDE2}. A \emph{mild solution} is a pair $(\varphi,\bar\varphi)$ of measurable functions on $[0,T]\times\R^d$ such that, for all $(t,x)$, with absolutely convergent integrals,
\begin{equation}\label{eq:mild}
\begin{aligned}
\varphi(t,x)&=e^{-c(T-t)}P_{T-t}\varphi_T(x)\\
&\quad+\int_t^T c\,e^{-c(s-t)}\,P_{s-t}\Big[\textstyle\sum_{k}a_k(s,\cdot)\,\varphi(s,\cdot)^{k+1}\,\bar\varphi(T-s,\cdot)^{k}\Big](x)\,\dd s,\\
\bar\varphi(t,x)&=e^{-c(T-t)}P_{T-t}\bar\varphi_T(x)\\
&\quad+\int_t^T c\,e^{-c(s-t)}\,P_{s-t}\Big[\textstyle\sum_{k}a_k(T-s,\cdot)\,\varphi(T-s,\cdot)^{k}\,\bar\varphi(s,\cdot)^{k+1}\Big](x)\,\dd s.
\end{aligned}
\end{equation}
These are the Duhamel formulas of \eqref{PDE1}-\eqref{PDE2} for the operator $\partial_t+\frac12\Delta-c$.

\paragraph{The tree.}
Particles are labelled by finite words (Ulam-Harris labelling), and each label carries an independent Brownian motion, an independent $\mathrm{Exp}(c)$ lifetime $\tau$ and an independent mark $K$ with law $(p_k)$. A particle born at time $b$ (on its own clock) at the point $y$ moves as $y+W_{\cdot-b}$. If $b+\tau\ge T$, it is a \emph{leaf} and contributes $\varphi_T(y+W_{T-b})$ if it is white, $\bar\varphi_T(y+W_{T-b})$ if it is black. Otherwise, with $s:=b+\tau$ and $z:=y+W_\tau$, it contributes the factor $a_K(s^\circ,z)/p_K$ of \eqref{eq:weight-update}, and it is replaced by $K+1$ particles of its own type born at $(s,z)$ and $K$ particles of the other type born at $(T-s,z)$. The root has generation $0$, the children of a particle of generation $j$ have generation $j+1$, and $D\in\mathbb N\cup\{\infty\}$ denotes the largest generation. On $\{D<\infty\}$, $\Phi$ is the product of all factors and of all leaf contributions; $\Phi^\circ_{t,x}$ (resp.\ $\Phi^\bullet_{t,x}$) is $\Phi$ for a white (resp.\ black) root born at $(t,x)$. We use the \emph{branching property}: conditionally on the lifetime, the mark and the path of the root, the subtrees rooted at its children are independent, and each is distributed as a tree rooted at the type and birth point of that child.

Because of the time reflection, a black child born at $T-s$ may live longer than its white parent, so the finiteness of the tree is not automatic: it is the object of Lemma~\ref{lem:tree-size}.

\begin{lemma}[Size of the tree]\label{lem:tree-size}
Let $n(t)\in[1,\infty]$ be the expected number of particles of a tree whose root is born at time $t$; it depends neither on the type nor on the position of the root. If $cT\bar K<1$, then
\[
n(t)\;\le\;1+\frac{c\,(T-t)\,(1+2\bar K)}{1-cT\bar K}\;<\;\infty
\]
and, in particular, $D<\infty$ almost surely. If $cT\bar K\ge1$, $n$ is unbounded on $[0,T]$.
\end{lemma}

\begin{proof}
Let $n_j(t)$ be the expected number of particles of generation at most $j$. The root branches at $s\in[t,T)$ with density $c\,e^{-c(s-t)}$, and then has on average $1+\bar K$ children born at $s$ and $\bar K$ children born at $T-s$. By the branching property, $n_0=1$ and
\[
n_{j+1}(t)=1+\int_t^T c\,e^{-c(s-t)}\big[(1+\bar K)\,n_j(s)+\bar K\,n_j(T-s)\big]\,\dd s .
\]
A direct computation shows that $N(t):=1+c(T-t)(1+2\bar K)/(1-cT\bar K)$ satisfies the same equation with $n_j$ and $n_{j+1}$ both replaced by $N$. As $N\ge1=n_0$ and the right-hand side is nondecreasing in $n_j$, induction gives $n_j\le N$ for all $j$, and $n=\lim_j n_j\le N$ by monotone convergence.

Conversely, if $n$ is bounded, it solves the equation above (monotone convergence), hence is continuously differentiable, with $n'(t)=c\,(n(t)-1)-c\big[(1+\bar K)n(t)+\bar K\,n(T-t)\big]=-c\big(1+\bar K S(t)\big)$, where $S(t):=n(t)+n(T-t)$. Then $S'(t)=n'(t)-n'(T-t)=0$ because $S(T-t)=S(t)$, so $S$ is a constant, $n(t)=1+c(1+\bar K S)(T-t)$ since $n(T)=1$, and evaluating $S=n(t)+n(T-t)$ gives $S\,(1-cT\bar K)=2+cT>0$. This is impossible when $cT\bar K\ge1$.
\end{proof}

\begin{lemma}[Time-reflected majorant system]\label{lem:majorant}
Let $\alpha_0,\beta_0>0$, $q_k\ge0$ for $k=0,\dots,K$, and $Q(r):=\sum_k q_k r^k$. Consider, for bounded nonnegative functions $U,\bar U$ on $[0,T]$, the system
\begin{equation}\label{eq:majorant-system}
\begin{aligned}
U(t)&=\alpha_0e^{-c(T-t)}+\int_t^T c\,e^{-c(s-t)}\textstyle\sum_k q_k\,U(s)^{k+1}\,\bar U(T-s)^{k}\,\dd s,\\
\bar U(t)&=\beta_0e^{-c(T-t)}+\int_t^T c\,e^{-c(s-t)}\textstyle\sum_k q_k\,U(T-s)^{k}\,\bar U(s)^{k+1}\,\dd s .
\end{aligned}
\end{equation}
\begin{itemize}
\item[(a)] $(U,\bar U)$ solves \eqref{eq:majorant-system} if and only if $U(t)=\alpha_0e^{c(T-t)(Q(r)-1)}$ and $\bar U(t)=\beta_0e^{c(T-t)(Q(r)-1)}$ for some $r>0$ such that
\begin{equation}\label{eq:fixed-point}
r=\Psi(r):=\alpha_0\beta_0\exp\big(cT\,(Q(r)-1)\big);
\end{equation}
then $U(t)\,\bar U(T-t)=r$ for all $t\in[0,T]$.
\item[(b)] \eqref{eq:fixed-point} has a root if and only if $\Psi(\bar r)\le\bar r$ for some $\bar r>0$. The smallest root $r_*$ is then the limit of the nondecreasing sequence $r_0=0$, $r_{j+1}=\Psi(r_j)$.
\item[(c)] Let $M$ be the functional of the tree obtained by replacing every branching factor by $q_K/p_K$ and every white (resp.\ black) leaf contribution by $\alpha_0$ (resp.\ $\beta_0$), and let $u(t):=\E[M\,\mathbf 1_{\{D<\infty\}}]$ and $\bar u(t)$ be its expectations for a white and a black root born at $t$. If \eqref{eq:fixed-point} has a root, $(u,\bar u)$ is the solution of (a) with $r=r_*$. Otherwise, $u$ or $\bar u$ is unbounded.
\end{itemize}
\end{lemma}

\begin{proof}
(a) Let $(U,\bar U)$ be a bounded nonnegative solution. The right-hand sides are continuous in $t$, hence so are $U,\bar U$, and then they are $C^1$. Since $\sum_kq_kU^{k+1}\bar U_{T-t}^{k}=U\,Q(U\bar U_{T-t})$, differentiation gives
\[
U'(t)=c\,U(t)\big[1-Q\big(U(t)\bar U(T-t)\big)\big],\qquad
\bar U'(t)=c\,\bar U(t)\big[1-Q\big(U(T-t)\bar U(t)\big)\big].
\]
For $R(t):=U(t)\bar U(T-t)$, this yields $R'(t)=U'(t)\bar U(T-t)-U(t)\bar U'(T-t)=c\,R\,[1-Q(R)]-c\,R\,[1-Q(R)]=0$. Hence $R\equiv r$, $U'=c(1-Q(r))\,U$ with $U(T)=\alpha_0$, so $U(t)=\alpha_0e^{c(T-t)(Q(r)-1)}$; similarly for $\bar U$; and $r=U(0)\bar U(T)=\alpha_0\beta_0e^{cT(Q(r)-1)}$. Conversely, if $r>0$ solves \eqref{eq:fixed-point}, the functions above satisfy $U(t)\bar U(T-t)=\alpha_0\beta_0e^{cT(Q(r)-1)}=r$, hence the linear equation $U'-cU=-c\sum_kq_kU^{k+1}\bar U_{T-t}^k$ with $U(T)=\alpha_0$, whose variation-of-constants formula is the first line of \eqref{eq:majorant-system}; the second line is obtained in the same way.

(b) $\Psi$ is continuous and nondecreasing, with $\Psi(0)>0$. If $\Psi(\bar r)\le\bar r$, induction gives $r_j\le r_{j+1}\le\bar r$, so $r_j$ converges to a root $r_*\le\bar r$; any root $r$ satisfies $r_j\le r$ for all $j$, hence $r_*\le r$. The converse is obvious.

(c) Let $u_j(t):=\E[M\,\mathbf 1_{\{D\le j\}}]$ and $\bar u_j(t)$ be the analogous quantities. The root is a leaf with probability $e^{-c(T-t)}$, and $\{D\le j+1\}$ means that all the subtrees of its children have depth at most $j$. Hence, by the branching property, $u_0(t)=\alpha_0e^{-c(T-t)}$, $\bar u_0(t)=\beta_0e^{-c(T-t)}$, and $(u_{j+1},\bar u_{j+1})$ is the right-hand side of \eqref{eq:majorant-system} evaluated at $(u_j,\bar u_j)$; note that $p_k\cdot q_k/p_k=q_k$. If a root exists, let $(U_*,\bar U_*)$ be the solution of (a) with $r=r_*$. Since $U_*\ge u_0$, $\bar U_*\ge\bar u_0$ and the right-hand side is nondecreasing, $u_j\le U_*$ and $\bar u_j\le\bar U_*$ for all $j$. By monotone convergence, $(u,\bar u)=\lim_j(u_j,\bar u_j)$ is a bounded solution of \eqref{eq:majorant-system}, so by (a) it is associated with a root $r=u(0)\,\beta_0\le U_*(0)\,\beta_0=r_*$; hence $r=r_*$ and $(u,\bar u)=(U_*,\bar U_*)$. If no root exists, $(u,\bar u)$ cannot be bounded, since it would then solve \eqref{eq:majorant-system}.
\end{proof}

\begin{proof}[Proof of Proposition~\ref{prop:phi-branch-lamc}]
\emph{Step 1: finiteness.} By (i) and Lemma~\ref{lem:tree-size}, $D<\infty$ a.s., so that $\Phi^\circ_{t,x}$ and $\Phi^\bullet_{t,x}$ are well defined.

\emph{Step 2: square integrability.} Let $M_1$ be the functional of Lemma~\ref{lem:majorant}(c) with $q_k=b_k$, $\alpha_0=A$ and $\beta_0=B$, and $M_1^\circ$, $M_1^\bullet$ its values for a white, resp.\ black, root born at $t$. Since $|a_k(s^\circ,z)|\le b_k$, $|\varphi_T|\le A$ and $|\bar\varphi_T|\le B$, we have $|\Phi|\le M_1$ pathwise. Moreover, $M_1^2$ is the functional of Lemma~\ref{lem:majorant}(c) with $q_k=b_k^2/p_k$, $\alpha_0=A^2$ and $\beta_0=B^2$, because $(b_K/p_K)^2=(b_K^2/p_K)/p_K$. Condition (ii) is $\Psi(r)\le r$ for this choice, so Lemma~\ref{lem:majorant}(b)-(c) gives $\E[\Phi^2]\le\E[M_1^2]=U_2(t)<\infty$, where $U_2$ is the corresponding bounded solution: $\Phi\in L^2$. By Cauchy-Schwarz, $\E[M_1]\le\sqrt{U_2(t)}$ is bounded on $[0,T]$, so by Lemma~\ref{lem:majorant}(c) the equation $r=AB\exp\big(cT(\sum_k b_kr^k-1)\big)$ has a smallest root $r_1$, and
\begin{equation}\label{eq:majorant-L1}
\begin{aligned}
\E\big[|\Phi^\circ_{t,x}|\big]&\le\E[M^\circ_1]=U_1(t):=A\,e^{c(T-t)(\sum_k b_kr_1^k-1)},\\
\E\big[|\Phi^\bullet_{t,x}|\big]&\le\E[M^\bullet_1]=\bar U_1(t):=B\,e^{c(T-t)(\sum_k b_kr_1^k-1)} .
\end{aligned}
\end{equation}

\emph{Step 3: mild solution.} Let $u(t,x):=\E[\Phi^\circ_{t,x}]$ and $\bar u(t,x):=\E[\Phi^\bullet_{t,x}]$. Consider a white root born at $(t,x)$, with lifetime $\tau$, mark $K$ and path $x+W$. On $\{\tau\ge T-t\}$, which has probability $e^{-c(T-t)}$, $\Phi=\varphi_T(x+W_{T-t})$. On $\{\tau<T-t\}$, with $s:=t+\tau$ and $z:=x+W_\tau$, the branching property gives
\[
\E\big[\Phi\,\big|\,\tau,K,W\big]=\frac{a_K(s,z)}{p_K}\,u(s,z)^{K+1}\,\bar u(T-s,z)^{K}.
\]
The bound $|\Phi|\le M_1$ justifies Fubini's theorem. Averaging over $K\sim(p_k)$, over $W$, and over $\tau$ with density $c\,e^{-cr}$ on $[0,T-t)$ yields the first line of \eqref{eq:mild}. The second line follows in the same way, the factor of a black root branching at $s$ being $a_K(T-s,z)/p_K$. Hence $(u,\bar u)$ is a mild solution, and $|u|\le U_1$, $|\bar u|\le\bar U_1$ by \eqref{eq:majorant-L1}.

\emph{Step 4: uniqueness in the dominated class.} Let $(\varphi,\bar\varphi)$ be a mild solution with $|\varphi(t,x)|\le U_1(t)$ and $|\bar\varphi(t,x)|\le\bar U_1(t)$. For $j\ge0$, let $\Phi_j$ be the functional of the tree stopped at generation $j$: particles of generation $j$ are not evolved and contribute $\varphi$ (white) or $\bar\varphi$ (black), evaluated at their birth time and position. Iterating \eqref{eq:mild} $j$ times, with the branching property as in Step~3, gives $\varphi(t,x)=\E[\Phi_j]$ for a white root born at $(t,x)$. On $\{D<j\}$ no particle of generation $j$ exists, so that $\Phi_j=\Phi$. Let $M_{1,j}$ be $M_1$ stopped at generation $j$ in the same way, with contributions $U_1$, $\bar U_1$ for the particles of generation $j$; then $|\Phi_j|\le M_{1,j}$, and $\E[M_{1,j}]=U_1(t)$ because $(U_1,\bar U_1)$ is a fixed point of \eqref{eq:majorant-system} with $q_k=b_k$. Therefore
\[
\big|\varphi(t,x)-\E\big[\Phi\,\mathbf 1_{\{D<j\}}\big]\big|
\le\E\big[M_{1,j}\,\mathbf 1_{\{D\ge j\}}\big]
=U_1(t)-\E\big[M_1\,\mathbf 1_{\{D<j\}}\big]
\xrightarrow[j\to\infty]{}U_1(t)-\E[M_1]=0,
\]
by monotone convergence, Step~1 and \eqref{eq:majorant-L1}. Since $\E[\Phi\,\mathbf 1_{\{D<j\}}]\to\E[\Phi]=u(t,x)$ by dominated convergence, $\varphi=u$; similarly $\bar\varphi=\bar u$.
\end{proof}

\begin{remark}[Checking (i)-(ii) in practice]\label{rmk:branching-conditions}
(a) \emph{Linear cases} ($F\equiv0$, or $F=V$ as in Appendix~\ref{sec:potential-running-cost}): $K=0$, so (i) holds trivially, and $Q$ is constant, so (ii) holds with $r=A^2B^2e^{cT(b_0^2/p_0-1)}$, for every horizon $T$. This covers all the experiments of Section~\ref{sec:experiments} except the Gross-Pitaevskii ones, as soon as the trained potentials are bounded.

(b) \emph{Gross-Pitaevskii coupling} $F(m)=\kappa m$, i.e.\ $a_0=1$ and $a_1=-\kappa$ in \eqref{canonical_form}: (i) reads $cTp_1<1$. Since $r=\alpha e^{\beta r}$ has a positive root if and only if $e\,\alpha\beta\le1$, (ii) reads
\[
e\,cT\,\frac{\kappa^2}{p_1}\,A^2B^2\,e^{cT(1/p_0-1)}\le1 .
\]
In the Gross-Pitaevskii experiments of Section~\ref{sec:experiments} ($c\le0.3$, with the horizon $T=2$ of Appendix~\ref{app:implementation-details}), (i) holds whatever $p_1$, since $cTp_1\le0.6$. Condition (ii) involves the sup norms $A$, $B$ of the trained potentials and has to be checked a posteriori; when it fails, the trees remain finite by (i), but the estimator may have infinite variance.

(c) \emph{The conditions cannot be dropped.} If $a_k\ge0$ and $\varphi_T\equiv A$, $\bar\varphi_T\equiv B$ are constants, then $\Phi=M_1$, and by Lemma~\ref{lem:majorant}(c) $t\mapsto\E[\Phi^\circ_{t,x}]$ or $t\mapsto\E[\Phi^\bullet_{t,x}]$ is unbounded as soon as $r=AB\exp\big(cT(\sum_ka_kr^k-1)\big)$ has no root, although all the data are bounded. For instance, for $a=(0,1)$, $p=(\frac12,\frac12)$ and $c=T=1$, a root exists if and only if $AB\le1$.
\end{remark}

\section{Feynman-Galton-Watson diagrams} \label{feynman-galton-watson-diagrams}
Each realization of the branching system is a Galton-Watson tree: lines correspond to Brownian trajectories between branching events, and vertices represent branchings. 
This is directly analogous to tree Feynman diagrams: dashed lines depict white particles (\(\varphi\)), solid lines depict black particles (\(\bar\varphi\)), and trivalent vertices encode the cubic nonlinearity. 
For the cubic nonlinearity $F_p(m)=m$ in \eqref{canonical_form} (i.e.\ $a_0=0$, $a_1=1$, and we take $p_1=1$), the only branching is \(k=1\):
\[
\circ \;\to\; 2\,\circ + 1\,\bullet, 
\qquad 
\bullet \;\to\; 1\,\circ + 2\,\bullet,
\]
so each vertex corresponds to one incoming line splitting into three outgoing lines.  
Iterating these vertices yields a perturbative expansion in the coupling constant \(c\): a diagram with \(n\) branchings contributes a weight of order \(c^n\).  

\no Importantly, if we truncate the Galton-Watson expansion to order $0$ (no branching), the dynamics are given only by free Brownian propagation. 
In this case, the system reduces to the \emph{classical Schr\"odinger bridge}, and the PDEs collapse to the linear forward-backward heat equations. 
Higher-order branchings then represent nonlinear corrections beyond this baseline. Appendix~\ref{app:branching-sanity-toy} gives a scalar toy sanity check of the branching scheme for the $T{-}t$ coupling in \eqref{PDE1}-\eqref{PDE2}.
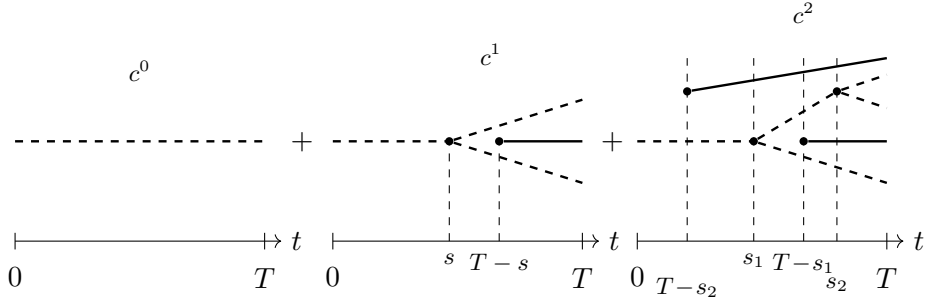
\begin{figure}[H]
\centering
\begin{tikzpicture}[x=1.1cm,y=1.1cm]

\begin{scope}[shift={(0,0)}]
  \draw[whitep] (0,0) -- (3,0);
  \node[lbl,above] at (1.5,0.6) {$c^0$};

  \draw[->] (0,-1.2) -- (3.2,-1.2) node[right] {$t$};
  \draw (0,-1.3) -- (0,-1.1);
  \node[below=3pt] at (0,-1.3) {$0$};
  \draw (3,-1.3) -- (3,-1.1);
  \node[below=3pt] at (3,-1.3) {$T$};
\end{scope}

\node at (3.45,0) {$+$};

\begin{scope}[shift={(3.82,0)}]
  \draw[whitep] (0,0) -- (1.4,0);
  \node[vtx] (s1) at (1.4,0) {};

  \draw[whitep] (s1) -- (3,0.5);
  \draw[whitep] (s1) -- (3,-0.5);

  \node[vtx] (Ts1) at (2.0,0) {};
  \draw[blackp] (Ts1) -- (3,0);

  \node[lbl,above] at (1.9,0.8) {$c^1$};

  \draw[->] (0,-1.2) -- (3.2,-1.2) node[right] {$t$};
  \draw (0,-1.3) -- (0,-1.1);
  \node[below=3pt] at (0,-1.3) {$0$};
  \draw (3,-1.3) -- (3,-1.1);
  \node[below=3pt] at (3,-1.3) {$T$};

  \draw[dashed] (1.4,-1.2) -- (1.4,0);
  \node[lbl,below=1pt] at (1.4,-1.2) {$\tiny s$};

  \draw[dashed] (2.0,-1.2) -- (2.0,0);
  \node[lbl,below=1pt] at (2.0,-1.2) {$\tiny T-s$};
\end{scope}

\node at (7.18,0) {$+$};

\begin{scope}[shift={(7.48,0)}]
  \draw[whitep] (0,0) -- (1.4,0);
  \node[vtx] (WsOne) at (1.4,0) {};

  \draw[whitep] (WsOne) -- (2.4,0.6);  
  \draw[whitep] (WsOne) -- (3,-0.5);

  \node[vtx] (BsOne) at (2.0,0) {};
  \draw[blackp] (BsOne) -- (3,0.0);

  \node[vtx] (WsTwo) at (2.4,0.6) {};

  \draw[whitep] (WsTwo) -- (3,0.8);
  \draw[whitep] (WsTwo) -- (3,0.4);

  \node[vtx] (BsTwo) at (0.6,0.6) {};
  \draw[blackp] (BsTwo) -- (3,1.0);

  \node[lbl,above] at (2.0,1.3) {$c^2$};

  \draw[->] (0,-1.2) -- (3.2,-1.2) node[right] {$t$};
  \draw (0,-1.3) -- (0,-1.1);
  \node[below=3pt] at (0,-1.3) {$0$};
  \draw (3,-1.3) -- (3,-1.1);
  \node[below=3pt] at (3,-1.3) {$T$};

  \draw[dashed] (1.4,-1.2) -- (1.4,1.0);
  \node[lbl,below=1pt] at (1.4,-1.2) {$\tiny s_1$};

  \draw[dashed] (2.0,-1.2) -- (2.0,1.0);
  \node[lbl,below=1pt] at (2.0,-1.2) {$\tiny T\!-\!s_1$};

  \draw[dashed] (2.4,-1.2) -- (2.4,1.0);
  \node[lbl,below=10pt] at (2.4,-1.2) {$\tiny s_2$};

  \draw[dashed] (0.6,-1.2) -- (0.6,1.0);
  \node[lbl,below=10pt] at (0.6,-1.2) {$\tiny T\!-\!s_2$};
\end{scope}

\end{tikzpicture}

\caption{
First three contributions in the perturbative expansion (Gross-Pitaevskii cubic branching).  
Dashed lines represent white particles ($\varphi$), solid lines represent black particles ($\bar\varphi$). Weights with $a_1=p_1=1$:  
order $c^0$, no branching, i.e.\ classical Schr\"odinger bridge / linear heat kernel;  
order $c^1$ with one branching at $s$ (white offspring starting at $s$, black offspring at $T-s$);  
order $c^2$ with two branchings at $s_1,s_2$ : offspring of the same type as the parent start at $s_i$, while offspring of the other type start at $T-s_i$.  
The horizontal time axes (all at the same vertical level) show time from $0$ to $T$, with dashed vertical lines indicating the splitting times.
}
\end{figure}

\section{Implementation details}
\label{app:implementation-details}

Training and sampling both rely on stochastic quantities built from branching trees and on differentiating $\log \widehat{\varphi}(t,x)$ with respect to $x$ (Algorithm~\ref{alg:meta-sample-ssb}).
When coordinates have very different scales or heavy tails, (i) products of terminal factors in the branching estimator can span many orders of magnitude, inflating Monte Carlo variance, and (ii) autograd through $\log \varphi_\theta$ and the reweighted score becomes ill-conditioned relative to the isotropic unit-variance reference noise in the SDE.
We therefore normalize inputs before bridge estimation, using one of the two mechanisms below depending on dimensionality and data modality.

\subsection{Covariance whitening in data space}

Let $X_0\sim\mu_0$ and $X_T\sim\mu_T$. We define
\begin{align*}
\hat{X}_0 &:= C_0^{-\frac{1}{2}} \bigl(X_0 - \mathbb{E}^{\mu_0}[X_0]\bigr), \quad C_0 := \mathrm{Cov}(\mu_0), \\[0.5em]
\hat{X}_T &:= \sigma \, C_T^{-\frac{1}{2}} \bigl(X_T - \mathbb{E}^{\mu_T}[X_T]\bigr), \quad C_T := \mathrm{Cov}(\mu_T), \quad \sigma^2 = 2,
\end{align*}
train the bridge on $\hat{\mu}_0 := \mathrm{Law}(\hat{X}_0)$ and $\hat{\mu}_T := \mathrm{Law}(\hat{X}_T)$, and map generated samples back to the original coordinates by the inverse affine map.
If $\mu_0$ and $\mu_T$ are Gaussian, then $\hat{\mu}_0\sim\mathcal{N}(0,I)$, $\hat{\mu}_T\sim\mathcal{N}(0,\sigma^2 I)$ and $\hat{\mu}_T=\hat{\mu}_0\star\mathcal{N}(0,I)$, which matches the classical Schr\"odinger bridge with $(\varphi_{\mathrm{NN}},\bar{\varphi}_{\mathrm{NN}})=(0,0)$ at horizon $T=1$.
In practice we set $T=2$ to retain nontrivial interpolation; when $\mu_T$ is close to Gaussian, initializing $(\varphi_{\mathrm{NN}},\bar{\varphi}_{\mathrm{NN}})=(0,0)$ is already near optimal with little additional training.

\subsection{Latent representations for imaging}

For high-dimensional images, empirical covariances are singular or badly conditioned (e.g.\ near-constant background pixels), so $C_T^{-1/2}$ is unstable and pixel-space whitening is inappropriate.
We instead encode images with a pretrained VAE (MNIST) or ALAE (FFHQ), yielding low-dimensional codes that are approximately standardized and better aligned with the isotropic diffusion reference; the bridge is trained in that latent space and decoded for visualization.

\section{Results details}
\label{app:results-details}

\subsection{2D synthetic experiments}

We adopt synthetic two-dimensional distributions commonly used in the literature \citep{bortoli2021diffusion}:  
\begin{itemize}
  \item \textbf{$\mu_0:=$ Gaussian $\to$ $\mu_T:=$ Mixture of Gaussians.}  
  The source $\mu_0$ is a centered Gaussian with variance $\sigma^2 I$ ($\sigma=5$), and the target $\mu_T$ is a mixture of $8$ Gaussians. 
  \item \textbf{$\mu_0:=$ Moon curves $\to$ $\mu_T:=$ Swiss curves.}  
  The source $\mu_0$ is a composition of two Moon-shaped distributions, and the target $\mu_T$ is a Swiss curve.
\end{itemize}  

\no The recovered Schr\"odinger bridges for these two cases are shown in Figure~\ref{fig:generated_vs_true_gaussians}.  
The temporal evolution of the interpolating flow illustrates how the source progressively deforms into the target (Figures~\ref{fig:evolution_strip}).  

\no Training dynamics are summarized in Figure~\ref{fig:generation_evolution_grid}, see Appendix \ref{additional-figures}.  As the number of epochs increases, the generated terminal distribution $\mathrm{Law}(X_T)$ gradually converges toward the target law $\mu_T$ (Figure~\ref{fig:generation_evolution_grid}).  
At the same time, the potentials $(\varphi_\theta,\bar\varphi_\eta)$ sharpen as training progresses, reflecting improved discrimination between source and target (Figures~\ref{fig:phi_evolution_grid} and \ref{fig:phi_bar_evolution_grid}, see Appendix \ref{additional-figures}).  

\begin{figure}[H]
  \centering
  \begin{subfigure}{0.48\textwidth}
    \centering
    \includegraphics[width=\linewidth]{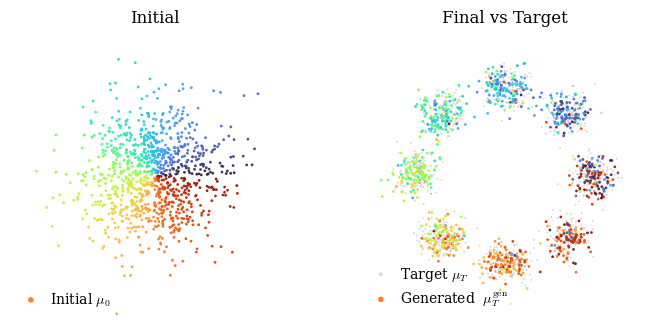}
    \caption{Gaussian $\to$ mixture of Gaussians.}
    \label{fig:gaussian_to_mixture}
  \end{subfigure}\hfill
  \begin{subfigure}{0.48\textwidth}
    \centering
    \includegraphics[width=\linewidth]{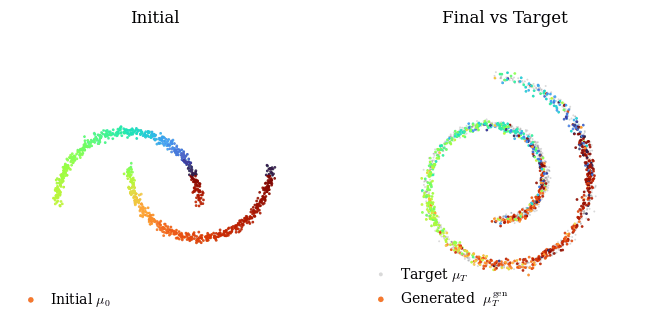}
    \caption{Moon curves $\to$ Swiss curve.}
    \label{fig:moon_to_swiss}
  \end{subfigure}
  \caption{Classical Schr\"odinger bridges in $\R^2$.}
  \label{fig:generated_vs_true_gaussians}
\end{figure}

\begin{figure}[H]
  \centering
  \includegraphics[width=\textwidth]{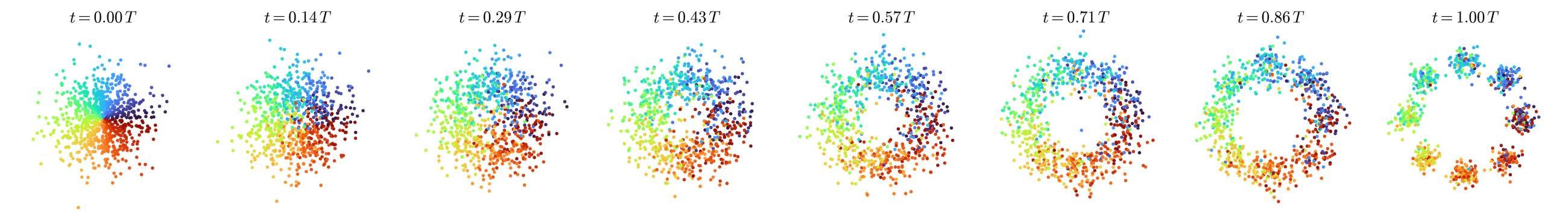}
  \caption{Extended view of the temporal evolution of $\mathrm{Law}(X_t)$ (Moons $\to$ Swiss).}
  \label{fig:evolution_strip_2}
  \end{figure}

  \subsection{Comparison on the 2D benchmark}
  \label{app:comparison-2d}
  
  We compare SBNL with LightSB \citep{korotin2023light}, LightSBB-M \citep{alouadi2026lightsbbmbridgingschrodingerbass}, DSBM \citep{shi2023diffusion} and BranchSBM \citep{tang2026branchedschrodingerbridgematching} on the Gaussian $\to$ mixture of $8$ Gaussians benchmark from Appendix~\ref{app:results-details}. We report (i) terminal samples for each method (Figure~\ref{fig:comparison_2D_samples_full}), (ii) the closed-form mixture $W_2$ distance between the generated terminal law and $\mu_T$ (Figure~\ref{fig:comparison_MW2_full}), and (iii) the per-mode coverage of the $8$ components (Figure~\ref{fig:comparison_per_mode_coverage_full}).
  
  \subsubsection{Computation of the comparison metrics}
  \label{app:comparison-metrics}
  
  Let $\hat\mu_T$ denote the empirical distribution of $N=10^4$ generated terminal samples and $\mu_T$ the target mixture of $8$ isotropic Gaussians in $\R^2$.
  
  \paragraph{Mixture Wasserstein $MW_2$.}
  We first fit a Gaussian mixture with $K=8$ components to the generated samples by expectation--maximization (sklearn's \texttt{GaussianMixture}, full covariances, restarted $5$ times, best log-likelihood kept), obtaining $\hat p = \sum_{k=1}^K \hat w_k \, \mathcal N(\hat m_k,\hat \Sigma_k)$. We compare it to the known target mixture $p^\star = \sum_{k=1}^K w_k^\star \, \mathcal N(m_k^\star,\Sigma_k^\star)$ using the \emph{mixture Wasserstein} distance \citep{delon2020wasserstein}
  \[
  MW_2^2(\hat p, p^\star) \;=\; \min_{\pi \in \Pi(\hat w, w^\star)} \; \sum_{k,\ell} \pi_{k\ell} \, W_2^2\!\bigl(\mathcal N(\hat m_k,\hat \Sigma_k),\, \mathcal N(m_\ell^\star,\Sigma_\ell^\star)\bigr),
  \]
  where the $W_2$ distance between Gaussians has the closed form
  $W_2^2(\mathcal N(m_1,\Sigma_1),\mathcal N(m_2,\Sigma_2)) = \|m_1-m_2\|^2 + \Tr\!\bigl(\Sigma_1+\Sigma_2 - 2(\Sigma_1^{1/2}\Sigma_2\Sigma_1^{1/2})^{1/2}\bigr)$, so $MW_2$ reduces to a small linear program on the $K \times K$ cost matrix and is computed in closed form.
  
  \paragraph{Sliced Wasserstein $SW_2$.}
  We use the standard Monte Carlo estimator of the sliced $2$-Wasserstein distance,
  \[
  SW_2^2(\hat\mu_T,\mu_T) \;=\; \frac{1}{L}\sum_{\ell=1}^{L} W_2^2\!\bigl(P_{\theta_\ell\sharp}\hat\mu_T,\, P_{\theta_\ell\sharp}\mu_T\bigr),
  \]
  where $\theta_\ell$ are $L=200$ directions drawn uniformly on $\mathbb S^{1}$ and $P_\theta x = \theta^\top x$; each $1$D $W_2^2$ is evaluated exactly on the sorted projected samples (with $10^4$ samples for both $\hat\mu_T$ and $\mu_T$).
  
  \paragraph{Empirical $W_2^{\text{emp}}$.}
  We solve the discrete optimal transport between $\hat\mu_T$ and an i.i.d.\ sample of size $10^4$ from $\mu_T$ using the POT library (exact LP solver on the squared Euclidean cost), reporting $W_2^{\text{emp}} = (\sum_{ij} \pi_{ij}^\star \|x_i - y_j\|^2)^{1/2}$.
  
  \paragraph{Energy distance $E$.}
  We use the kernel-based energy distance \citep{szekely2004testing}, estimated by U-statistics on the same $10^4$ samples from $\hat\mu_T$ and $\mu_T$:
  \[
  E^2(\hat\mu_T,\mu_T) \;=\; 2\,\mathbb E\|X-Y\| - \mathbb E\|X-X'\| - \mathbb E\|Y-Y'\|, \quad X,X'\sim\hat\mu_T,\; Y,Y'\sim\mu_T.
  \]
  $E \ge 0$, with $E=0$ iff the two distributions coincide.
  
  \paragraph{Moment and coverage diagnostics.}
  The mean error $|\Delta\mu|=\|\mathbb E_{\hat\mu_T}[X]-\mathbb E_{\mu_T}[X]\|$ and the covariance error $\|\Delta\Sigma\|_F = \|\mathrm{Cov}(\hat\mu_T)-\mathrm{Cov}(\mu_T)\|_F$ are computed from empirical moments. The per-mode coverage assigns each generated sample to its nearest component center in $\mu_T$ and reports the resulting empirical frequencies (Figure~\ref{fig:comparison_per_mode_coverage_full}); the uniform reference is $1/K = 1/8$.

  \paragraph{Confidence intervals.}
  Each metric is a random variable of the generated sample, so every value reported in Table~\ref{tab:comparison_2D_metrics} and Figures~\ref{fig:comparison_MW2_full}-\ref{fig:comparison_W2emp_full} is accompanied by a bootstrap confidence interval. 
  For every method and every metric we draw $B=5$ resamples with replacement of both the prediction tensor (size $N=10^4$) and a freshly drawn reference sample from $\mu_T$, then recompute the entire metric pipeline on each replica-including refitting the $8$-component GMM with a fresh seed for $MW_2$, so that variability from the GMM fit is also captured. 
  The reported point estimate is the mean across replicas and the $\pm$ value is the sample standard deviation (with Bessel correction). 
  For the empirical Wasserstein distance, the $10^4 \times 10^4$ EMD is too heavy inside the bootstrap loop, so we down-sample to $2000$ points per side in each replica (hence the $W_2^{\text{emp}}$ label in Figure~\ref{fig:comparison_W2emp_full}). 
  Per-mode coverage error bars in Figure~\ref{fig:comparison_per_mode_coverage_full} are obtained the same way. 
  Each method is evaluated from a single trained checkpoint, so the reported intervals reflect \emph{sample/estimator} uncertainty rather than training-seed variability.

  \begin{figure}[H]
  \centering
  \includegraphics[width=\textwidth]{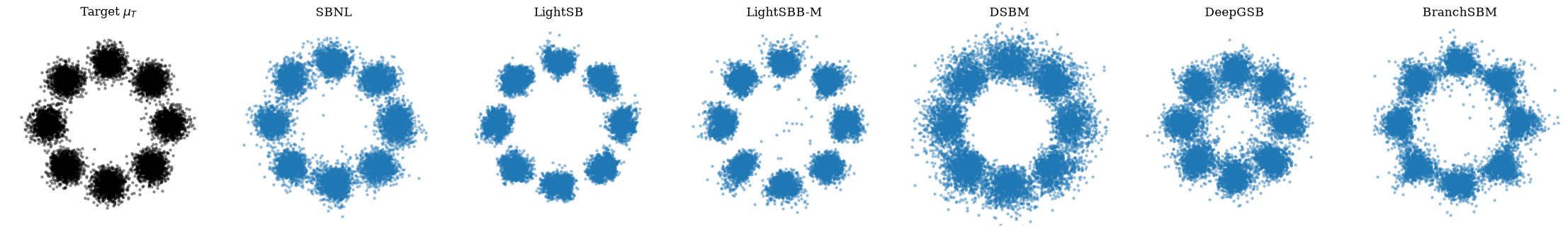}
  \caption{Terminal samples on Gaussian $\to$ mixture of $8$ Gaussians for each method.}
  \label{fig:comparison_2D_samples_full}
  \end{figure}
  

  \begin{figure}[H]
  \centering
  \begin{subfigure}[t]{0.27\textwidth}
    \centering
    \includegraphics[width=\linewidth]{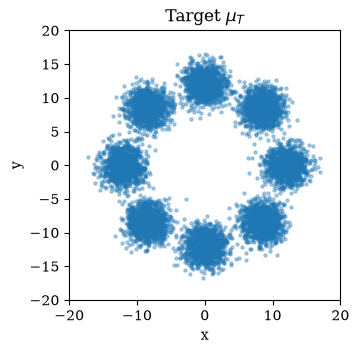}
    \caption{Target $\mu_T$}
    \label{fig:comparison_alone_target}
  \end{subfigure}\hfill
  \begin{subfigure}[t]{0.27\textwidth}
    \centering
    \includegraphics[width=\linewidth]{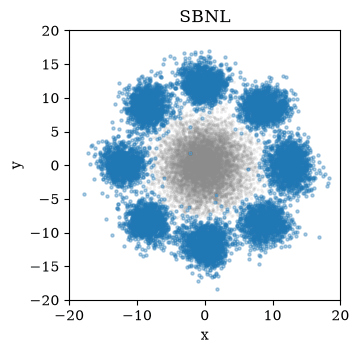}
    \caption{SBNL}
    \label{fig:comparison_alone_sbnl}
  \end{subfigure}\hfill
  \begin{subfigure}[t]{0.27\textwidth}
    \centering
    \includegraphics[width=\linewidth]{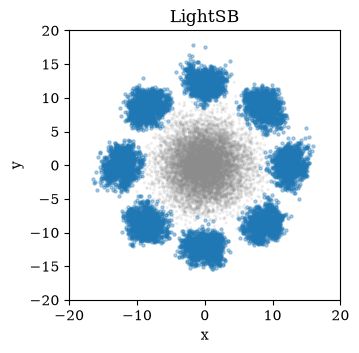}
    \caption{LightSB}
    \label{fig:comparison_alone_lightsb}
  \end{subfigure}\\[0.5em]
  \begin{subfigure}[t]{0.27\textwidth}
    \centering
    \includegraphics[width=\linewidth]{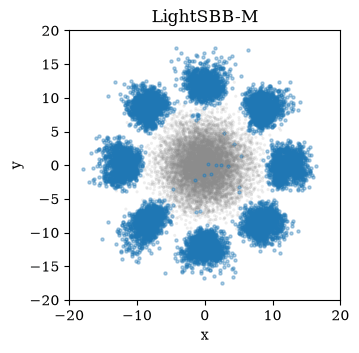}
    \caption{LightSBB-M}
    \label{fig:comparison_alone_lightsb_m}
  \end{subfigure}\hfill
  \begin{subfigure}[t]{0.27\textwidth}
    \centering
    \includegraphics[width=\linewidth]{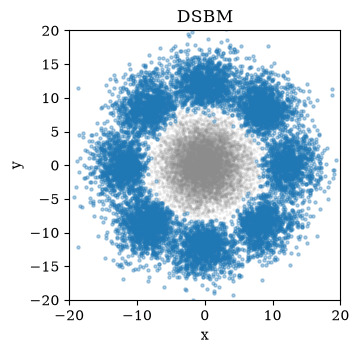}
    \caption{DSBM}
    \label{fig:comparison_alone_dsbm}
  \end{subfigure}\hfill
  \begin{subfigure}[t]{0.27\textwidth}
    \centering
    \includegraphics[width=\linewidth]{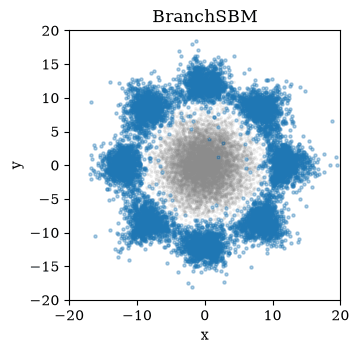}
    \caption{BranchSBM}
    \label{fig:comparison_alone_branchsbm}
  \end{subfigure}\hfill
  \begin{subfigure}[t]{0.27\textwidth}
    \centering
    \includegraphics[width=\linewidth]{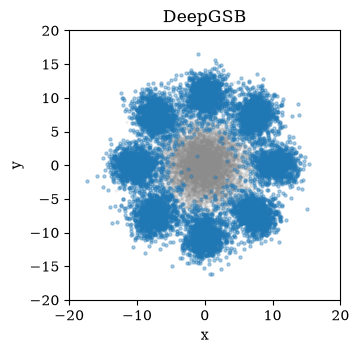}
    \caption{DeepGSB}
    \label{fig:comparison_alone_deepgsb}
  \end{subfigure}
  \caption{Terminal samples per method on the Gaussian $\to$ mixture of $8$ Gaussians benchmark.}
  \label{fig:comparison_alone_grid}
  \end{figure}

  \begin{figure}[H]
  \centering
  \begin{subfigure}[t]{0.49\textwidth}
    \centering
    \includegraphics[width=\linewidth]{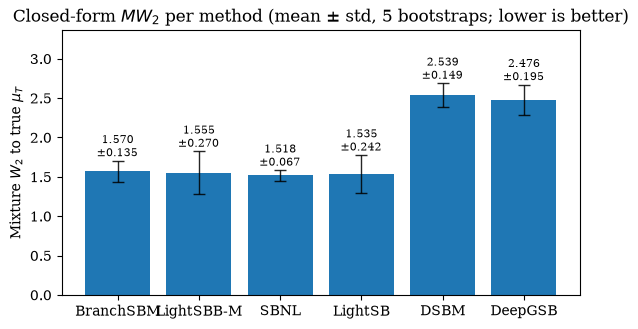}
    \caption{Closed-form mixture $W_2$ between generated terminal law and $\mu_T$.}
    \label{fig:comparison_MW2_full}
  \end{subfigure}\hfill
  \begin{subfigure}[t]{0.49\textwidth}
    \centering
    \includegraphics[width=\linewidth]{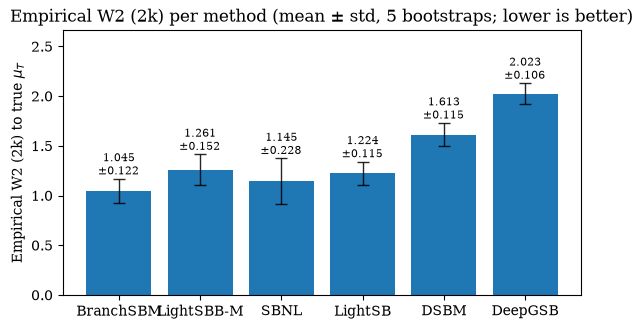}
    \caption{Empirical $W_2$ between generated terminal law and $\mu_T$.}
    \label{fig:comparison_W2emp_full}
  \end{subfigure}
  \caption{Per-method comparison on the Gaussian $\to$ mixture of $8$ Gaussians benchmark (lower is better; mean $\pm$ std over $5$ bootstraps).}
  \label{fig:comparison_W2_full}
  \end{figure}
  
  \begin{figure}[H]
  \centering
  \includegraphics[width=0.7\textwidth]{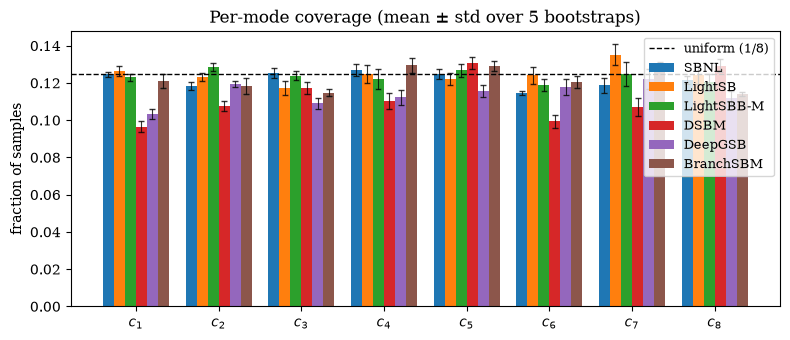}
  \caption{Per-mode coverage of the $8$ target components. The dashed line marks the uniform reference $1/8$.}
  \label{fig:comparison_per_mode_coverage_full}
  \end{figure}

\subsection{Potential running cost}
\label{sec:potential-running-cost}

We now consider a \emph{state-dependent} running cost of potential type,
\[
f\big(m(t,\cdot)\big)=\int_{\R^d} V(x)\,m(t,x)\,\dd x
\qquad\Longrightarrow\qquad
F(t,x,\varphi,\bar\varphi)=V(x),
\]
so that the canonical PDE \eqref{PDE1} reads $\partial_t\varphi+\tfrac12\Delta\varphi=c\,\varphi\,\big(1-(1-V(x))\big)$. The interaction is thus fully specified by a single \textbf{callable} $V:\R^d\to\R$, passed to the model and evaluated on the fly during both training and sampling to encode the environment.

\paragraph{Branching mechanism.}
Matching the canonical form $F=1-F_p$ with $F_p(t,x,m)=1-V(x)$ gives a single degree-zero coefficient $a_0(x)=1-V(x)$ (and $a_k=0$ for $k\ge1$). With $p_0=1$, the only branching outcome is that a particle continues as a single copy of the \emph{same} type, and each branching event reweights the tree by
\[
\mathcal W\;\longrightarrow\;\mathcal W\cdot\frac{a_0(X_s)}{p_0}=\mathcal W\cdot\big(1-V(X_s)\big).
\]
There are no offspring and no type flips: the branching estimator of Proposition~\ref{prop:phi-branch-lamc} reduces to a pure Feynman-Kac weighting of Brownian paths (Appendix~\ref{app:special-cases}, confining potentials), with $V$ bounded so that $1-V>0$ keeps the weights positive.

\paragraph{Training and sampling.}
Because $F$ is state- but not density-dependent, the system \eqref{PDE1}-\eqref{PDE2} is linear and $\varphi(0,x)$ is estimated by the reweighted trees above; the potentials are trained with the MMD objective \eqref{eq:J-mc}. Two practical ingredients make the importance-sampled estimator stable. First, the interior weights $W=\widehat A\,\bar\varphi_\eta/\mu^*$ (resp.\ $\widehat B\,\varphi_\theta/\mu^*$) are heavy-tailed when $\mu^*$ is too narrow, so we take a \emph{wide} Gaussian reference $\mu^*$ covering the region actually explored by the branching diffusion. Second, we use \emph{self-normalized} importance sampling (dividing each weight vector by its own sum, gradient preserved), which turns the interior expectations into low-variance ratio estimators comparing the \emph{normalized} residual measures. At sampling time, the drift $\nabla\log\varphi$ is estimated by the same weighted trees (Algorithm~\ref{alg:meta-sample-ssb}), so $V$ deflects trajectories away from high-$V$ regions.

\paragraph{Endpoint offset trick.}
When $V$ has an asymmetric geometry (e.g.\ gates at different heights, a curved valley), we place $\mu_0$ and $\mu_T$ at \emph{different} anchor points aligned with the entrance/exit of the low-energy corridor, rather than symmetrically. This gives the score a clear valley to follow through the obstacle. Transporting mass over a distance $\sim\!\ell$ requires the drift to integrate to $\sim\!\ell$, so we also use a longer horizon $T$ (with a small rate $c$): too small a $T$ leaves the drift no time to route the mass around the barrier.

\begin{figure}[H]
\centering
\captionsetup{font=footnotesize}

\includegraphics[width=0.55\textwidth]{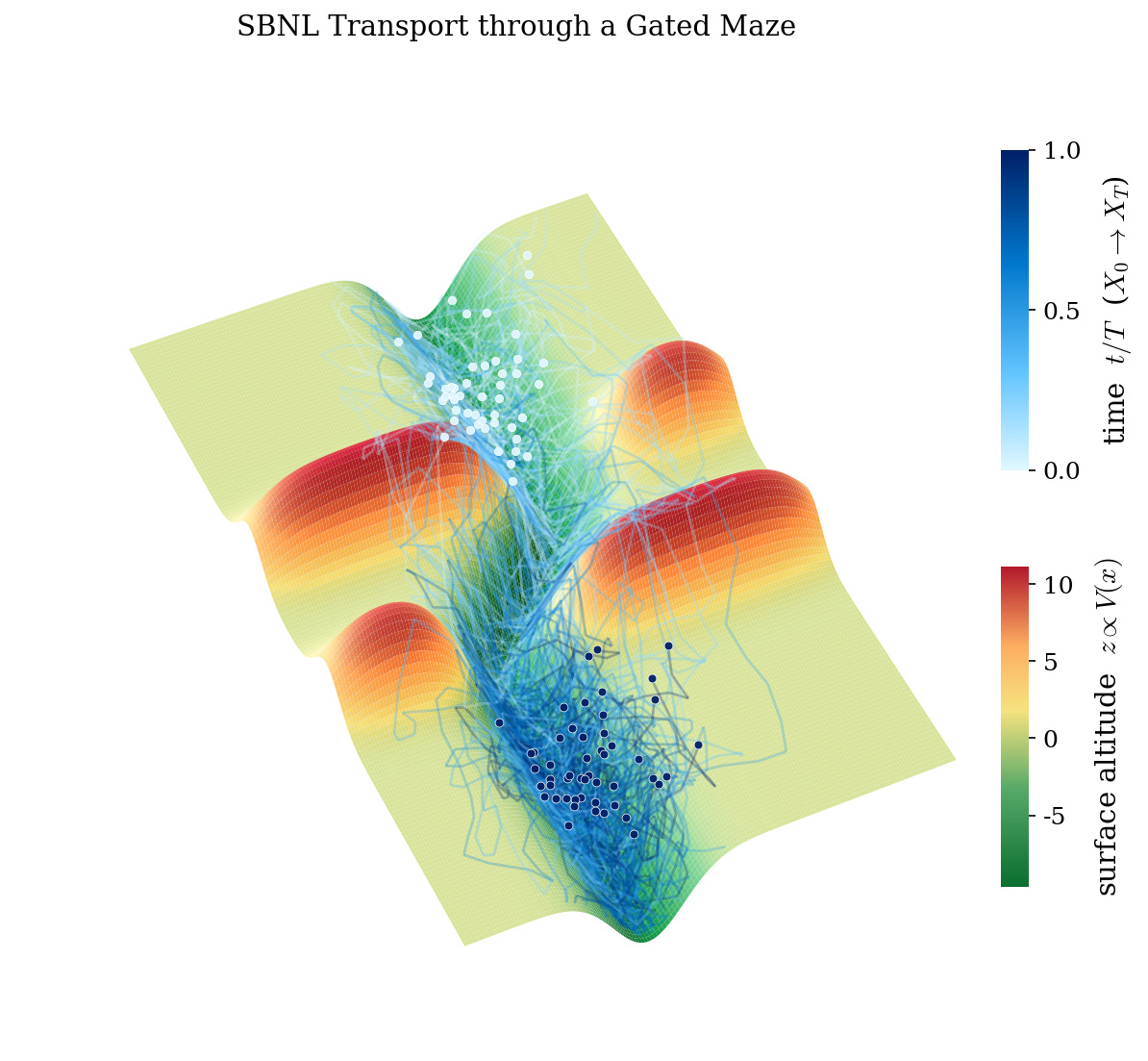}\\[-4pt]
\includegraphics[width=0.9\textwidth]{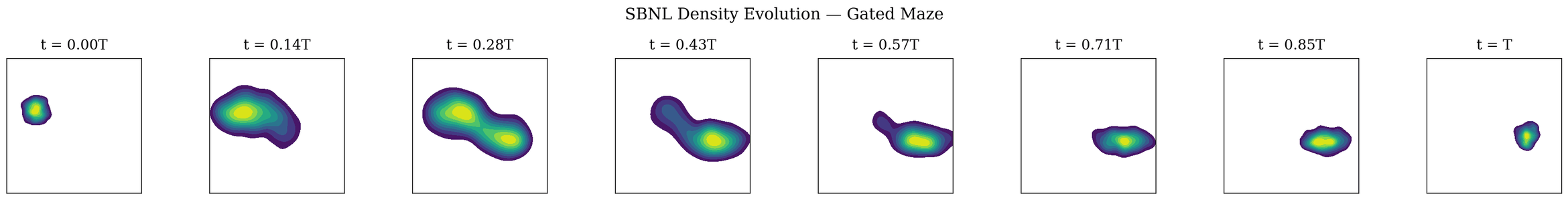}

\caption{Gated maze: SBNL transport and estimated density evolution $m(t,\cdot)=\varphi(t,\cdot)\bar\varphi(T-t,\cdot)$.}
\label{fig:potential_maze}
\end{figure}

\begin{figure}[H]
\centering
\captionsetup{font=footnotesize}

\includegraphics[width=0.55\textwidth]{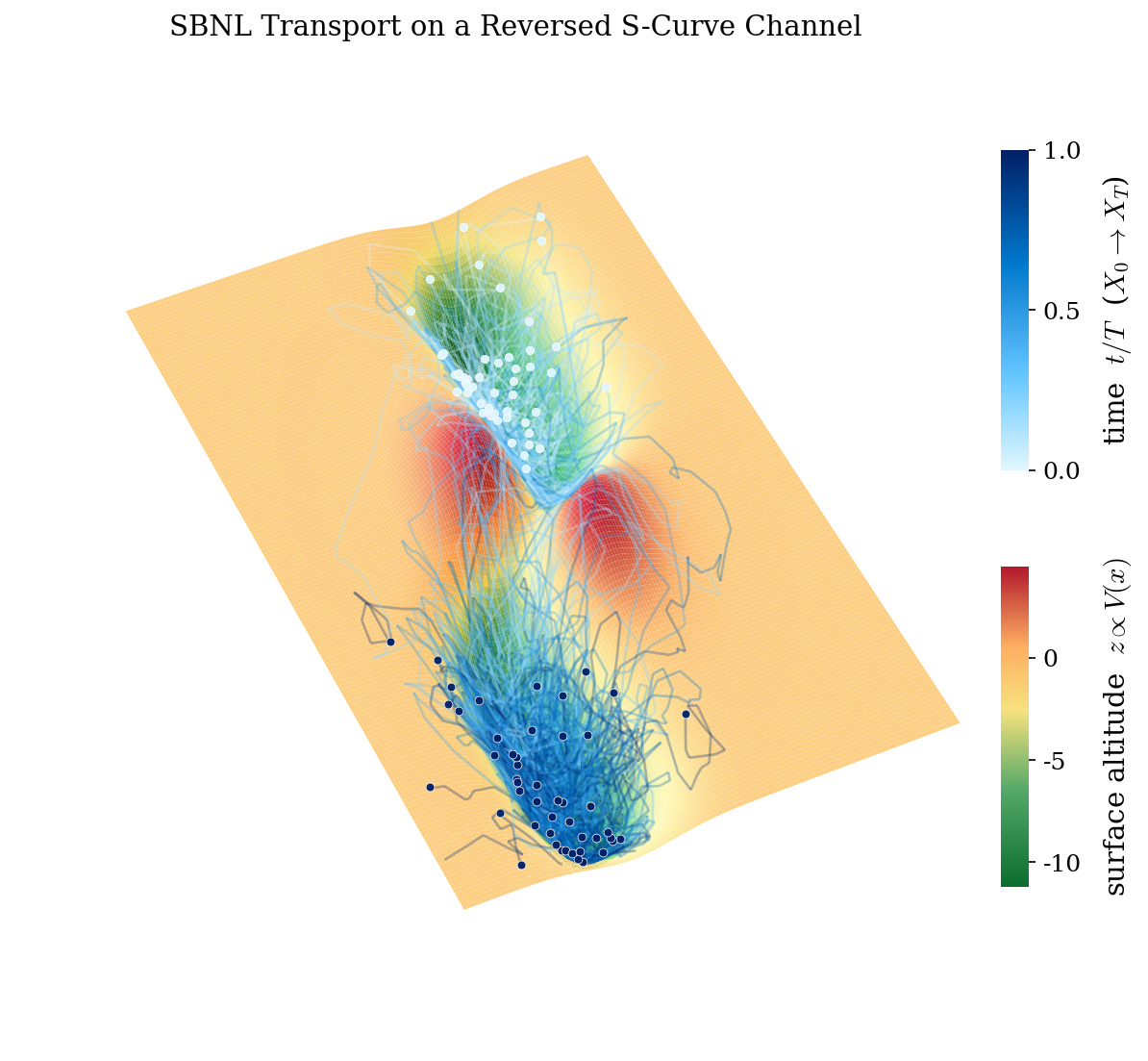}\\[-4pt]
\includegraphics[width=0.9\textwidth]{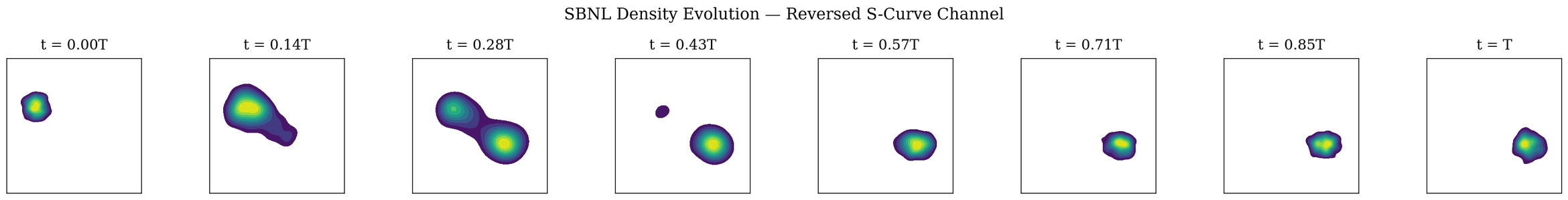}

\caption{S-curve: SBNL transport and estimated density evolution $m(t,\cdot)=\varphi(t,\cdot)\bar\varphi(T-t,\cdot)$.}
\label{fig:potential_scurve}
\end{figure}

\begin{figure}[H]
\centering
\captionsetup{font=footnotesize}

\includegraphics[width=0.55\textwidth]{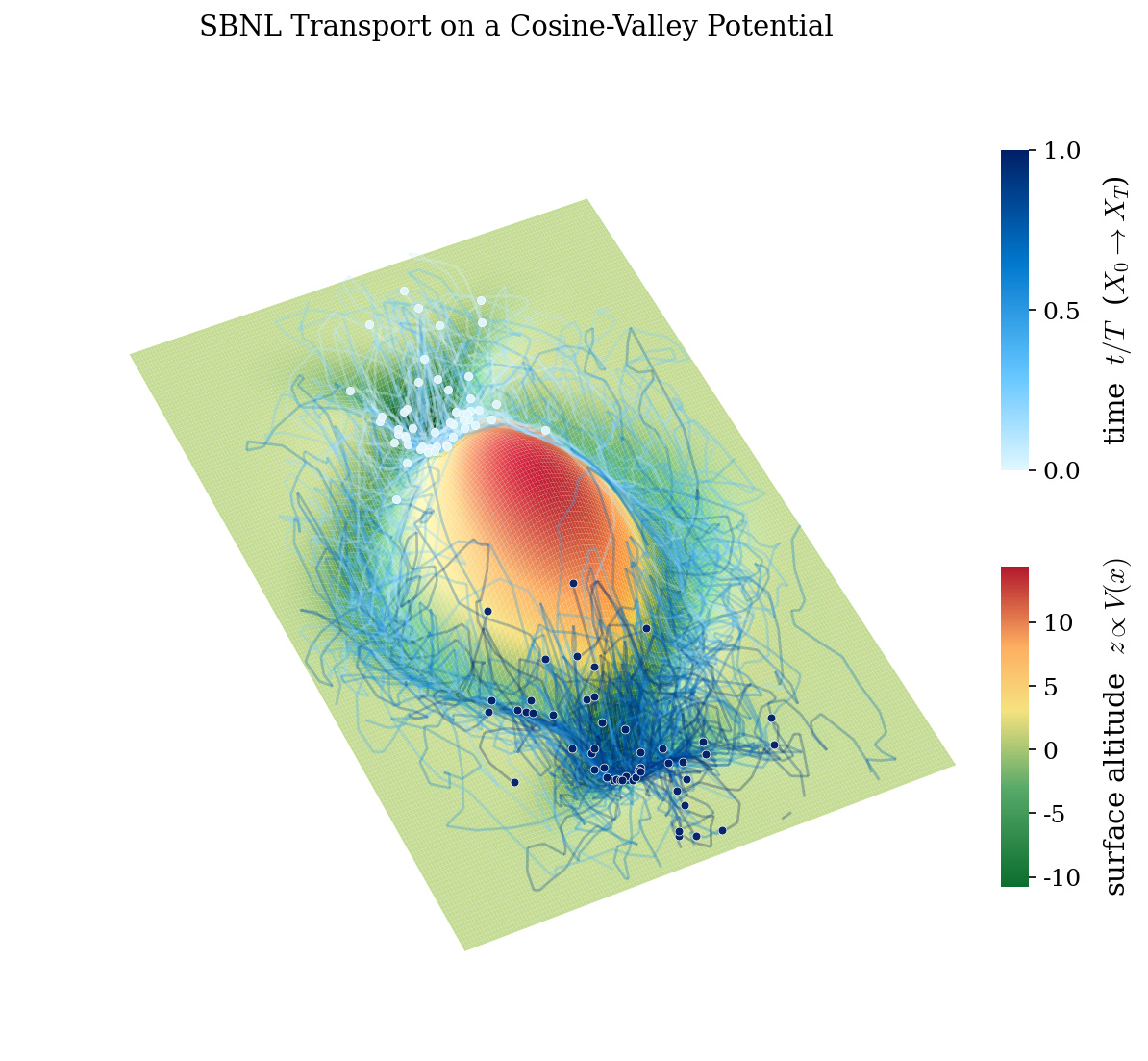}\\[-4pt]
\includegraphics[width=0.9\textwidth]{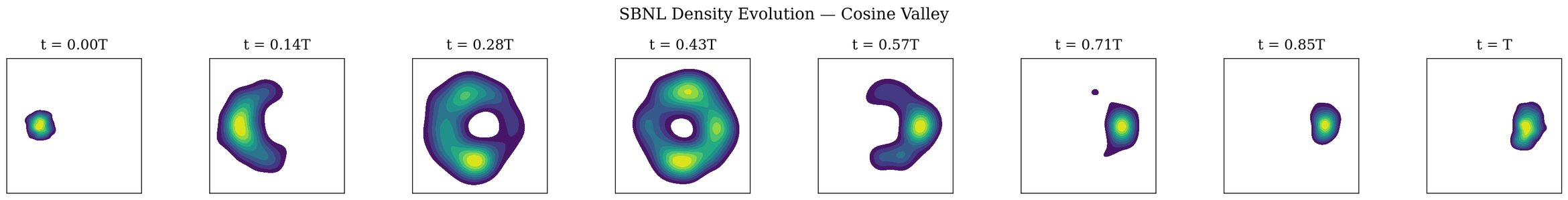}

\caption{Valley: SBNL transport and estimated density evolution $m(t,\cdot)=\varphi(t,\cdot)\bar\varphi(T-t,\cdot)$.}
\label{fig:potential_valley}
\end{figure}

\begin{figure}[H]
\centering
\captionsetup{font=footnotesize}

\includegraphics[width=0.55\textwidth]{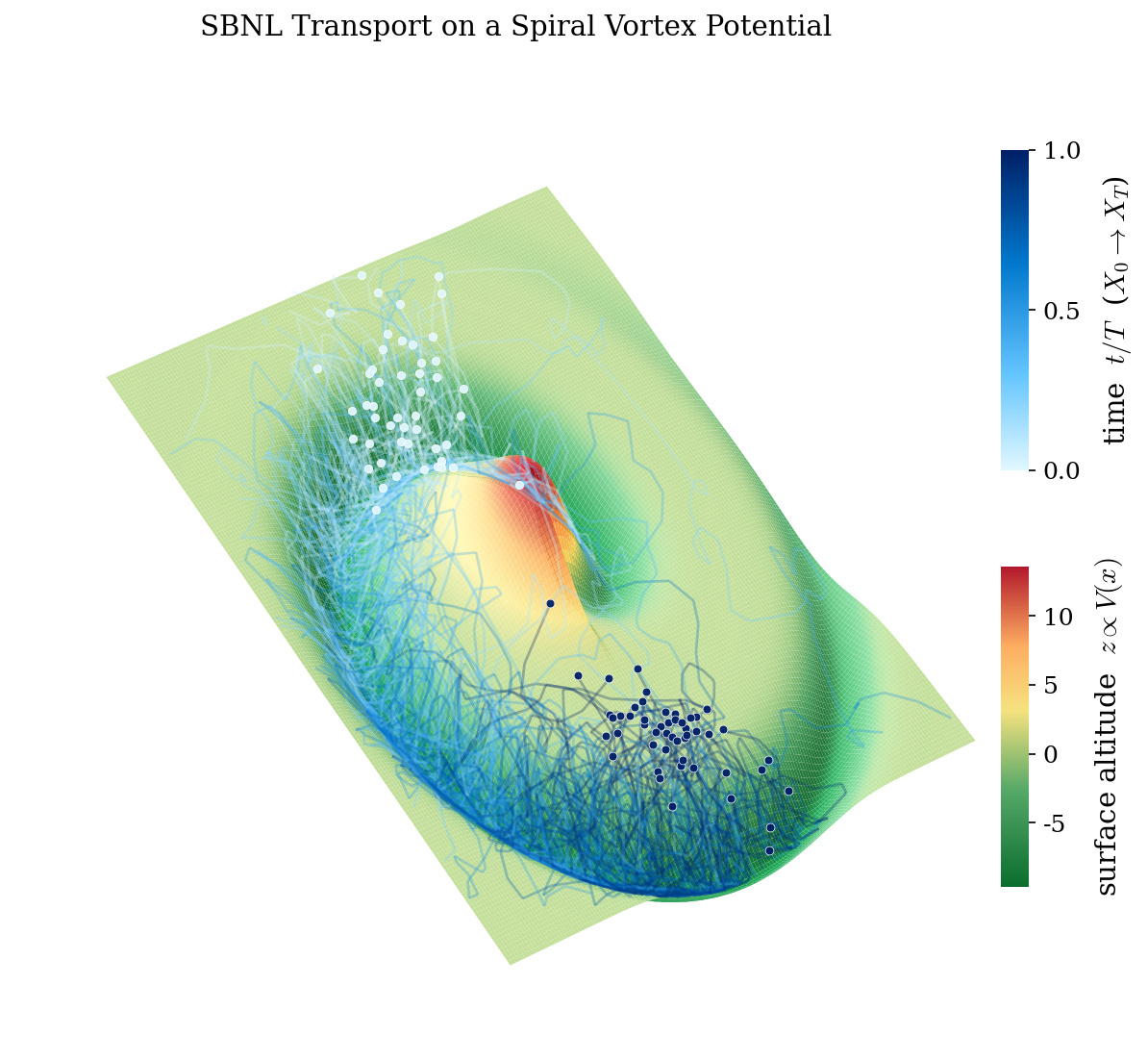}\\[-4pt]
\includegraphics[width=0.9\textwidth]{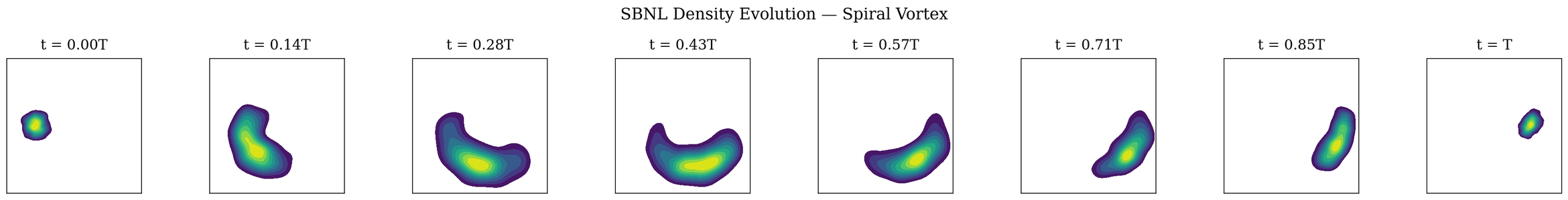}

\caption{Spiral: SBNL transport and estimated density evolution $m(t,\cdot)=\varphi(t,\cdot)\bar\varphi(T-t,\cdot)$.}
\label{fig:potential_spiral}
\end{figure}

\begin{figure}[H]
\centering
\captionsetup{font=footnotesize}

\includegraphics[width=0.55\textwidth]{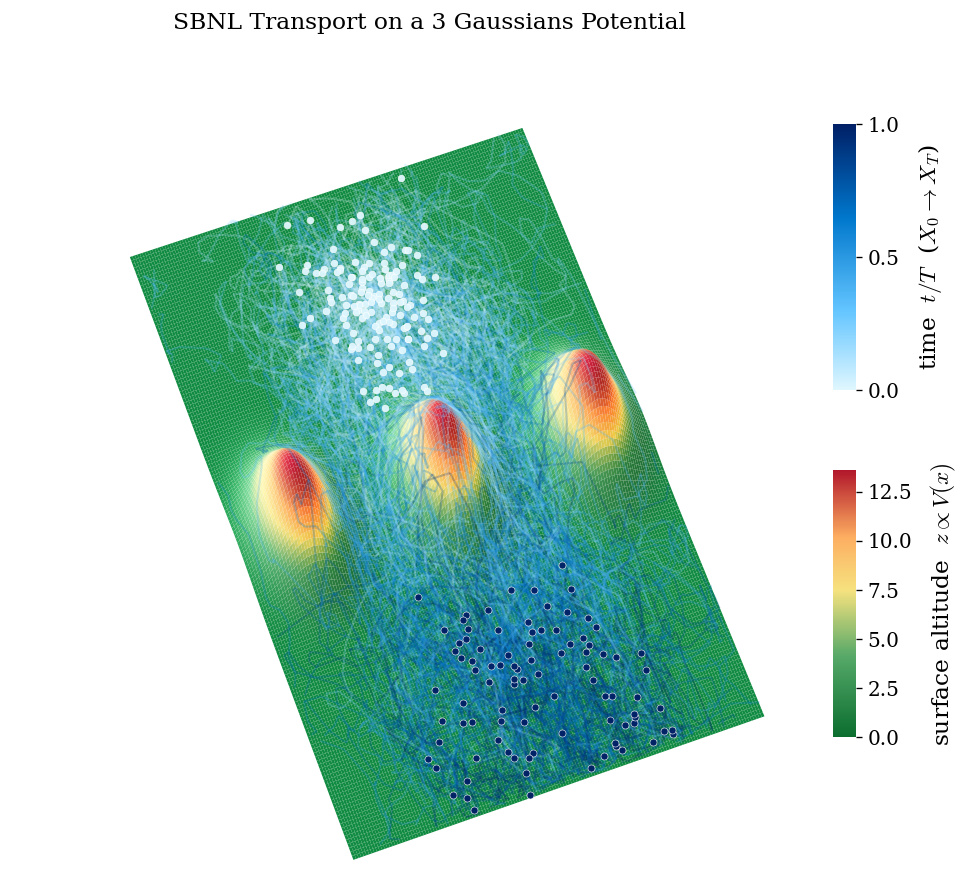}\\[-4pt]
\includegraphics[width=0.9\textwidth]{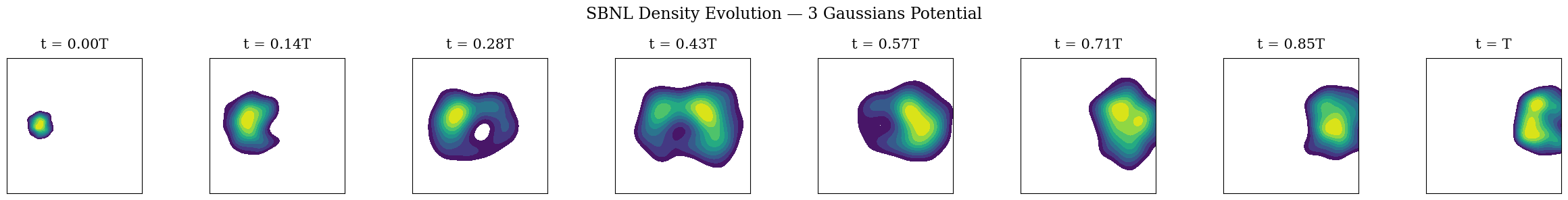}

\caption{Three Gaussians: SBNL transport and estimated density evolution $m(t,\cdot)=\varphi(t,\cdot)\bar\varphi(T-t,\cdot)$.}
\label{fig:potential_threegauss}
\end{figure}

\subsection{50D synthetic experiments}
\label{app:50D-synthetic-experiments}

We consider a nonlinear bridge between Gaussian mixtures in $\mathbb{R}^{50}$ with $c=0.1$.  
Samples and trajectories are visualized by projection onto the first two PCA directions, with the PCA basis fit on pooled validation draws from $\mu_0$ and $\mu_T$.  
Figure~\ref{fig:gen_vs_true_50D} compares generated and target marginals after training; Figure~\ref{fig:evolution_50D} shows the evolution of projected trajectories.

\begin{figure}[H]
  \centering
  \includegraphics[width=0.7\textwidth]{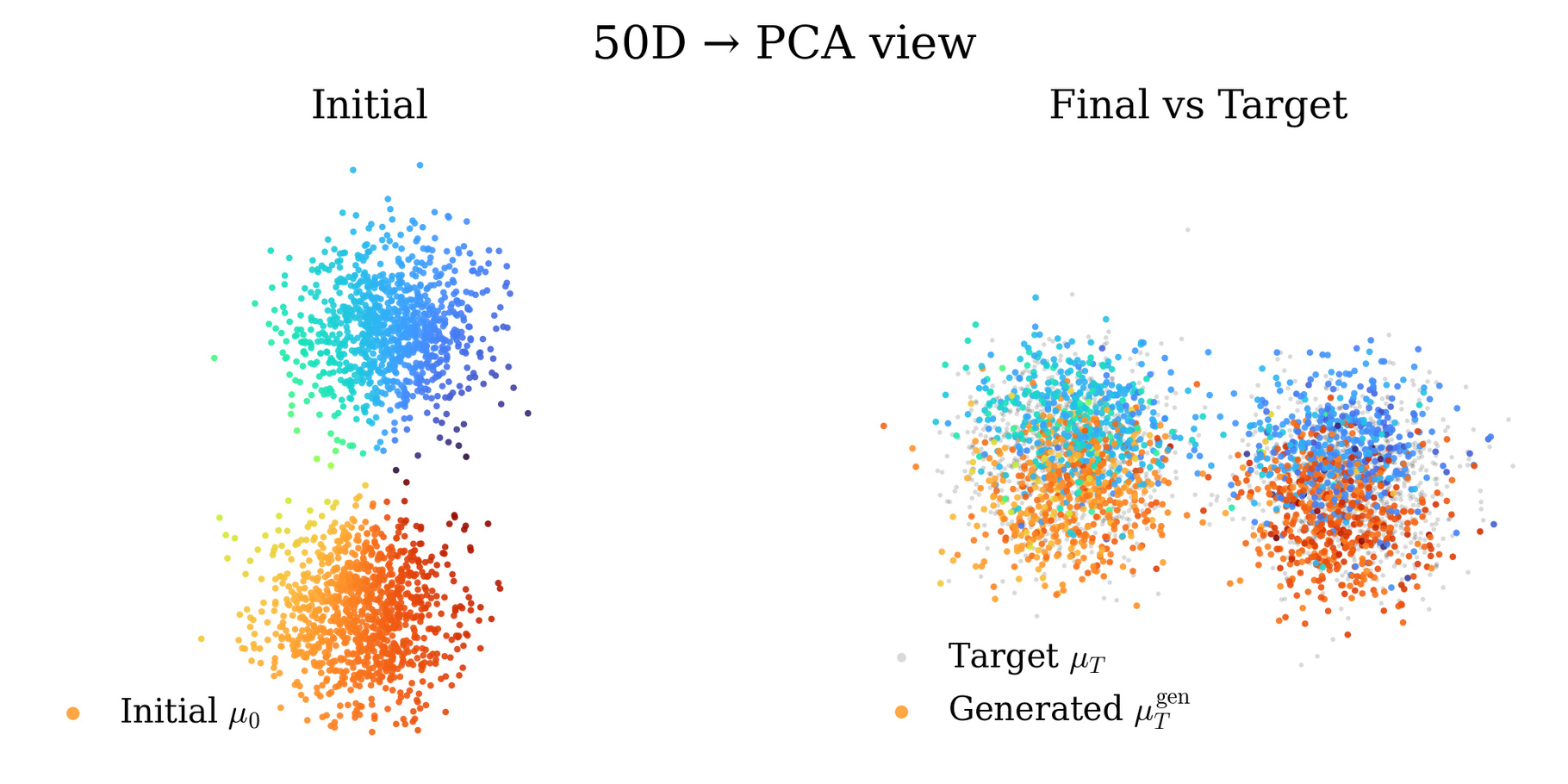}
  \caption{Generated vs.~true samples in $\R^{50}$ (PCA projection).}
  \label{fig:gen_vs_true_50D}
  \end{figure}
  
  \begin{figure}[H]
  \centering
  \includegraphics[width=\textwidth]{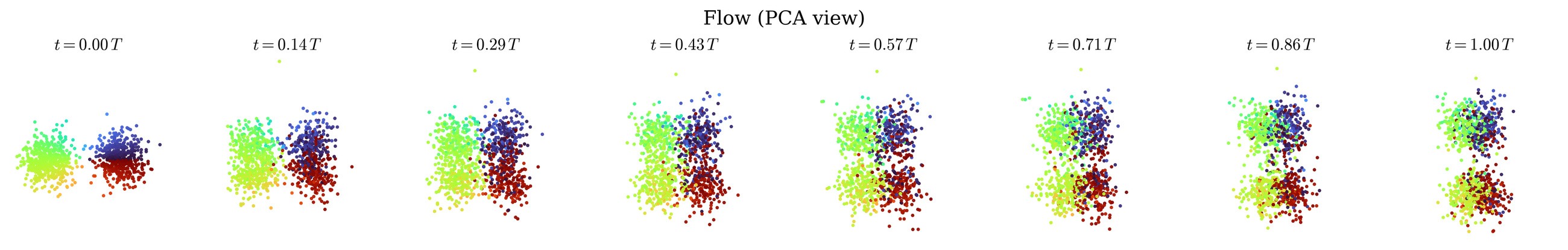}
  \caption{Temporal evolution in $\R^{50}$ (PCA projection).}
  \label{fig:evolution_50D}
  \end{figure}

\begin{figure}[H]
    \centering
    \begin{subfigure}[b]{0.32\textwidth}
        \centering
        \includegraphics[width=\textwidth]{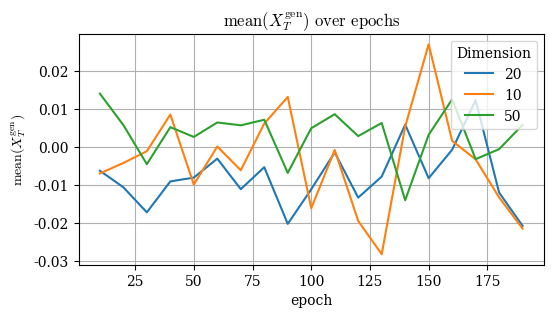}
        \caption{Mean of the generated samples across different dimensions (target mean: $0$).}
        \label{fig:mean_X_T_gen_over_epochs}
    \end{subfigure}
    \hfill
    \begin{subfigure}[b]{0.32\textwidth}
        \centering
        \includegraphics[width=\textwidth]{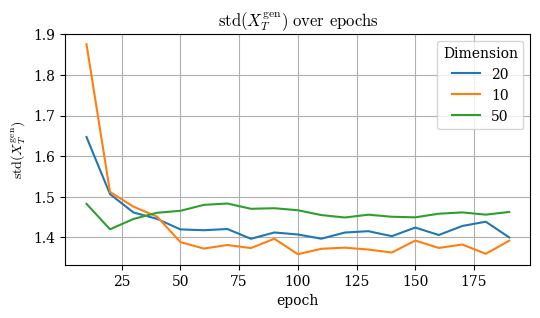}
        \caption{Std of the generated samples across different dimensions (target std: $\sqrt{2}$).}
        \label{fig:std_X_T_gen_over_epochs}
    \end{subfigure}
    \hfill
    \begin{subfigure}[b]{0.32\textwidth}
        \centering
        \includegraphics[width=\textwidth]{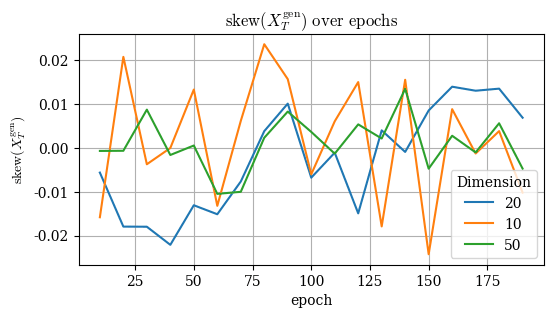}
        \caption{Skewness of the generated samples across different dimensions (target skewness: $0$).}
        \label{fig:skew_X_T_gen_over_epochs}
    \end{subfigure}
    \caption{Evolution of the first three moments (mean, standard deviation, skewness) of the generated samples over training epochs, for different latent dimensions.}
    \label{fig:three_in_a_row}
\end{figure}

\subsection{MNIST in latent space}
We consider a nonlinear bridge between MNIST digits in latent space with $c=0.1$.  
Samples and trajectories are visualized by projection onto the first two PCA directions, with the PCA basis fit on pooled validation draws from $\mu_0$ and $\mu_T$ (Figure~\ref{fig:image_diffusion_1_to_3}).  
Figure~\ref{fig:initial_vs_generated_latent} projects the $d{=}32$ latent codes onto the first two PCA directions (basis fit on pooled validation samples from $\mu_0$ and $\mu_T$): the left panel shows the initial law $\mu_0$, and the right panel overlays generated terminal latents on the target $\mu_T$, illustrating that the bridge maps the source geometry into the target bulk while preserving local structure visible in the point colors.
Figure~\ref{fig:generated_samples_1_to_3} compares generated and target digits after training.

\begin{figure}[H]
  \centering
  \begin{subfigure}[b]{0.49\textwidth}
    \centering
    \includegraphics[width=\textwidth]{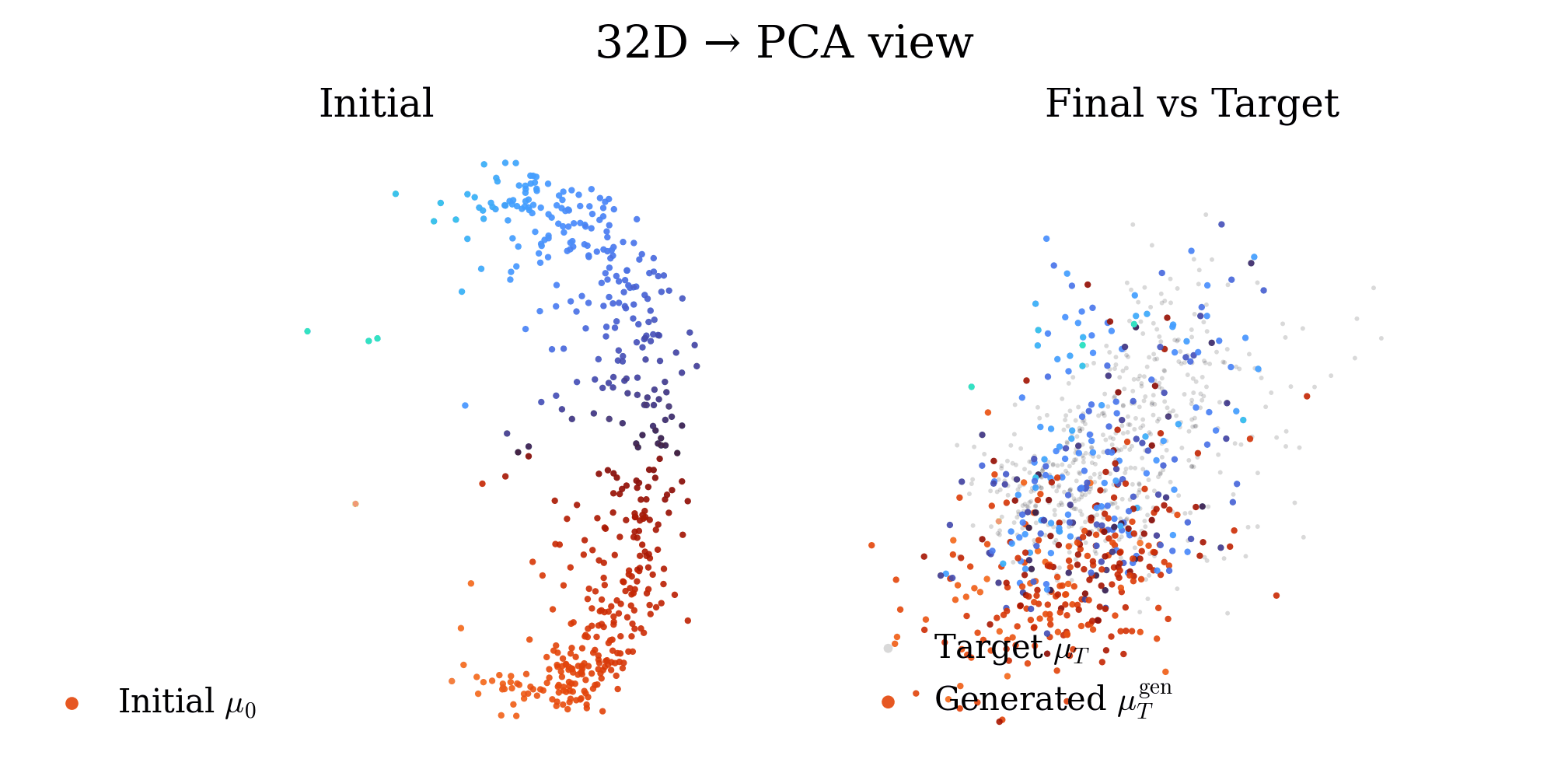}
    \caption{PCA view of $\mu_0$ (left) and of generated terminal samples vs.\ target $\mu_T$ (right; target in gray, generated colored).}
    \label{fig:initial_vs_generated_latent}
  \end{subfigure}
  \hfill
  \begin{subfigure}[b]{0.49\textwidth}
    \centering
    \includegraphics[width=\textwidth]{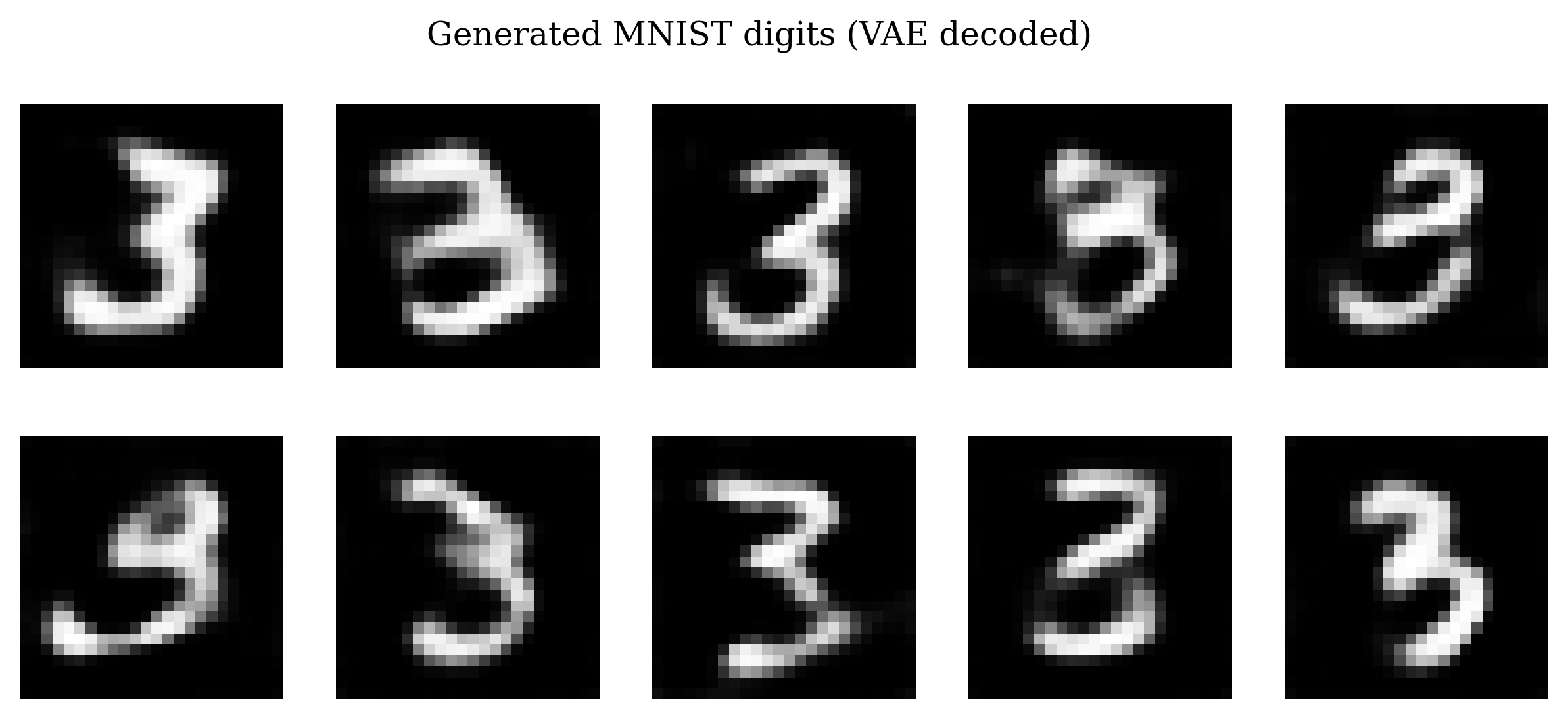}
    \caption{Generated MNIST digits obtained by decoding latent interpolations.}
    \label{fig:generated_samples_1_to_3}
  \end{subfigure}
  \caption{MNIST latent bridge ($d{=}32$): latent PCA visualization and decoded samples.}
  \label{fig:mnist_latent_and_decoded_row}
\end{figure}

\subsection{FFHQ in latent space}

\begin{figure}[H]
  \centering
  \includegraphics[width=0.95\textwidth]{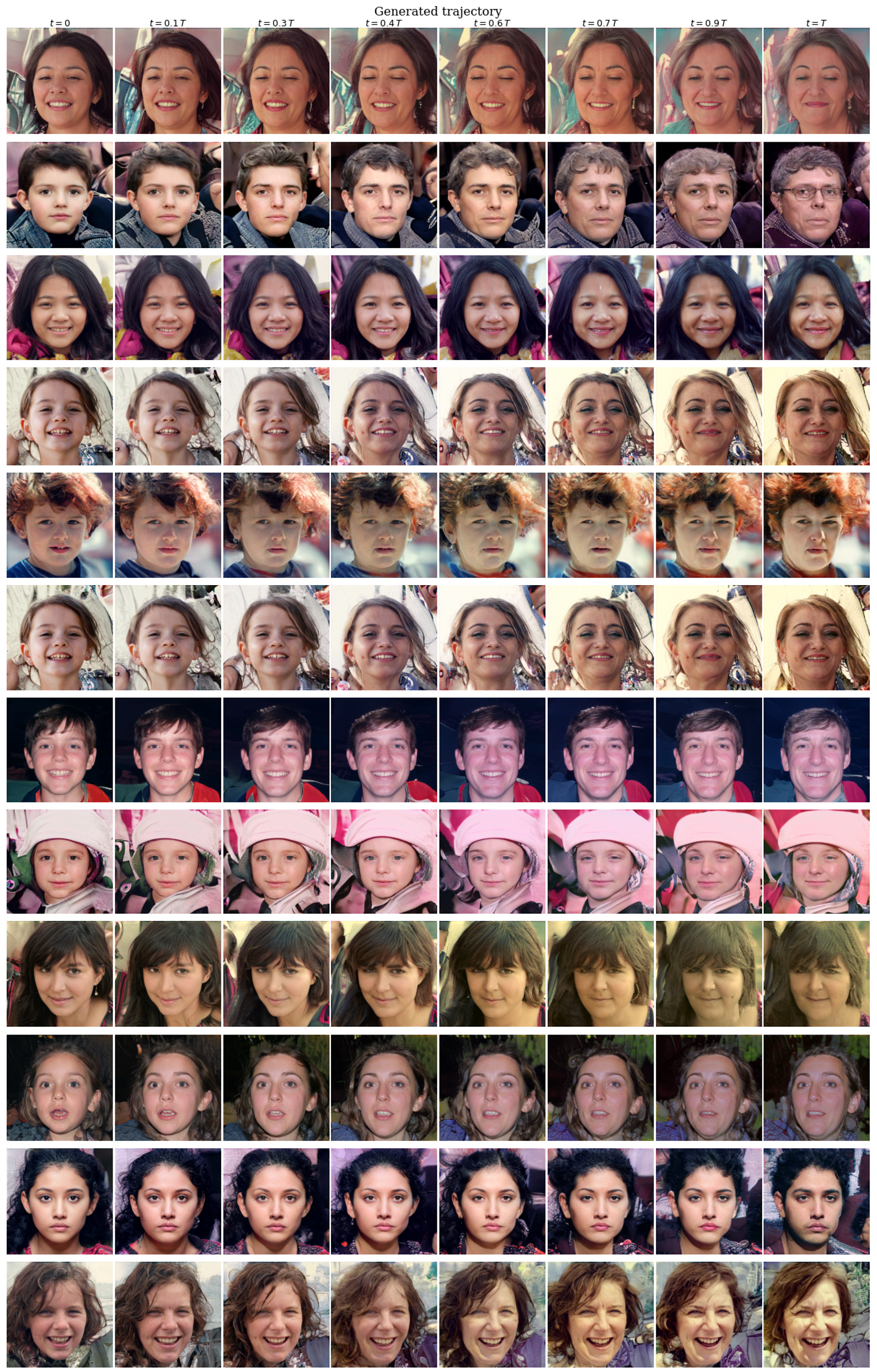}
  \caption{FFHQ young-to-old: decoded trajectories during sampling.}
  \label{fig:young_old_traj_large}
  \end{figure}

\subsection{Comparison on FFHQ}
\label{app:comparison-ffhq}

We extend the FFHQ young-to-old comparison from Figure~\ref{fig:comparison_ffhq_main} to a larger panel of source children and all baselines, in ALAE latent space. Figure~\ref{fig:comparison_ffhq_full} shows, for each row, the source child ($\mu_0$) on the left, the terminal samples produced by SBNL, LightSB, LightSBB-M, DSBM and BranchSBM (with separate female and male conditioning for BranchSBM), and a reference adult image drawn at random from $\mu_T$ on the right.

\begin{figure}[H]
\centering
\includegraphics[width=0.95\textwidth]{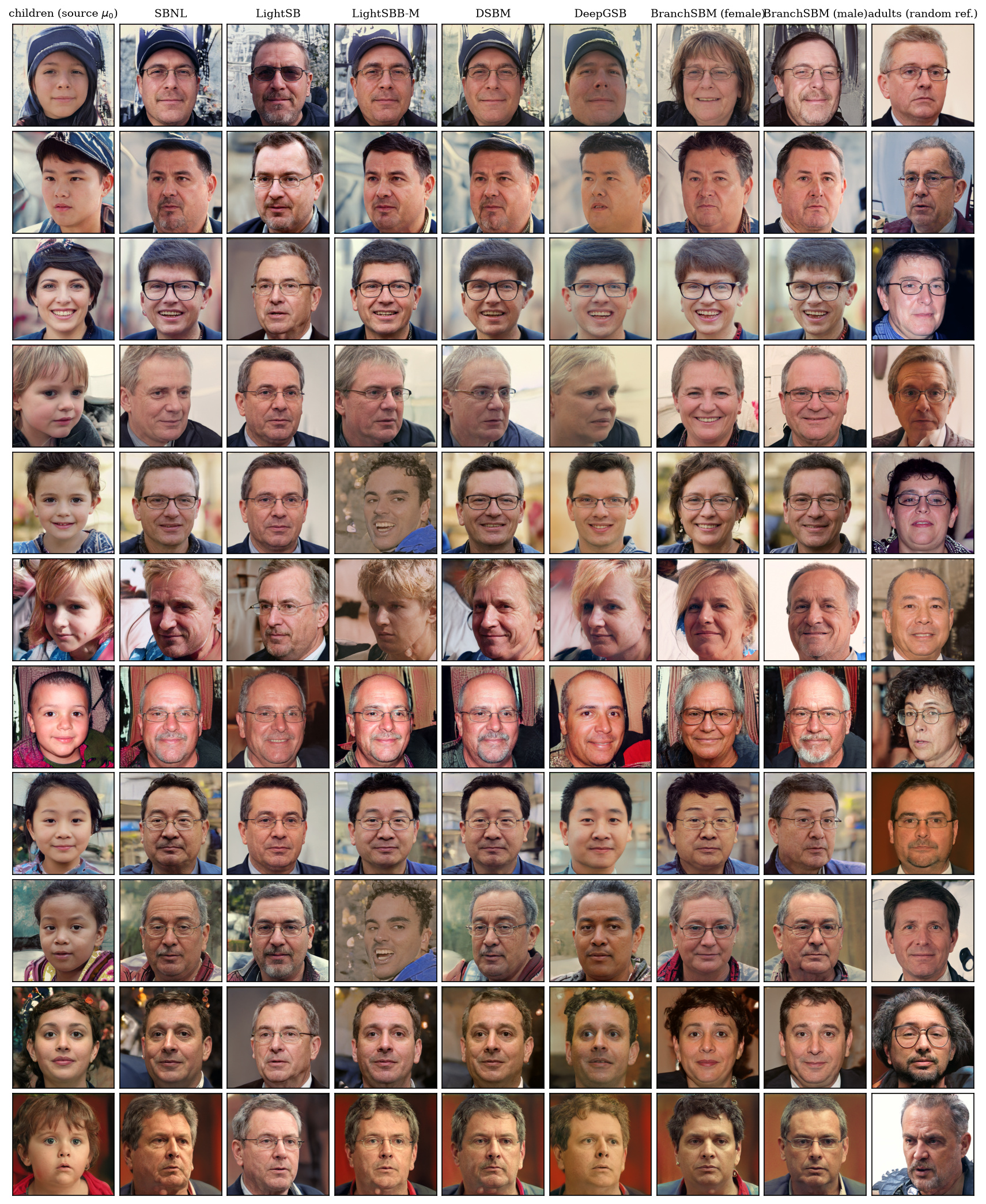}
\caption{FFHQ young-to-old comparison in ALAE latent space, extended panel.}
\label{fig:comparison_ffhq_full}
\end{figure}

\section{Additional figures}

\subsection{Evolution of the potentials during training}

\label{additional-figures}
\begin{figure}[H]
\centering
\setlength{\tabcolsep}{2pt}
\renewcommand{\arraystretch}{0}
\begin{tabular}{ccccc}
\includegraphics[width=0.19\textwidth]{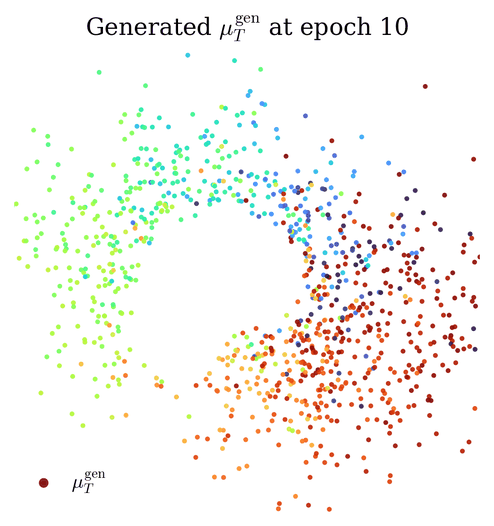} &
\includegraphics[width=0.19\textwidth]{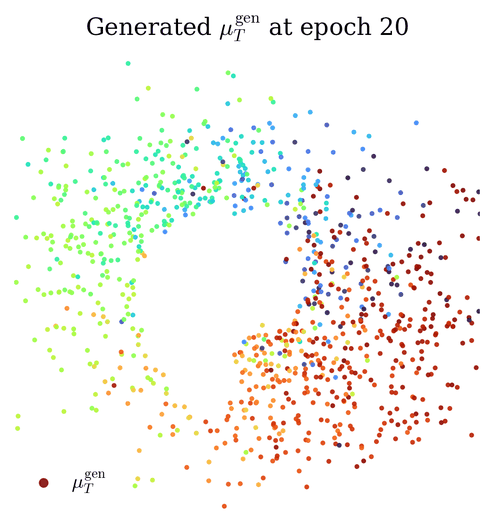} &
\includegraphics[width=0.19\textwidth]{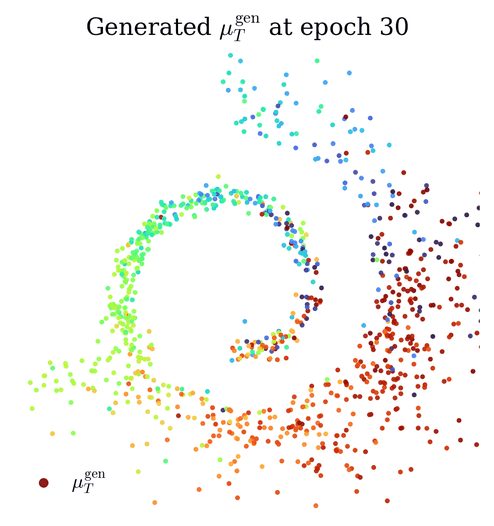} &
\includegraphics[width=0.19\textwidth]{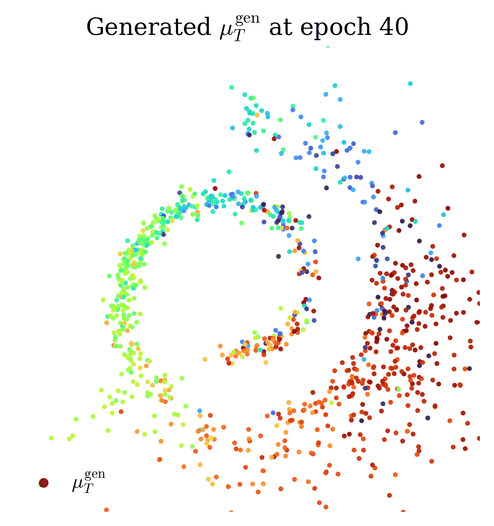} &
\includegraphics[width=0.19\textwidth]{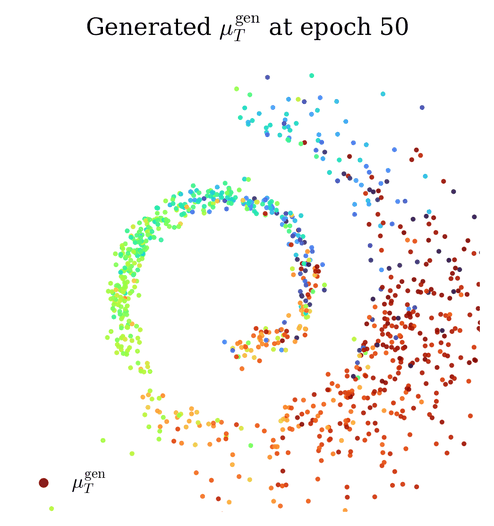} \\
\includegraphics[width=0.19\textwidth]{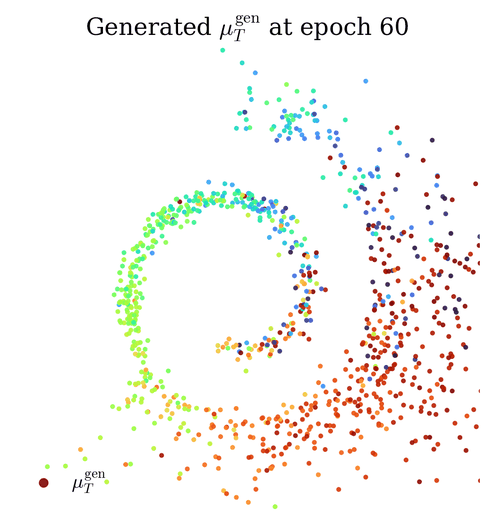} &
\includegraphics[width=0.19\textwidth]{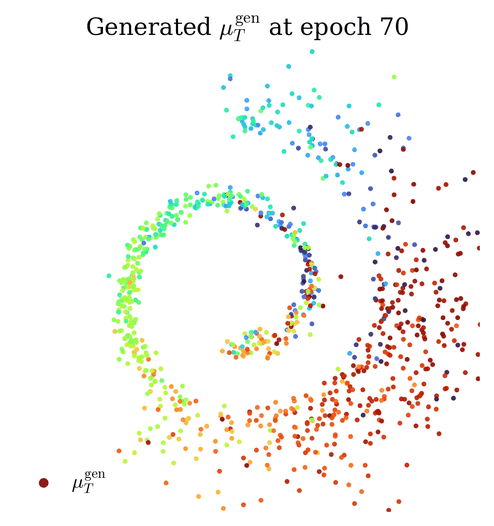} &
\includegraphics[width=0.19\textwidth]{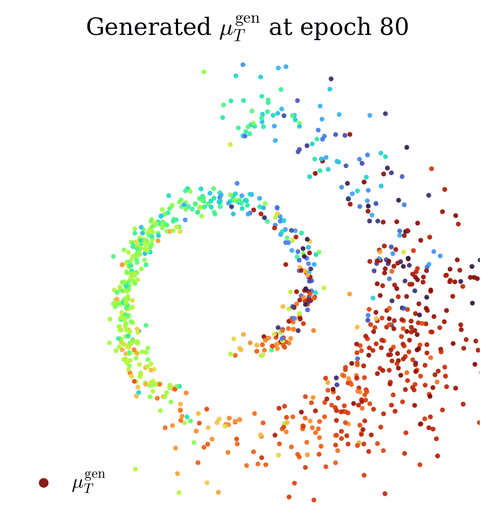} &
\includegraphics[width=0.19\textwidth]{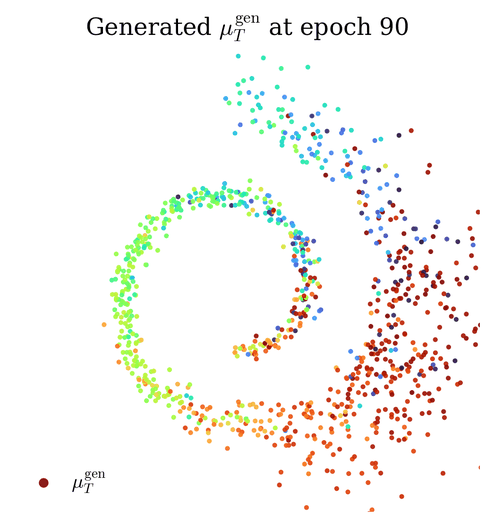} &
\includegraphics[width=0.19\textwidth]{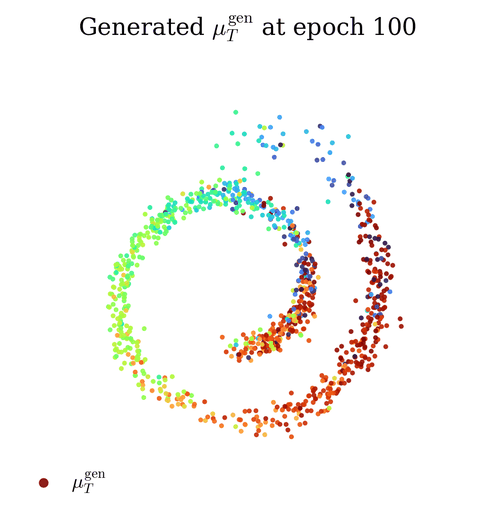} \\
\end{tabular}
\caption{Convergence of $\mathrm{Law}(X_T)$ with training  (Moon curves $\to$ Swiss).}
\label{fig:generation_evolution_grid}
\end{figure}

\begin{figure}[H]
\centering
\setlength{\tabcolsep}{2pt}
\renewcommand{\arraystretch}{0}
\begin{tabular}{ccccc}
\includegraphics[width=0.19\textwidth]{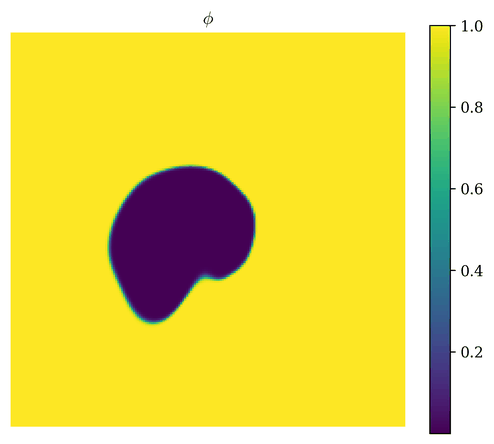} &
\includegraphics[width=0.19\textwidth]{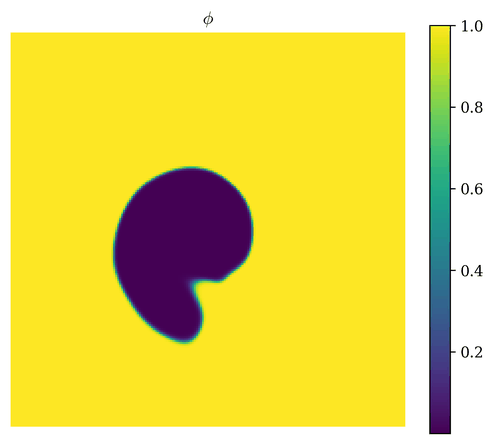} &
\includegraphics[width=0.19\textwidth]{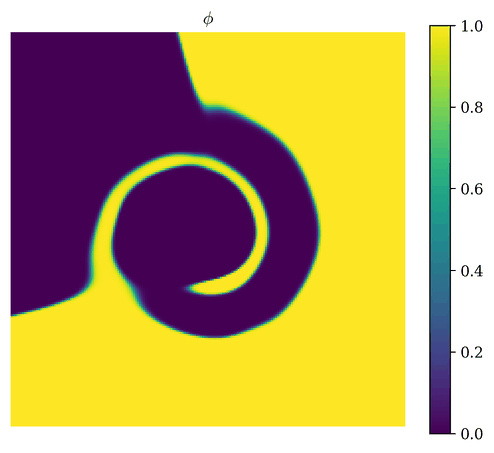} &
\includegraphics[width=0.19\textwidth]{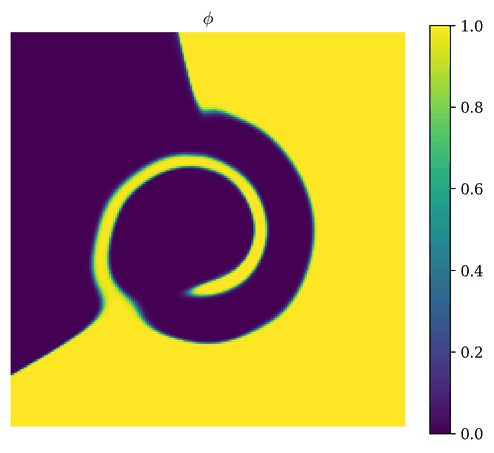} &
\includegraphics[width=0.19\textwidth]{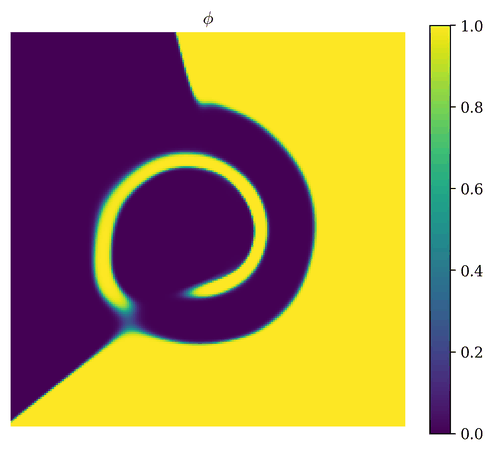} \\
\includegraphics[width=0.19\textwidth]{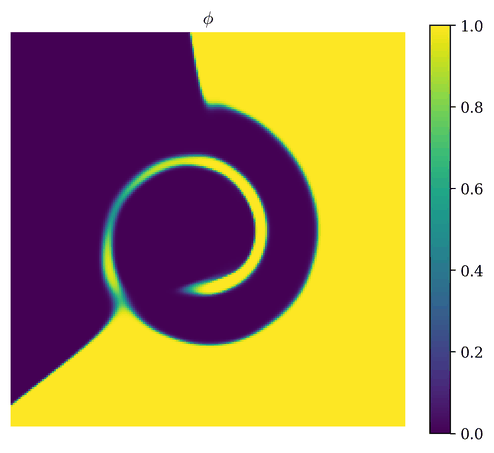} &
\includegraphics[width=0.19\textwidth]{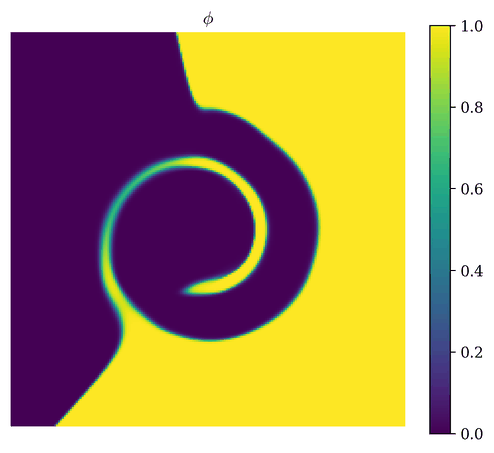} &
\includegraphics[width=0.19\textwidth]{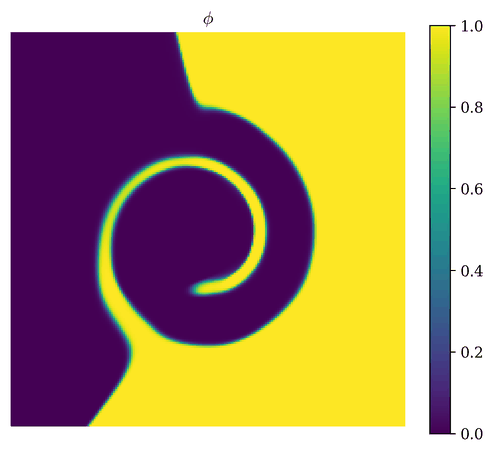} &
\includegraphics[width=0.19\textwidth]{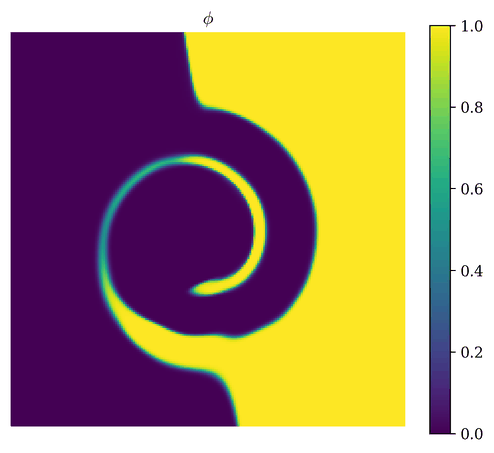} &
\includegraphics[width=0.19\textwidth]{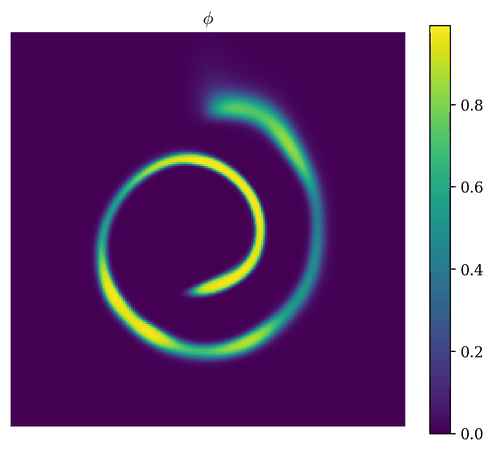} \\
\end{tabular}
\caption{Evolution of potential $\varphi$ (Moon curves $\to$ Swiss).}
\label{fig:phi_evolution_grid}
\end{figure}

\begin{figure}[H]
\centering
\setlength{\tabcolsep}{2pt}
\renewcommand{\arraystretch}{0}
\begin{tabular}{ccccc}
\includegraphics[width=0.19\textwidth]{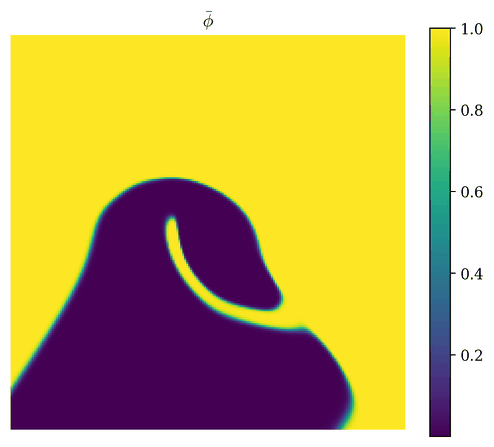} &
\includegraphics[width=0.19\textwidth]{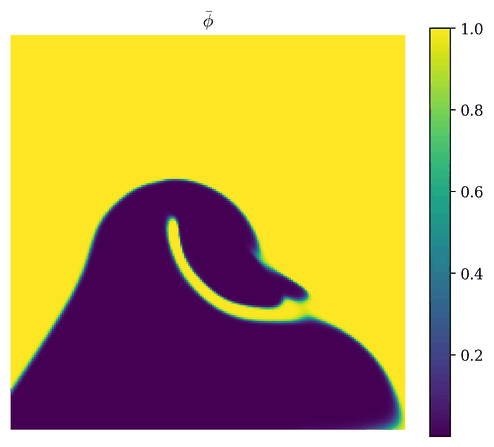} &
\includegraphics[width=0.19\textwidth]{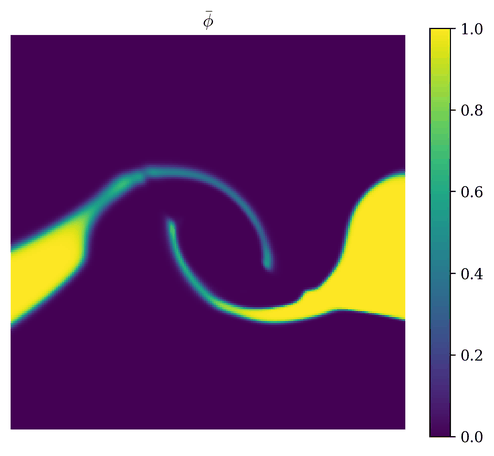} &
\includegraphics[width=0.19\textwidth]{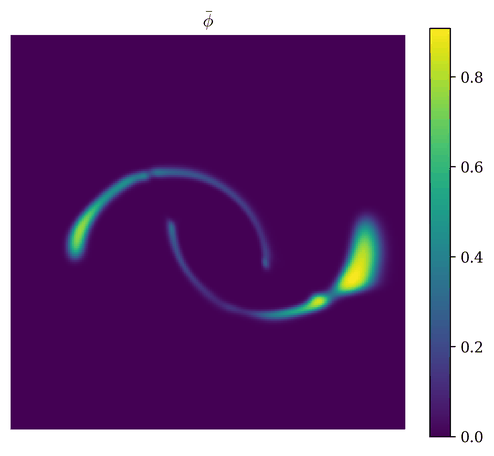} &
\includegraphics[width=0.19\textwidth]{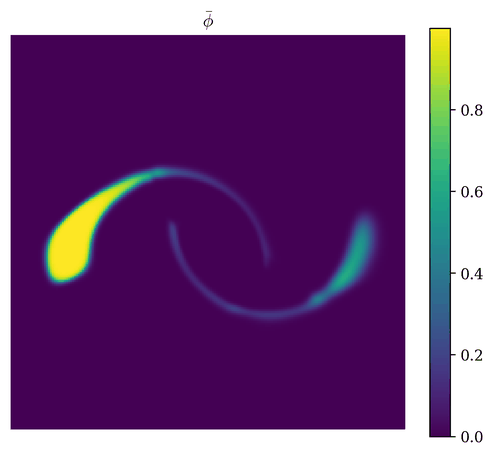} \\
\includegraphics[width=0.19\textwidth]{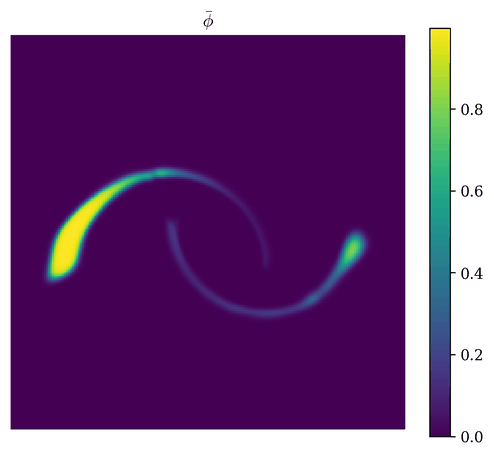} &
\includegraphics[width=0.19\textwidth]{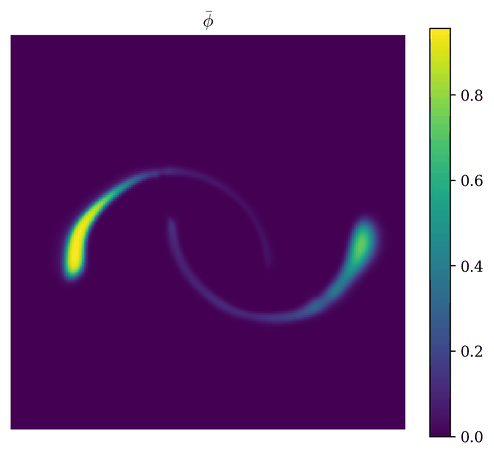} &
\includegraphics[width=0.19\textwidth]{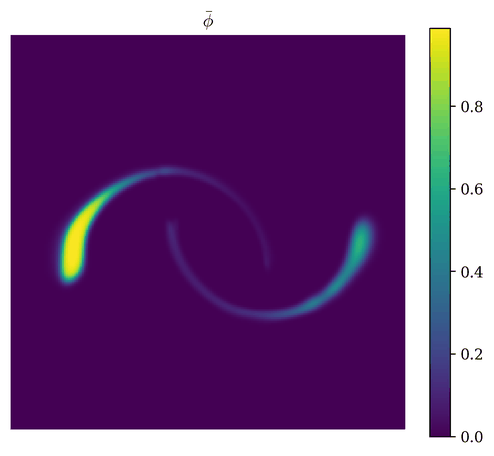} &
\includegraphics[width=0.19\textwidth]{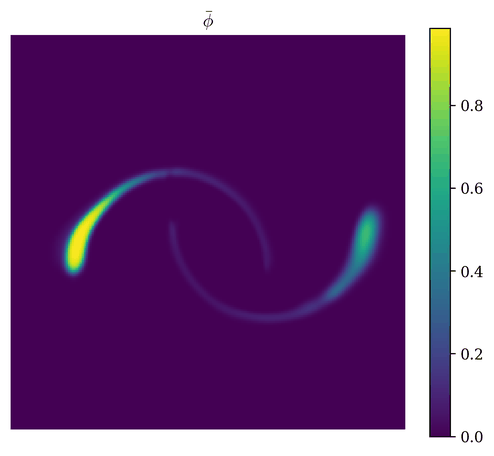} &
\includegraphics[width=0.19\textwidth]{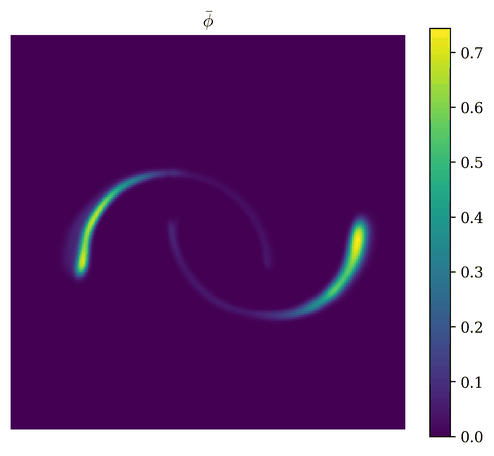} \\
\end{tabular}
\caption{Evolution of potential $\bar\varphi$ (Moon curves $\to$ Swiss).}
\label{fig:phi_bar_evolution_grid}
\end{figure}

\subsection{Relaxation of the hard constraint and finite penalty $\lambda$}
\begin{figure}[H]
\centering
\includegraphics[width=\textwidth]{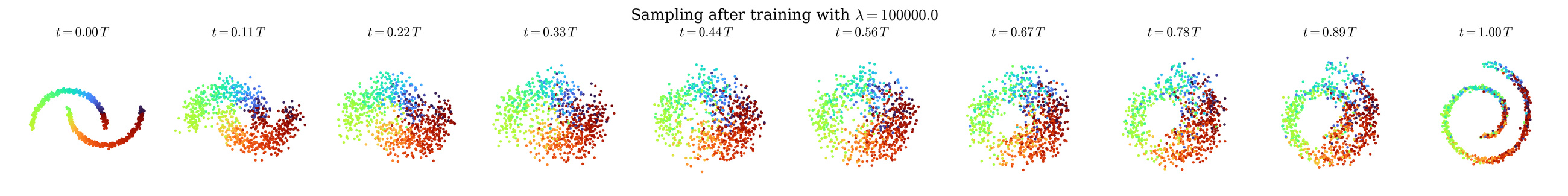}\\[0.5em]
\includegraphics[width=\textwidth]{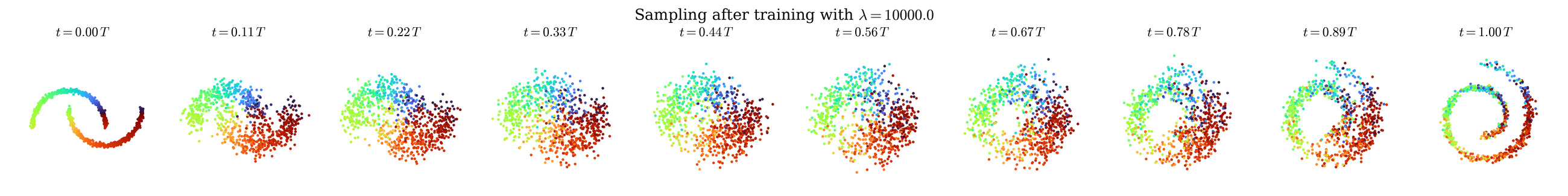}\\[0.5em]
\includegraphics[width=\textwidth]{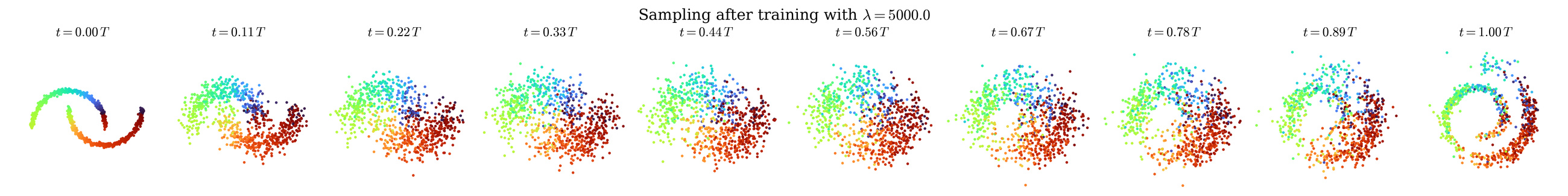}\\[0.5em]
\includegraphics[width=\textwidth]{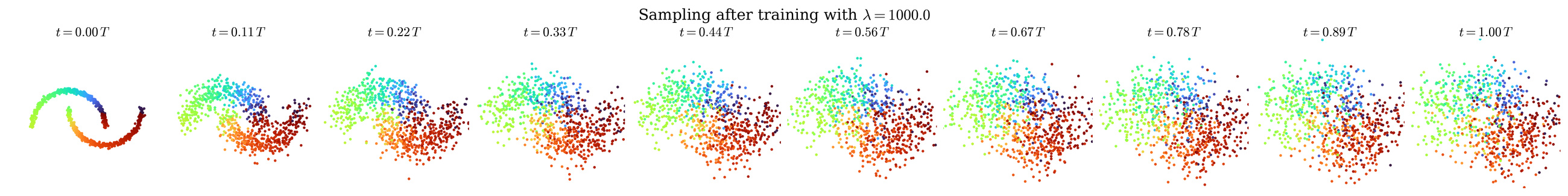}\\[0.5em]
\includegraphics[width=\textwidth]{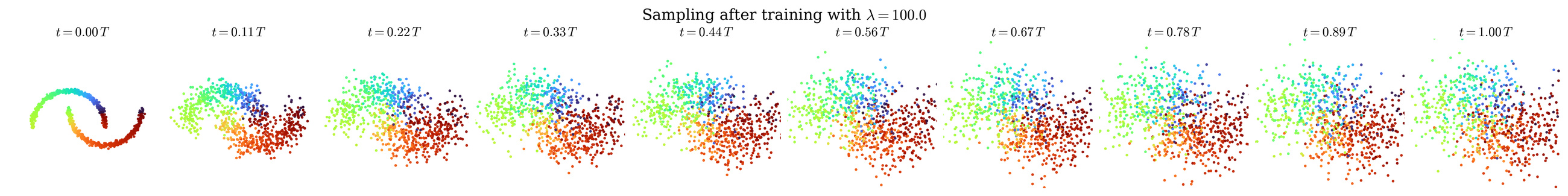}
\caption{Flow evolution for $\lambda \in \{10^5,10^4,5\!\times\!10^3,10^3,10^2\}$.}
\label{fig:evolution_lambda}
\end{figure}


\end{document}